\documentclass{applemlr}
\usepackage{amsmath}
\usepackage{aliascnt}
\usepackage{enumerate}
\usepackage{algorithm}
\usepackage{algpseudocode}
\usepackage{amsfonts}
\usepackage{amsthm}
\usepackage{newtxtt}
\usepackage{diagbox}
\usepackage{colortbl}
\usepackage{amssymb}
\usepackage{xspace}
\usepackage{wrapfig}
\usepackage{adjustbox}
\usepackage{tabularx}
\usepackage{booktabs}
\usepackage{mathtools}
\usepackage{tikz}
\usepackage{enumitem}
\usepackage{silence}
\usepackage{dsfont}
\usepackage{multirow}
\usepackage{makecell}
\usepackage{xfakebold}
\usepackage{scalefnt}

\usepackage{placeins}

\usepackage{amssymb,amsmath,amsthm}
\usepackage{bbm}
\usepackage{xifthen}
\usepackage{todonotes}
\usepackage{natbib}
\usepackage{xcolor}

\usepackage{algorithm}
\usepackage{algpseudocode} %

\usepackage{tcolorbox}
\tcbuselibrary{skins,breakable}
\usepackage{listings}

\usepackage{titlesec}
\newcommand{\arxiv}[1]{}

\newcommand{\1}{\mathbbm{1}}

\DeclareMathOperator*{\argmax}{argmax}
\newcommand{\x}{\times}

\newcommand{\R}{\mathbb{R}}
\newcommand{\N}{\mathbb{N}}

\newcommand{\E}{\mathop{\mathbb{E}}}

\renewcommand{\epsilon}{\varepsilon}

\newcommand{\cV}{\mathcal{V}}
\newcommand{\cD}{\mathcal{D}}

\newcommand{\cX}{\mathcal{X}}
\newcommand{\cW}{\mathcal{W}}

\newcommand{\eE}{\mathcal{E}}
\newcommand{\eK}{\eE_K}

\newcommand{\WBcal}{\mathcal{W}_{B}^{\rm{(full)}}}
\newcommand{\WBconf}{\mathcal{W}_{B}}

\newtheorem{theorem}{Theorem}
\newaliascnt{lemma}{theorem}
\newtheorem{lemma}[lemma]{Lemma}
\aliascntresetthe{lemma}

\newaliascnt{remark}{theorem}
\newtheorem{remark}[remark]{Remark}
\aliascntresetthe{remark}

\newaliascnt{claim}{theorem}
\newtheorem{claim}[claim]{Claim}
\aliascntresetthe{claim}

\newaliascnt{definition}{theorem}
\newtheorem{definition}[definition]{Definition}
\aliascntresetthe{definition}

\newaliascnt{corollary}{theorem}
\newtheorem{corollary}[corollary]{Corollary}
\aliascntresetthe{corollary}

\newaliascnt{conjecture}{theorem}

\aliascntresetthe{conjecture}

\DeclareRobustCommand{\[}{\begin{equation}}
\DeclareRobustCommand{\]}{\end{equation}}

\let\Oldforall\forall
\renewcommand{\forall}{~\Oldforall} %

\let\Oldinf\inf
\renewcommand{\inf}{\Oldinf\limits}
\let\Oldsup\sup
\renewcommand{\sup}{\Oldsup\limits}

\usepackage{cleveref}

\usepackage{longtable, tabularx,multirow,array}
\usepackage{booktabs}
\usepackage{titletoc}

\usepackage{dsfont}
\usepackage{soul}
\usepackage{tikz}
\usetikzlibrary{arrows.meta, positioning}
\colorlet{col1}{blue}
\definecolor{col2}{RGB}{83,172,121} %
\definecolor{col3}{RGB}{250,151,92} %
\definecolor{col4}{HTML}{FFAABB}

\usepackage{xcolor}
\definecolor{AppleWhite}{RGB}{255,255,255}
\definecolor{ApplePrimaryCoolGray}{RGB}{116,128,139}
\definecolor{AppleCoolGray1}{RGB}{199,209,214}
\definecolor{AppleCoolGray2}{RGB}{147,174,190}
\definecolor{AppleCoolGray3}{RGB}{124,147,160}
\definecolor{AppleCoolGray4}{RGB}{92,102,109}
\definecolor{AppleCoolGray5}{RGB}{78,93,100}
\definecolor{AppleCoolGray6}{RGB}{53,60,65}
\definecolor{AppleBlack}{RGB}{0,0,0}
\definecolor{AppleSecondaryChartGray}{RGB}{168,168,168}
\definecolor{AppleChartGray2}{RGB}{233,233,233}
\definecolor{AppleChartGray3}{RGB}{211,211,211}
\definecolor{AppleChartGray4}{RGB}{190,190,190}
\definecolor{AppleChartGray5}{RGB}{140,140,140}
\definecolor{AppleChartGray6}{RGB}{102,102,102}
\definecolor{AppleChartGray7}{RGB}{64,64,64}
\definecolor{ApplePrimaryChartBlue}{RGB}{84,151,193}
\definecolor{AppleBlue2}{RGB}{212,229,239}
\definecolor{AppleBlue3}{RGB}{169,202,223}
\definecolor{AppleBlue4}{RGB}{127,177,209}
\definecolor{AppleBlue5}{RGB}{71,130,166}
\definecolor{AppleBlue6}{RGB}{55,99,128}
\definecolor{AppleBlue7}{RGB}{45,72,89}
\definecolor{ApplePrimaryChartGreen}{RGB}{83,172,121}
\definecolor{AppleGreen2}{RGB}{212,234,221}
\definecolor{AppleGreen3}{RGB}{169,213,188}
\definecolor{AppleGreen4}{RGB}{126,193,155}
\definecolor{AppleGreen5}{RGB}{58,140,82}
\definecolor{AppleGreen6}{RGB}{39,102,54}
\definecolor{AppleGreen7}{RGB}{29,58,31}
\definecolor{ApplePrimaryChartYellow}{RGB}{253,195,93}
\definecolor{AppleYellow2}{RGB}{254,240,214}
\definecolor{AppleYellow3}{RGB}{254,224,174}
\definecolor{AppleYellow4}{RGB}{254,210,134}
\definecolor{AppleYellow5}{RGB}{230,168,69}
\definecolor{AppleYellow6}{RGB}{191,131,46}
\definecolor{AppleYellow7}{RGB}{153,107,54}
\definecolor{ApplePrimaryChartOrange}{RGB}{250,151,92}
\definecolor{AppleOrange2}{RGB}{254,229,214}
\definecolor{AppleOrange3}{RGB}{252,203,173}
\definecolor{AppleOrange4}{RGB}{252,178,133}
\definecolor{AppleOrange5}{RGB}{227,121,68}
\definecolor{AppleOrange6}{RGB}{191,87,46}
\definecolor{AppleOrange7}{RGB}{143,59,36}
\definecolor{ApplePrimaryChartRed}{RGB}{227,94,105}
\definecolor{AppleRed2}{RGB}{248,215,217}
\definecolor{AppleRed3}{RGB}{241,174,180}
\definecolor{AppleRed4}{RGB}{234,135,143}
\definecolor{AppleRed5}{RGB}{196,63,77}
\definecolor{AppleRed6}{RGB}{153,35,53}
\definecolor{AppleRed7}{RGB}{102,19,43}
\definecolor{ApplePrimaryChartPurple}{RGB}{161,150,204}
\definecolor{ApplePurple2}{RGB}{231,228,242}
\definecolor{ApplePurple3}{RGB}{208,202,229}
\definecolor{ApplePurple4}{RGB}{185,176,217}
\definecolor{ApplePurple5}{RGB}{128,113,171}
\definecolor{ApplePurple6}{RGB}{89,76,128}
\definecolor{ApplePurple7}{RGB}{62,46,101}
\definecolor{AppleCoolGray}{RGB}{116,128,139}
\definecolor{AppleChartGray}{RGB}{168,168,168}
\definecolor{AppleBlue}{RGB}{84,151,193}
\definecolor{AppleGreen}{RGB}{83,172,121}
\definecolor{AppleYellow}{RGB}{253,195,93}
\definecolor{AppleOrange}{RGB}{250,151,92}
\definecolor{AppleRed}{RGB}{227,94,105}
\definecolor{ApplePurple}{RGB}{161,150,204}

\definecolor{textgray}{HTML}{6E6E73}
\usetikzlibrary{positioning, calc}
\usetikzlibrary{decorations.pathmorphing}

\makeatletter
\patchcmd{\wrong@fontshape}{\@gobbletwo}{}{}{}
\makeatother
\numberwithin{equation}{section}
\makeatletter
\AtBeginDocument{
  \urlstyle{sf}
  
}
\makeatother

\definecolor{light}{RGB}{125, 125, 125}
\crefname{tcb@cnt@pbox}{code}{code}
\Crefname{tcb@cnt@pbox}{Code}{Code}

\crefname{assumption}{assumption}{assumption}
\Crefname{assumption}{Assumption}{Assumptions}
\crefname{claim}{claim}{claims}
\Crefname{claim}{Claim}{Claims}

\Crefname{equation}{Eq.}{Eqs.}
\Crefname{figure}{Fig.}{Figs.}
\Crefname{tabular}{Tab.}{Tabs.}
\Crefname{section}{Sec.}{Secs.}
\Crefname{appendix}{App.}{Apps.}
\Crefname{theorem}{Thm.}{Thms.}

\newtcolorbox[auto counter]{pbox}[2][]{
  colback=white,
  title=Code~\thetcbcounter: #2,
  #1,fonttitle=\sffamily,
  fontupper=\sffamily,
  arc=2pt,
  colframe=bgcolor,
  coltitle=fgcolor,
  colbacktitle=bgcolor,
  toptitle=0.25cm,
  bottomtitle=0.125cm
}

\makeatletter
\newcommand\applefootnote[1]{%
  \begingroup
  \renewcommand\thefootnote{}%
  \renewcommand\@makefntext[1]{\noindent##1}%
  \footnote{#1}%
  \addtocounter{footnote}{-1}%
  \endgroup
}
\makeatother

\definecolor{cverbbg}{gray}{0.90}

\title{\scalefont{0.93}Trained on Tokens, Calibrated on Concepts:\\
The Emergence of Semantic Calibration in LLMs}

\author{Preetum Nakkiran}
\author{Arwen Bradley}
\author{Adam Goliński}
\author{Eugene Ndiaye}
\author{Michael Kirchhof}
\author{Sinead Williamson}

\affiliation{Apple}

\abstract{Large Language Models (LLMs) often lack meaningful confidence estimates for their outputs. While base LLMs are known to exhibit next-token calibration,
it remains unclear whether they can assess confidence in the actual meaning of their responses beyond the token level.
We find that, when using a certain sampling-based notion of semantic calibration, base LLMs are remarkably well-calibrated:
they can meaningfully assess confidence in
open-domain question-answering tasks,
despite not being explicitly trained to do so.
Our main theoretical contribution 
establishes a mechanism for why semantic calibration emerges 
as a byproduct of next-token prediction, leveraging a recent connection between calibration and local loss optimality.
The theory relies on a general definition of ``$B$-calibration,''
which is a notion of calibration parameterized by
a choice of equivalence classes (semantic or otherwise).
This theoretical mechanism leads to a testable prediction:
base LLMs will be semantically calibrated when they can easily predict their own distribution over semantic answer classes before generating a response.
We state three implications of this prediction, which we validate through experiments: 
(1) Base LLMs are semantically calibrated across question-answering tasks,
(2) RL instruction-tuning systematically breaks this calibration,
and (3) chain-of-thought reasoning breaks calibration.
To our knowledge, our work provides the first principled explanation of when and why semantic calibration emerges in LLMs.

}

\date{\sffamily\today}

\begin{document}

\maketitle

\section{Introduction}
As Large Language Models (LLMs) become increasingly capable, 
it is important to understand the nature and extent of their uncertainty.
While LLMs can produce fluent answers to a range of difficult questions,
they do not inherently convey any sense of certainty in those answers.
Addressing this is an active research question:
can we extract a meaningful notion of confidence in an LLM's response?
This question is scientifically interesting even aside from applications:
it is a way of asking, do LLMs ``know what they don't know''?~\citep{kadavath2022language}

In the classification literature, one well-understood criterion for
uncertainty quantification is \emph{calibration}:
do the predicted probabilities reflect empirical frequencies?
For example, if an image classifier is 80\% confident 
on a set of inputs,
then it should be correct on 80\% of those predictions.
To apply this definition to LLMs, one approach is to treat the LLM
as a classifier that predicts the next-token, given all previous tokens.
There is strong empirical and theoretical evidence that base LLMs,
which are only pre-trained with the maximum likelihood loss, are typically \emph{next-token-calibrated} \citep{openai2023gpt,zhang-etal-2024-study,desai2020calibration}.
Next-token calibration is a meaningful notion of calibration in certain settings like True/False or multiple choice questions, where a single token encapsulates the entire response \citep{kadavath2022language,plaut2025probabilities}. 
For example, if we ask an LLM a multiple-choice question, then 
its probability distribution on the next-token (A/B/C/D) defines a 
prediction which is often calibrated.
However, when the model produces long-form answers to open-ended questions,
we desire a notion of uncertainty with respect to the \emph{semantic meaning} of the response, which next-token calibration does not 
directly capture.
E.g. if we ask the LLM ``What is the capital of France?,''
then it might answer ``Paris'' or ``It's Paris'' or ``The capital of France is Paris,''
and it is not clear how to use token-wise probabilities to derive meaningful confidences
in the response.

Prior works have proposed a variety of notions of semantic confidence for long-form text, including verbalized measures and sampling-based measures
(e.g.\ \emph{semantic entropy} of \citet{farquhar2024detecting}).
See \citet{vashurin-etal-2025-benchmarking} for a comprehensive overview.
However, from the empirical data it is unclear whether LLMs are naturally calibrated
with respect to \emph{any} of these semantic notions of confidence,
without being specifically trained for calibration \citep{kadavath2022language,yin-etal-2023-large,band2024linguisticcalibrationlongformgenerations,kapoor2024large,yoon2025reasoning,mei2025reasoning,tian-etal-2023-just}.
Empirically, calibration may depend on many factors:
the test distribution (math, trivia, etc.),
the post-training procedure (RLHF, DPO, RLVR, none, etc.),
the inference-time procedure (few-shot examples, chain-of-thought (CoT), best-of-K, etc.), the model size, the model architecture, the sampling temperature, etc.
All of these factors have been posited to affect calibration, for reasons that are not yet well understood 
\citep{kadavath2022language,openai2023gpt,leng2024taming,xiao2025restoring,zhang-etal-2024-study,wang-etal-2025-towards-objective}.

A priori, there is no reason to expect \emph{emergence}\footnote{
We use \emph{emergent} here to mean a 
structural regularity that arises implicitly (``for free'') due to system dynamics, not as a result of explicit external constraints. That is, ``Emergence Through Compression'' in the terminology of \citet{krakauer2025large}.
We do not mean to discuss changes as a result of model scaling, which is another common use of the term emergence \citep{wei2022emergent}.} of any of these forms of semantic calibration as a byproduct of standard pre-training with the maximum likelihood loss. 
In this work, we propose and test a mechanism by which 
a particular type of sampling-based semantic calibration actually can emerge for a large class of LLMs.
At a high level, the mechanism treats the LLM as a standard multi-class classifier (by collapsing outputs with the same semantic meaning),
and then adapts recent theoretical results 
on mechanisms of classifier calibration \citep{gopalan2024computationally,blasiok2023when,blasiok2024loss}.
\Cref{fig:B-collapse-intrp} illustrates the overall phenomenon of semantic calibration\footnote{
This definition of semantic calibration is
closely related to semantic entropy \citep{farquhar2024detecting},
as well as the sampling-based definitions of confidence in \citet{wang2023selfconsistency}, \citet{wei2024measuring}, and \citet{lamb2025semantic}.
}, described in detail in the next section. To our knowledge, our work is the first to propose a theoretically plausible mechanism
for semantic calibration in LLMs, and we validate the predictions of this theory empirically.

\begin{figure}[t]
    \centering
    \includegraphics[width=1.0\linewidth]{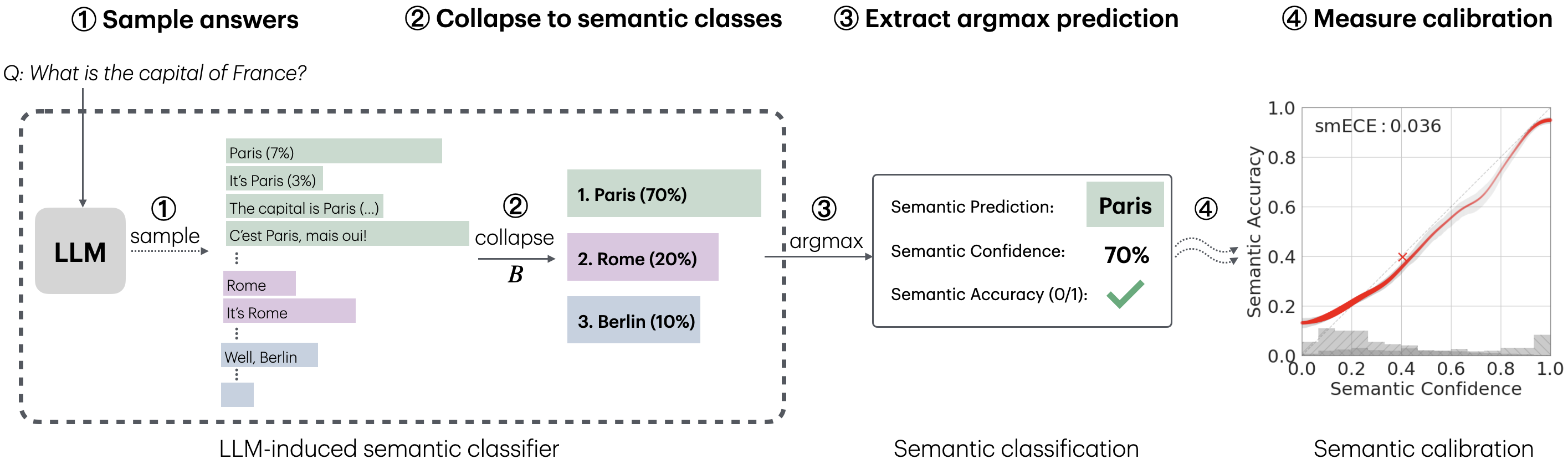}
    \caption{{\bf Semantic calibration} refers to calibration of an
    \emph{LLM-induced semantic classifier} (dashed box): the classifier induced by
     post-processing LLM outputs with a given semantic collapsing function,
     which we refer to as $B$ throughout.
    To measure semantic confidence calibration:
    for a given question, sample multiple temperature $T\!=\!1$ generations,
    and extract semantic answers by applying the
    collapsing function $B$ (e.g.\ a strong LLM prompted to extract one-word answers).
    This yields an empirical distribution over semantic classes (above: Paris, Rome, Berlin),
    which we treat as the classifier output.
    This classifier output defines a semantic prediction (=argmax probability)
    and a semantic confidence (=max probability).
    \emph{Semantic confidence calibration} means,
    over all questions, these predictions are confidence-calibrated in the
    standard classification sense. 
    }
    \label{fig:B-collapse-intrp}
\end{figure}

{\bf Summary of Contributions.}
We empirically show that LLMs \emph{are}
semantically-calibrated surprisingly often, for certain settings and types of questions.
We offer a candidate theoretical mechanism to explain how this %
calibration emerges from standard LLM training (that does not explicitly encourage it), and discuss under which settings and for which questions we expect it.
The basic prediction of our theory is that semantic calibration is likely to hold when
(1) the model is a base LLM, and 
(2) the model is able to \emph{immediately} predict the probability that its answer will land in a given semantic class, even before it has started to generate it. Specifically, this immediate prediction should be ``easy to learn'' in the sense that, for example, the model could be LoRA-adapted to perform it.
Intuitively, in order to be semantically calibrated, the model must
``know'' how likely it is to generate a ``Paris''-type answer,
before it has determined exactly how it will phrase its answer.
This theoretical insight leads to a number of practical predictions
about which models and tasks should be semantically calibrated, which
we then test experimentally.

{\bf Organization.} We start by formally defining the notion
of calibration we consider in \Cref{sec:semantic-calibration}.
In \Cref{sec:theory}, we introduce our proposed theoretical mechanism for
emergent calibration, and state our formal results.
In \Cref{sec:theory-exp-bridge}, we apply the theory
to make three concrete predictions about when LLMs are semantically calibrated,
and in~\Cref{sec:experiments}, we experimentally test these predictions.

\arxiv{
\paragraph{Practical Takeaways (informally):}
\begin{itemize}
    \item {\bf Sampling multiple generations and computing semantic confidence is a
    theoretically-principled way to measure uncertainty in LLMs.}
    Under certain conditions (which we describe),
    this confidence will be calibrated without any additional training.
    \item {\bf Base LLMs evaluated on ``reasonably in-distribution'' data are almost always semantically calibrated},
    regardless of model size or accuracy, when using direct generation (not chain-of-thought). See \Cref{fig:ece-grid}.
    \item {\bf Chain-of-thought reasoning typically breaks calibration.}
    \item {\bf Models trained with anything but a sequence-level proper loss are unlikely to be calibrated.}
    This includes many Instruct models post-trained with RLHF, DPO, or RLVR.
\end{itemize}
}

\begin{figure}[t!]
    \centering
    \includegraphics[width=\linewidth]{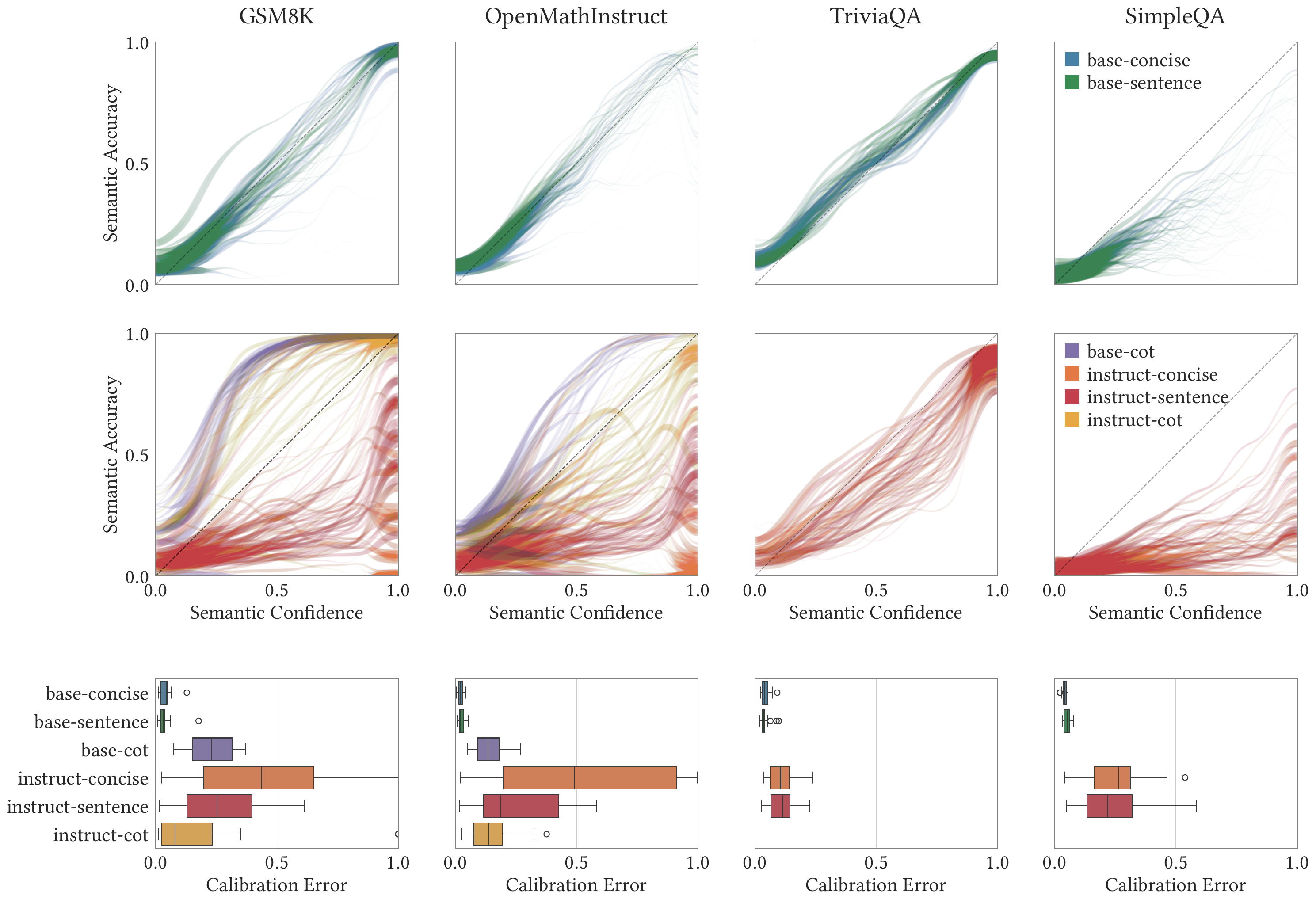}
    \caption{
    {\bf Semantic Calibration of LLMs.} 
    Overlaid reliability diagrams evaluating semantic calibration of
    Qwen, Gemini, Mistral, and Llama-family models of sizes
    from 0.5B to 70B, on four datasets.
    Each model is prompted to respond in one of three different styles:
    a single word (``concise''), a complete sentence (``sentence''),
    or using chain-of-thought (``CoT''). 
    This yields 6 color-coded configurations for each model:
    (model-variant, response-style) $\in$
    \texttt{\{Base, Instruct\} $\x$ \{Concise, Sentence, CoT\}}. 
    We group these configurations into two rows based on our theoretical predictions.
    {\bf First row (predicted calibrated):} Reliability diagrams of all configurations
    predicted to be confidence-calibrated according to our theory:
    base models with \textcolor{AppleBlue5}{concise} or \textcolor{AppleGreen5}{sentence} 
    response types.
    {\bf Second row (not predicted calibrated):}
    Configurations which need not be calibrated according to our theory:
    post-trained instruct models with any response type: \textcolor{AppleOrange5}{concise}, \textcolor{AppleRed5}{sentence}, \textcolor{AppleYellow5}{chain-of-thought}; 
    and 
    \textcolor{ApplePurple5}{base models with chain-of-thought}.
    {\bf Third row:}
    Box plots summarizing the distribution of calibration errors for each of the 6 configurations.
    Only the first two configurations
    (\textcolor{AppleBlue5}{base-concise} and \textcolor{AppleGreen5}{base-sentence})
    are reliably well-calibrated, as predicted by our theory.
    Individual reliability diagrams for all experiments are in \Cref{app:encyclopedia}.
    }
    \label{fig:ece-grid}
\end{figure}

\begin{figure}[t]
    \centering
    \includegraphics[width=0.99\linewidth]{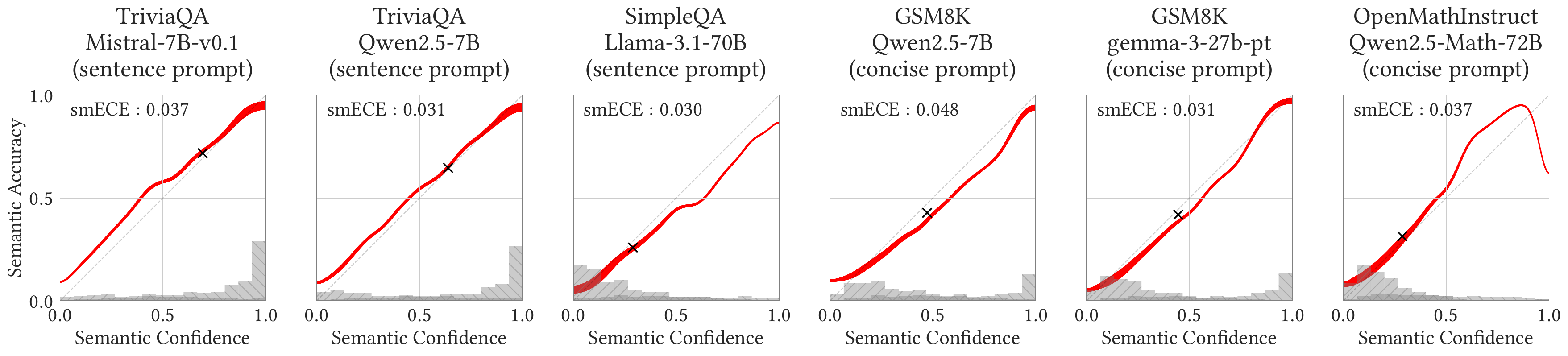} \\
    \vspace{-2mm}
    \caption{
Reliability diagrams demonstrating
\emph{semantic confidence-calibration}
of base (pretrained-only) LLMs across
various combinations of datasets, models, and prompts. 
Calibration error measured with SmoothECE (smECE),
average confidence and accuracy marked with a black cross,
and density of semantic confidences shown in gray histogram;
details in Appendix~\ref{app:rel-diagram}.
}
    \label{fig:calib_works_grid}
\end{figure}

\section{Semantic Calibration and $B$-Calibration}
\vspace{-2mm}
\label{sec:semantic-calibration}

We now informally describe our framework; formal definitions follow in \Cref{subsec:notation}.
The core of our approach is a collapsing function $B$ which post-processes
the LLM's raw text outputs, mapping each generation to one of a finite set of classes.
Of particular interest are \emph{semantic collapsing functions}\footnote{To implement this function, we use a strong auxiliary LLM prompted to extract a canonical short answer from a long-form string. Details in \Cref{app:experiments}.}, which we focus on now.
As illustrated in \Cref{fig:B-collapse-intrp},
a semantic collapsing function implicitly transforms the LLM into an \emph{LLM-induced semantic classifier}: 
For a given question, the classifier's output is a distribution over semantic classes, 
whose probabilities can be empirically estimated by sampling multiple generations from the LLM and applying $B$ to each.
From this distribution, we define the
semantic confidence as the probability of the most-likely semantic class,
and the semantic accuracy as whether the most-likely semantic class matches the ground truth's semantic class.
The LLM is \emph{semantically confidence-calibrated}
if these confidences and accuracies are calibrated across a dataset---e.g.,
among questions with 70\% semantic confidence, the average semantic accuracy is also 70\%.
This definition coincides with \citet{lamb2025semantic}'s definition of ``Empirical Semantic Confidence''
when applied to the full distribution.
For example, \Cref{fig:calib_works_grid} measures calibration of several models using this approach (full experimental details in \Cref{sec:experiments}).

\subsection{Notation and Setup}
\label{subsec:notation}

We now establish the notation used throughout the paper.
We assume that our semantic collapsing function outputs at most $K \in \N$ classes, which we represent by the set of indices $[K] \equiv \{1, \dots, K\}$. We allow $K$ to be arbitrarily large. We identify these classes with the set of standard basis vectors $\eK \subset \mathbb{R}^K$. The set of probability distributions over a finite set $S$ is denoted $\Delta(S)$. For convenience, we use the shorthand $\Delta_K \equiv \Delta([K])$ for the probability simplex over the $K$ classes.

\textbf{Language Model and Data.}
Let $\cV$ be the model's vocabulary. We assume throughout that the evaluation data %
comes from a ground-truth distribution $\cD$ over prompt-completion pairs $(x, y) \in \cV^* \times \cV^N$, where $N$ is a maximum generation length. An LLM is a function $p_\theta: \cV^* \to \Delta(\cV^N)$ that maps a prompt $x$ to a distribution over output strings. We use conventional notation:
$p_x \equiv p_\theta(\cdot \mid x)$ is the entire distribution over sequences for a given prompt, 
so we can denote $p_x(z) = p_\theta(z \mid x)$ as the probability of a specific sequence $z$.
The conditional probability of the next token is denoted $p_\theta(z_i \mid x, z_{<i})$.
To distinguish model outputs from the dataset, we use $z \in \cV^N$ for generated strings and $y \in \cV^N$ for ground-truth completions from $\cD$.

\textbf{Collapsing function.}
The core of our framework is the collapsing function $B: \cV^* \times \cV^N \to [K]$ that classifies a given prompt-completion pair into one of $K$ categories.
In our theory, $B$ is allowed to be arbitrary, but we often will think of it as a 
\textbf{``semantic collapsing'' function}, grouping many different strings into a single semantic class, as visualized in \Cref{fig:B-collapse-intrp}. 
An example of such a function is described in \Cref{app:experiments}. 
For convenience, we write $B_x(z) \equiv B(x, z)$ to emphasize its role as a classifier for outputs $z$ given a fixed prompt $x$.

\subsection{Confidence Calibration}

We first recall the relevant definitions of calibration in the multi-class setting (for a unified treatment, see \citet[Section 2]{gopalan2024computationally}). In the $K$-class setting, 
classifiers output values $c \in \Delta_K$
and the true labels take values $y \in \eK$ (one-hot encodings).
Calibration is a property defined for \emph{any} joint distribution of 
prediction-label pairs $(c, y) \in \Delta_K \times \eK$, regardless of whether it was 
generated by a classifier.
We will focus primarily on \emph{confidence calibration}, which only considers the probability assigned to the predicted class; however, we provide analogous results for full calibration in \Cref{app:b_cal}.
The following definition is standard:
\begin{definition}[Confidence-calibration]
    A distribution $\cD$ over prediction-output pairs
    $(c, y) \in \Delta_K \times \eK$
    is \emph{perfectly confidence-calibrated} if
    $$\E_{\substack{(c, y) \sim \cD}}
    \left[y_{k^\star} - c_{k^\star} \mid c_{k^\star} \right] \equiv 0 \text{ where } k^\star \gets \argmax_{k \in [K]}c_k .$$ \label{def:confidence_calibration}
\end{definition}
The definition depends crucially on the distribution $\cD$.
In this work we take $\cD$ to be the evaluation distribution of interest (e.g. TriviaQA, GSM8k, etc),
unless otherwise specified.

\paragraph{From Language Model to Categorical Predictor}
For a given prompt $x$, we obtain a distribution over $K$ categories by
pushing-forward the LLM's output distribution $p_\theta(\cdot \mid x)$ 
via the function $B_x$.
Specifically, the distribution over categories
$\pi_x := B_x \sharp p_x \equiv B_x \sharp p_\theta( \cdot \mid x)$
assigns to each category $k \in [K]$ the sum of probabilities of all strings $z$ that $B_x$ maps to that category:
\begin{equation}
\label{eqn:pushforward_dist}
(B_x \sharp p_x)(k) = \Pr_{z \sim p_\theta(\cdot \mid x)}[B_x(z) = k] = \sum_{z \,:\, B_x(z) = k} p_\theta(z \mid x).
\end{equation}
This process transforms the original prompt-answer pair $(x, y)$ from the dataset $\cD$ into a pair suitable for calibration analysis: $(B_x \sharp p_x, B_x(y))$, where $B_x \sharp p_x$ is the model's predicted distribution over categories and $B_x(y)$ is the ground-truth category.
Now, we say that the model $p_\theta$ is $B$-confidence-calibrated if the induced distribution over $(B_x \sharp p_x, B_x(y))$ is confidence-calibrated. 
That is, $B$-confidence-calibration means if the generated and ground-truth answers are both post-processed by $B$, then the resulting $K$-way-classifier is confidence-calibrated. 

\begin{definition}[$B$-confidence-calibration]
\label{def:B-conf-cal}
The model $p_\theta$ is
\emph{$B$-confidence-calibrated} with respect to distribution $\cD$ if
the induced distribution over pairs $(B_x \sharp p_x, B_x(y)) \in \Delta_K \x [K]$
is perfectly confidence-calibrated (per \Cref{def:confidence_calibration}).
\end{definition}

Our entire framework is well-defined for any function $B$,
though we usually choose $B$ to be a semantic-collapsing function.
In general, an LLM might be $B$-confidence-calibrated for some choices of $B$, but not others---one goal of our theory is to understand why.

\arxiv{
Our entire framework is well-defined for any arbitrary computable function $B$,
though we usually choose $B$ to be a semantic-collapsing function, 
in which case we refer to $B$-calibration as simply \emph{semantic calibration}.
In general, an LLM might be $B$-calibrated for some choices of $B$, but not others---
one goal of our theory is to understand why.
}

\arxiv{
\subsection{Remarks}
There are several subtleties in the above setup.
First, for a given question,
we would experimentally measure accuracy not by sampling the LLM once (as is often done in practice),
but rather by sampling it many times and checking correctness of
only the most-frequent semantic answer.\footnote{
Taking majority vote is identical to the ``self-consistency'' method proposed 
by \citet{wang2023selfconsistency} to improve performance.
The connection to classifier calibration lends a principled justification to this method.
}
This choice may seem unnatural for language models, but
it is exactly the notion of accuracy which
arises from treating the LLM-induced-classifier as a classifier:
we consider only the argmax prediction.
Second, although we formally defined the LLM-induced-classifier
as outputting a distribution over classes,
we never explicitly materialize this output distribution in our experiments (e.g. \Cref{fig:B-collapse-intrp}).
Rather, we only need to obtain \emph{samples} from the semantic class distribution,
which are sufficient for our purposes.
}

\section{Theoretical Mechanism}
\label{sec:theory}

Our conjectured mechanism for emergent calibration 
builds on the work of \citet{blasiok2023when,blasiok2024loss} 
which connects the \emph{statistical} property of calibration
to the \emph{optimization} property of local loss optimality.
The core intuition is that a miscalibrated model implies the existence of a ``simple'' perturbation 
to the model that would reduce its test loss.
For example, suppose an LLM is semantically miscalibrated in the following way:
on questions where it is 70\% semantically-confident, it is on average only 60\% accurate.
Then, an obvious way to improve the LLM's test loss is:
whenever the original LLM was 70\% semantically confident,
it should downweight the probability mass it places on all strings in its majority semantic class, thereby decreasing its confidence.
We argue that base LLMs, trained to minimize cross-entropy loss, should not leave such ``easy wins'' on the table, and thus should be well-calibrated.

\looseness=-1
This example reveals some of the subtlety in the LLM setting: unlike standard classifiers,
the LLM does not explicitly output its [semantic] confidences.
Thus to implement such a loss-improving perturbation during pretraining,
the LLM must implicitly ``know'' its semantic confidence for a given question even before generating its answer---in order to know what type of upweighting/downweighting
of answer strings is required.
In settings where the LLM does not ``know'' its semantic confidences (informally),
we may expect poor calibration---we will see this aspect in both our theory and experiments.
A technical overview of our results is in \Cref{sec:proposed_mech},
followed by formal theorem statements in \Cref{sec:theory-loss} and \Cref{sec:theory-ar}.
All proofs are deferred to \Cref{app:theory}.

\subsection{Conjectured Mechanism: Overview}
\label{sec:proposed_mech}

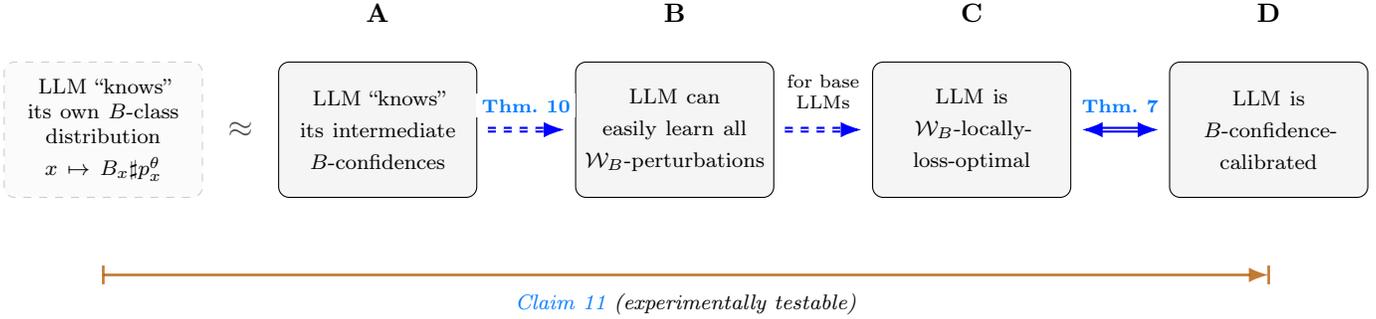
\begin{figure}[t]
\centering
\makebox[\textwidth]{%
\begin{tikzpicture}[
    node distance=0.3cm and 0.6cm,
    box/.style={draw, rounded corners, text width=2.4cm, align=center, minimum height=1.8cm, fill=gray!8},
    box faded/.style={draw, dashed, rounded corners, text width=2.4cm, align=center, minimum height=1.8cm, fill=gray!3, draw=gray!40},
    theory arrow/.style={-{Latex}, double, thick, draw=blue, shorten <=4pt, shorten >=4pt},
    heuristic arrow/.style={-{Latex}, double, thick, draw=orange!70!black, dashed, shorten <=4pt, shorten >=4pt},
    theory dashed/.style={-{Latex}, double, thick, draw=blue, dashed, shorten <=4pt, shorten >=4pt},
    claim arrow/.style={orange!70!black, line width=1.0pt, opacity=0.8},
    box label/.style={font=\bfseries, yshift=0.1cm} %
]
\node(nodeE)[box] at (0,0) {
    \footnotesize LLM is \mbox{$B$-confidence}-calibrated
};
\node(nodeD)[box, left=of nodeE, xshift=-0.7cm] {
    \footnotesize LLM is \mbox{$\cW_B$-locally}-loss-optimal
};
\node(nodeC)[box, left=of nodeD, xshift=-0.7cm] {
    \footnotesize LLM can easily learn all \mbox{$\cW_B$-perturbations}
};
\node(nodeB)[box, left=of nodeC, xshift=-0.7cm] {
    \footnotesize LLM ``knows'' its
    intermediate \mbox{$B$-confidences} %
};
\node(nodeA)[box faded, left=of nodeB, xshift=-0.4cm] {
    \footnotesize LLM ``knows''
    its own $B$-class distribution\\ $x \mapsto B_x\sharp p^\theta_x$
};

\path (nodeA) -- (nodeB) node[midway, text=black, text opacity=0.7, font=\large] {$\approx$};

\draw[theory dashed] (nodeB) -- (nodeC)
    node[black, midway, above, yshift=0.15cm, font=\scriptsize\bfseries, fill=white, inner sep=2pt] {
        \Cref{thm:ar-b-cal}
    };

\draw[theory dashed] (nodeC) -- (nodeD) 
    node[black, midway, above, yshift=0.15cm, text width=1.5cm, align=center, font=\scriptsize] {
        for base LLMs
    };

\draw[theory arrow, {Latex}-{Latex}] (nodeD) -- (nodeE) 
    node[black, midway, above, yshift=0.15cm, font=\scriptsize\bfseries, fill=white, inner sep=2pt] {
        \Cref{thm:calibration_equivalences}
    };

\node[box label, above=of nodeB] {A};
\node[box label, above=of nodeC] {B};
\node[box label, above=of nodeD] {C};
\node[box label, above=of nodeE] {D};

\def\claimarrowyshift{-0.9cm}
\def\claimstopperheight{0.25cm}

\draw[claim arrow] 
    ([yshift=\claimarrowyshift]nodeA.south) -- ++(0,-\claimstopperheight);
\draw[claim arrow, -{Latex}] 
    ([yshift=\claimarrowyshift-0.125cm]nodeA.south) -- ([yshift=\claimarrowyshift-0.125cm]nodeE.south);
\draw[claim arrow] 
    ([yshift=\claimarrowyshift]nodeE.south) -- ++(0,-\claimstopperheight);
    
\path (nodeA.south) -- (nodeE.south) node[midway, below=1.15cm, font=\footnotesize\itshape, text=black] {
   \Cref{claim:main_heuristic} (experimentally testable)
};

\end{tikzpicture}%
}
\caption{
Conjectured Mechanism for Semantic Calibration.
Implications have varying levels of support:
the \textcolor{blue}{solid blue} arrow
(\protect\tikz[baseline=-0.5ex]{\protect\draw[blue, double, thick, latex-latex] (0,0) -- (0.7,0);})
has a formal proof;
the \textcolor{blue!70!black}{dashed blue} arrows
(\protect\tikz[baseline=-0.5ex]{\protect\draw[blue!70!black, dashed, double, thick, -latex] (0,0) -- (0.6,0);})
have proofs of ``morally similar'' (but weaker) implications.
\Cref{claim:main_heuristic}
encompasses the full chain of implications,
and has experimental support.
}
\label{fig:mechanism}
\end{figure}

\Cref{fig:mechanism} illustrates our conjectured mechanism. 
There are three main steps in the conjecture, 
with 
different degrees of evidence for each step.
For the first step, we have a fully rigorous proof.
For the other two steps, we have partial theoretical evidence:
proofs of weaker claims which are ``morally similar'' to our conjectured claim.
Finally, we have experimental evidence for our overall conjecture
(presented later in \Cref{sec:experiments}).
We outline each step below,
following \Cref{fig:mechanism} from right-to-left.

{\bf (C)}
\protect\tikz[baseline=-0.5ex]{\protect\draw[blue, double, thick, latex-latex] (0,0) -- (0.7,0);}
{\bf (D):}
The first step of our argument, described in more detail in \Cref{sec:theory-loss},
is a general equivalence between calibration and local loss optimality.
We say an LLM is locally-loss-optimal if its test loss cannot be improved 
by post-processing its output distribution via any function in some given set of perturbation functions
(formally, \Cref{def:local_optimality_condition}).
For a particular choice of perturbation functions,
this turns out to exactly characterize $B$-calibration.
We prove in \Cref{thm:calibration_equivalences} that for any choice of
collapsing function $B$, $B$-confidence-calibration is \emph{equivalent}
to local-loss-optimality with respect to
a corresponding family of perturbations, denoted $\mathcal{W}_B$.
Roughly speaking, this family $\cW_B$ consists of perturbations like our earlier example:
``If the $B$-class-confidence was 70\%, then downweight the probability 
of generating all strings in the majority $B$-class.''
Overall, \Cref{thm:calibration_equivalences} tells us that if we want to understand
when LLMs are $B$-confidence-calibrated, we can equivalently understand
which types of perturbations LLMs are loss-optimal with respect to.

{\bf(B)}
\protect\tikz[baseline=-0.5ex]{\protect\draw[blue, dashed, double, thick, -latex] (0,0) -- (0.7,0);}
{\bf (C):}
At this point, we invoke an informal
assumption proposed in \citet{blasiok2023when}, and likely folklore much earlier:
we assume that base LLMs are nearly locally-loss-optimal on their pretraining distribution, w.r.t.
any perturbation that is ``easy'' for the LLM to learn.
We state this assumption more precisely as \Cref{claim:w-general} in the Appendix.
\citet[Theorem 1.2]{blasiok2024loss} offers partial theoretical justification for this claim, by proving that, intuitively, if small models can represent a set of perturbations, then ERM over a family of slightly larger models yields local loss optimality w.r.t. these perturbations; this serves as a approximate representational analog of the desired assumption.
The intuition is that pretraining not leave any easy wins
on the table: if a simple (i.e. easily-learnable) perturbation could have
improved the test loss, the LLM would have learned it during training.\footnote{\label{footnote:multical}Technically, we need local-loss-optimality
not only for the overall pretraining distribution, but also for each evaluation distribution individually (TriviaQA, GSM8k, etc), since we are evaluating calibration on individual distributions.
We will however assume that the latter holds
(which is plausible if each evaluation distribution is a reasonably-sized sub-distribution of the pretraining distribution on which local-loss-optimality holds).} We agree with \citet{blasiok2023when} that this assumption is plausible, because it is fairly weak; it does not require that models are \emph{globally} optimal in any sense.

{\bf(A)}
\protect\tikz[baseline=-0.5ex]{\protect\draw[blue, dashed, double, thick, -latex] (0,0) -- (0.7,0);}
{\bf (B):}
From the above two points, we can conclude that a base LLM 
will be $B$-confidence-calibrated if the corresponding perturbation family $\cW_B$
is simple for the LLM to learn.
But when is $\cW_B$ simple to learn? 
This is subtle because the perturbations $\cW_B$ are 
defined over the \emph{sequence-level} probability distribution
but LLMs must implement perturbations by modifying \emph{next-token} probabilities.
For example, in order to implement a perturbation such as
``increase the probability of ultimately generating a Paris-type answer'',
the model must begin by deciding how to adjust its \emph{first token} probabilities in order to achieve this.
We bridge this gap in \Cref{thm:ar-b-cal},
by proving a representational analogue of the 
implication {\bf(A)}
\protect\tikz[baseline=-0.5ex]{\protect\draw[blue, dashed, double, thick, -latex] (0,0) -- (0.7,0);}
{\bf (B)} of \Cref{fig:mechanism}:
we show that if the LLM ``knows'' its own induced distribution over $B$-classes at each intermediate point during generation (even the very beginning),
then it can implement the associated family of perturbations $\mathcal{W}_B$ in a ``simple'' way. (Notably, this does not require the model to know the
\emph{correct answer's} $B$-class, only that of its own generation.)
Formally, we prove a circuit-complexity version of this:
the next-token probabilities of the perturbed model
can be computed with a shallow circuit given oracle access to the intermediate $B$-confidence functions, and the original next-token probabilities. In practice, we will focus primarily on the model's ability to predict its $B$-distribution at the beginning of generation (before outputting the first token). Intuitively, this is more likely to hold for straightforward questions such as ``What is the capital of France?'' than questions requiring many steps of reasoning--- we will say more about this in the following section.

Putting everything together, the overall mechanism predicts that a base LLM will be
$B$-confidence-calibrated if the LLM ``knows'' the distribution of $B$-classes of its own answers
(i.e. if it can be LoRA-adapted to immediately output this $B$-class distribution, given only the question).
When $B$ is a semantic collapsing function,
this theory naturally suggests a number of practical predictions about which models and tasks should be semantically calibrated,
which we explore and test experimentally in \Cref{sec:experiments}.
The next several sections give the formal theory supporting the mechanism we have just outlined.

\subsection{$B$-calibration and local loss optimality}

\label{sec:theory-loss}
We now setup and establish 
the equivalence between calibration and local loss optimality
(\Cref{thm:calibration_equivalences}).
We consider the sequence-level cross-entropy loss, which decomposes into the
standard autoregressive next-token log-loss:
$
\E_{(x,y) \sim \cD}[\ell(y, p_x)] = 
    \E_{\substack{(x, y) \sim \cD}}
    \left[ -\sum_{i \in [N]}
    \log p_\theta(y_i \mid y_{<i}, x) \right].
$
We will use the following notion of perturbing a probability distribution, 
known as \emph{exponential tilting} \citep[Chapter 11]{cover1999elements},
which turns out to be the appropriate notion\footnote{The appropriate notion of perturbation
depends on the loss function via convex duality; see \Cref{subsec:Calib_loss_optimality} for more details.} for the cross-entropy loss.

\begin{definition}[Perturbation operator]\label{def:perturbation_operator}
Given a distribution $f \in \Delta(\cV^N)$ over sequences,
and a signed measure $\mu \in \R^{|\cV^N|}$,
define the perturbed distribution $(f \star \mu) \in \Delta(\cV^N)$ as:
\begin{align}
\forall z \in \cV^N: \quad
(f \star \mu)[z] &:= 
\mathrm{softmax}\big(\mu[z] + \log f[z] \big).
\end{align}
\end{definition}

This is an operation defined over probability distributions.
We can use it to perturb a model in the following way.
Recall that for a model $p_\theta$
and prompt $x \in \cV^*$,
we write $p_x \equiv p_\theta(\cdot \mid x)$.
\begin{definition}[Perturbed model] 
\label{def:perturb-model}
Given a model $p_\theta: x \mapsto p_x$
and a perturbation function
$w: \cV^* \x \Delta(\cV^N) \to \R^{|\cV^N|}$,
we define the perturbed model $(p_\theta \star w) \equiv \tilde{p}$ as
\[
\tilde{p}: x \mapsto (p_x \star w_x)
\quad
\mathrm{where} ~w_x\equiv w(x, p_x) \in \R^{|\cV^N|}
\]
\end{definition}
That is, a perturbation function $w$ takes as input the prompt $x$
and the model's generative distribution $p_x$,
and defines how to perturb the generative distribution for that specific prompt.
We can now define local loss optimality with respect to an
arbitrary family of perturbation functions $\cW$.
\begin{definition}[$\cW$-local loss optimality]
We say that model $p_\theta$ is \emph{$\mathcal{W}$-locally-loss-optimal} on distribution $\cD$ if
\begin{equation*}
    \forall w \in \mathcal{W}:
    \quad \E_{(x,y) \sim \cD}[\ell(y, p_x)] \le \E_{(x,y) \sim \cD}[\ell(y, p_x \star w_x)]
    \quad
    \mathrm{where} ~w_x\equiv w(x, p_x)
    ~,~ p_x \equiv p_\theta(\cdot \mid x).
\end{equation*}
\label{def:local_optimality_condition}
\end{definition}

Next we define a 
specific class of perturbations $\WBconf$
which characterize $B$-confidence-calibration.
The formal definition is somewhat technical, based on the language of weighted calibration
developed in \citet{gopalan2024computationally}.

\begin{definition}[Semantic Perturbation Function Classes]
\label{def:perturbation_classes}
Given an arbitrary collapsing function $B_x(z) \in [K]$,
we define the class
$\WBconf$
of perturbation functions
$w_\tau(x, p_x) \in \R^{|\cV^N|}$ as follows.
Each function $w_\tau$
is indexed by a map $\tau: [0, 1] \to [-1, 1]$,
and generates a perturbation vector in $\R^{|\cV^N|}$ 
based on the prompt $x$ and the model's predictive distribution $p_x$.
\begin{align*}
\WBconf &:= 
\{ w_\tau \mid \tau: [0, 1] \to [-1, 1] \}
\end{align*}
where $w_\tau(x, p_x) \in \R^{|\cV^N|}$ is defined componentwise as follows.
For index $z \in \cV^N$,
\begin{align*}
w_\tau(x, p_x)[z] =  \tau\big( \pi_x[k^*] \big)
\cdot 
\mathds{1}\{B_x(z) = k^*\}, 
~\mathrm{where}~ \pi_x := B_x \sharp p_x,
\quad \mathrm{and}~~ k^* \gets \argmax_{k \in [K]} \pi_x [k].
\end{align*}

\end{definition}

These perturbations implement a re-mapping of the
$B$-class-confidences governed by the function $\tau$.
For example, if $B$ is a semantic collapsing function, then
a perturbation $w_\tau$ could implement the change 
``whenever the semantic confidence of a question is 70\%, decrease the semantic confidence to 60\%,
by downweighting the probability of all strings in the top semantic class.''
Unpacking the notation in \Cref{def:perturbation_classes}: $\pi_x \in \Delta_K$ 
is the model's distribution over $B$-classes, 
$k^* \in [K]$ is the top $B$-class,
$z \in \cV^N$ is a string,
and $w_\tau(x, p_x)[z]$ represents how much 
the perturbed model should up-weight the answer string $z$,
for question $x$. Then,
\begin{align*}
\underbrace{w_\tau(x, p_x)[z]}_{\textrm{Desired perturbation to $p_\theta(z \mid x)$}} = 
\tau\big(
\underbrace{
\pi_x[k^*]}_{\textrm{$B$-confidence}}
\big)
\cdot 
\underbrace{\mathds{1}\{B_x(z) = k^*\}}_{\textrm{$z$ in top $B$-class?}}.
\end{align*}

We can now state the main result of this section
(see \Cref{app:theory} for all proofs).
\begin{theorem}[Equivalence of Calibration and Local Loss Optimality]
\label{thm:calibration_equivalences}
For all models $p_\theta$, collapsing functions $B$ and distributions $\cD$, the following are equivalent:
\begin{enumerate}
    \item The model $p_\theta$ is perfectly \text{$B$-confidence-calibrated} on $\cD$ 
    \item The model $p_\theta$ is \text{$\WBconf$-locally-loss-optimal} on $\cD$.
\end{enumerate}
    
\end{theorem}

\begin{remark}
\Cref{thm:calibration_equivalences} states a simplified version
of our full theoretical results, for the sake of exposition.
\Cref{thm:calibration_equivalences} only characterizes
perfect confidence-calibration,
but it is possible to show a much more robust equivalence:
it turns out that a model is ``close to'' $B$-calibrated if and only if
it is ``close to'' locally-loss-optimal in the appropriate sense.
We state and prove this generalized version as \Cref{thm:bcal_calibration_equivalences}
in \Cref{app:theory}, where we also generalize to
allow any arbitrary proper-loss $\ell$, and any notion of weighted-calibration
(including canonical calibration and confidence calibration).
\end{remark}

\subsection{Which Perturbations are Easy to Learn Autoregressively?}
\label{sec:theory-ar}
It remains to understand
when the perturbation class $\WBconf$ is easy for an LLM to learn
(box {\bf (B)} in \Cref{fig:mechanism}).
Although we cannot currently fully answer this question, we can gain insight by studying a simpler question of representation:
when is a perturbation class $\WBconf$ ``easy'' for the LLM to represent
(for example, as a small circuit on top of the original LLM)?
The main remaining challenge is that perturbations 
are defined on probability distributions over \emph{sequences} (\Cref{def:perturbation_operator}),
whereas autoregressive models must implement perturbations \emph{token-by-token}.
Fortunately, for perturbations in $\WBconf$, it turns out
the perturbed next-token distribution can be expressed as a simple re-weighting of
the LLM's original next-token distribution.
This re-weighting is governed by a set of scalar-valued functions $\{g_i\}$, defined below.
We call these functions ``intermediate $B$-confidences'', because $g_i(z_{\leq i}; x)$ is the probability mass the model places
on its most-likely $B$-class, given both the question $x$ and the response prefix $z_{\leq i}$ generated so far.
\arxiv{
For $i=0$, the function $g_0$ is simply the $B$-confidence of the model given the
question $x$.}
Thus, the difficulty of representating the sequence-level perturbation reduces to the difficulty of representing these intermediate confidences values during generation.

\begin{definition}[Intermediate $B$-Confidences]
For a given function $B: \cV^* \times \cV^N \to [K]$
and model $p_\theta$,
we define the \emph{intermediate $B$-confidences} as
the scalar-valued functions $\{g_i\}_{i \in \{0, 1, \dots, N\}}$:
\begin{align*}
g_i(z_{\leq i}; x) := 
\Pr_{z \sim p_\theta(\cdot \mid x, z_{\leq i})}[ B_x(z) = k^* ]
~~\mathrm{where} ~~
k^* \gets \argmax_{k \in [K]} (B_x \sharp p_x)[k].
\end{align*}
\label{def:ar-b-conf}
\end{definition}

We will informally say 
that the LLM ``knows'' its intermediate $B$-confidences if
the functions $g_i$ have a simple representation
(e.g. each $g_i$ is computable by a small circuit on top of the LLM).
In that case, we show in \Cref{thm:ar-b-cal}
that for any perturbation $w \in \WBconf$,
the perturbed model $p_\theta \star w$ has an only-slightly-more-complex representation than the
original model $p_\theta$.
Specifically, the perturbed model can be computed by composing a circuit $C_w$ with the
functions $g_i$. Explicit formulas are provided in \Cref{app:ar-conf-proofs}.

\begin{theorem}
\label{thm:ar-b-cal}
For all functions $B: \cV^* \times \cV^N \to [K]$
and all perturbations $w \in \WBconf$,
there exists a small circuit\footnote{Specifically, an arithmetic circuit of constant depth and $\Theta(K)$ width.} $C_w$
such that for all models
$p_\theta : \cV^* \to \Delta(\cV^N)$,
all $x \in \cV^*, z \in \cV^N$, all $i \in [N]$, 
and with
$p_x := p_\theta(\cdot \mid x)$,
$w_x := w(x, p_x)$,
the perturbed model 
$x \mapsto p_x \star w_x$ satisfies
\begin{align}
(p_x \star w_x)( z_i \mid z_{< i} ) 
\propto C_w (a, g_i(z_{\leq i}; x), g_0(x)) 
\end{align}
where the constant of proportionality is independent of $z_i$, $a := p_x(z_i\mid z_{<i})$ is the original next-token probabilities, and $g_0, g_i$ are the intermediate $B$-confidences of \Cref{def:ar-b-conf}.
\end{theorem}

Putting all the theory together, 
the message is: if the LLM ``knows'' its intermediate $B$-confidences,
then perturbations $\WBconf$ are easy to implement,
and we should expect emergent $B$-calibration.

\section{Experimental Predictions: When are LLMs calibrated?}
\label{sec:theory-exp-bridge}

Our main empirical question is:
\textit{
Under what conditions and
for which functions $B$ should we expect a pretrained LLM to be $B$-confidence-calibrated?
}

The theory of the previous section suggested an answer:
we should expect emergent $B$-confidence-calibration for a base LLM when
the LLM ``knows'' its intermediate $B$-confidences (\Cref{def:ar-b-conf}).
We simplify this (as discussed below) into an experimentally-testable heuristic:
for a given question $x$, does the LLM ``know''
the distribution of of its answers post-processed by $B$ (i.e. $B_x \sharp p_x$)?
Practically, we operationalize this
by training a small LoRA on top of the base LLM to predict the $B$-class of the answer.

\begin{claim}[Main, heuristic]
Let $(x, y) \sim \cD$ be a distribution on question-answer pairs,
let $B: \cV^* \times \cV^N \to [K]$ be a collapsing function,
and let
$p_\theta(z \mid x)$ be an autoregressive language model trained on $\cD$ with cross-entropy loss.
Then, $p_\theta$ will be $B$-confidence-calibrated on $\cD$ 
if the function $G: \cV^* \to \Delta_K$ defined as
\begin{align*}
\label{eqn:ar-learn}
G: x \mapsto B_x \sharp p_x\quad \textrm{is ``easy to learn'' for the LLM (e.g. with a LoRA adapter)}
\end{align*}
In words: the LLM should be able to accurately estimate
the distribution over semantic labels $B_x(z)$, under its own generative process,
given the question $x$.
\label{claim:main_heuristic}
\end{claim}

\begin{remark}
Importantly, \Cref{claim:main_heuristic} does not require ground-truth labels to verify.
That is, although calibration is a property of the model $p_\theta$ and the joint distribution $(x, y)$,
we manage to predict calibration using only $p_\theta$ and the marginal distribution of $x$.
\end{remark}

\begin{remark}[Heuristic Simplifications]
\Cref{claim:main_heuristic} involves two main simplifications of \Cref{thm:ar-b-cal}. Recall that \Cref{thm:ar-b-cal} considers the functions $\{g_i\}$ of \Cref{def:ar-b-conf}, for all prefix lengths $i \in [N]$.
First, 
\Cref{claim:main_heuristic} only considers the empty prefix ($i=0$)
i.e. the model's $B$-class distribution given only the question.
Intuitively, the prediction from the empty prefix is likely the most challenging, and practically, this simplification means that only one simple-to-implement probe is required. 
Second, instead of considering learnability of only the \emph{$B$-confidences} ($g_0$),
\Cref{claim:main_heuristic} considers learnability of the entire $B$-class distribution
($B_x \sharp p_x$),
which can be estimated as a standard KL loss (see Appendix~\ref{app:lora}).
Empirically, we did not find these simplifications to significantly affect the conclusions.
\end{remark}

Finally, we specialize \Cref{claim:main_heuristic} to the practical case of semantic calibration---that is, we let $B$ be a function that collapses long-form
answers into semantic equivalence classes, yielding the following:

\begin{corollary}[Main, heuristic]
LLMs trained autoregressively with cross-entropy loss will be semantically calibrated
on in-distribution data if: the model ``knows'' its own output distribution
over semantic answers,
given only the question.
\label{corr:main_heuristic_semantic}
\end{corollary}

\Cref{corr:main_heuristic_semantic} leads to the following predictions, which we verify experimentally in \Cref{sec:experiments}. 

\textbf{Prediction 1: Semantic calibration emerges from standard pretraining.} 
    When $B$ is a semantic-collapsing function,
   we expect it to be easy-to-learn in many settings:
Claim~\ref{claim:main_heuristic} only requires
that the LLM intuitively ``knows'' what types of semantic-answers
it is likely to output for a given question. 
    Thus, we \emph{should} expect emergent semantic calibration
    for a large class of pretrained LLMs, a remarkable fact not previously understood.

\textbf{Prediction 2: RL post-training can break calibration.}
We only
    theoretically predict calibration in models trained autoregressively with cross-entropy loss, that is, standard pretraining or SFT. (Cross-entropy loss is required to connect calibration with local-loss-optimality in \Cref{thm:calibration_equivalences}.) We have no reason to expect calibration in models trained in other ways, including Instruct models post-trained with RLHF, DPO, or RLVR -- although our theory does not preclude it. 
    
\textbf{Prediction 3: Chain-of-thought reasoning (CoT) can break calibration.}
    To satisfy the conditions of our theory, the 
    model must ``know'' its own distribution over semantic answers,
   even before generating the first token.
    In hard CoT setting such as math problem-solving,
    the model usually \emph{does not} know what its final answer
    will be until it has finished ``thinking''.
    Therefore, CoT is expected to break our mechanism for calibration.
    Notably, what makes CoT powerful
    (allowing the model to leverage more
    compute to produce a better answer than it could have produced immediately)
    is exactly what makes our mechanism of calibration fail.

\begin{remark}
These predictions are not entirely novel; some versions of them have been made in prior works,
with varying degrees of evidence.
Our contribution is providing a unified theoretical explanation of these phenomena,
and more conclusive experimental evidence.
\end{remark}

\section{Experiments}
\label{sec:experiments}

In this section, we experimentally test the predictions of our theory 
on real models and datasets.
Full experimental details are in \Cref{app:experiments}.

\paragraph{Datasets.}
We focus on open-ended question-answer (QA) settings,
since calibration for multiple-choice QA is
already well-studied \citep{kadavath2022language, zhu-etal-2023-calibration},
and a special case of our results.
We evaluate on four datasets:
(1) GSM8K \citep{cobbe2021gsm8k} containing grade-school math word problems,
(2) OpenMathInstruct-2 \citep{toshniwal2024openmathinstruct}
containing primarily\footnote{OpenMathInstruct-2 also contains $\sim$16\% of problems derived from GSM8K. We used OpenMathInstruct-2 as a large set of challenging math problems with a permissible license.} competition-level math problems
synthetically derived from MATH \citep{hendrycks2021measuring},
(3) TriviaQA \citep{joshi-etal-2017-triviaqa} containing trivia questions,
and (4) SimpleQA \citep{wei2024measuring} containing factual questions
selected to be hard for GPT 3.5/4o.
Notably, the space of possible semantic answers is very large in all these settings.

\paragraph{Sampling and Evaluation.}
All of our experiments include $5$-shot examples in the prompt.
We compare three different prompts, designed to elicit different styles
of responses from the model: ``concise'' (answer in a single word/phrase), 
``sentence'' (answer in a complete sentence), and ``chain-of-thought (CoT)''.
The few-shot examples are formatted in the desired style
(e.g. for the ``sentence'' type, the few-shot examples have complete sentence answers).
For each question evaluated, we construct the appropriate $5$-shot prompt,
sample $M=50$ responses from the LLM at temperature $1$,
and then apply the semantic collapsing function (described below) to each response.
To measure calibration error, we use the SmoothECE metric\footnote{We use a particular bandwidth choice; see \Cref{app:experiments} for details.} \citep{blasiok2024smooth}.

\paragraph{Semantic Collapsing.}
We implement the semantic collapsing function
differently depending on the dataset type  and response type.
Briefly, for math settings with ``concise'' and ``chain-of-thought'' prompt types,
we just extract the final answer from the generated string using regex matching,
and then perform basic string normalization. For other settings, we use 
a strong LLM (Qwen3-14B-Instruct) to extract and cluster canonical answers
from long-form generations.
See \Cref{app:experiments} for more details.

\subsection{Experimental Results}
\label{sec:experimental_results}
We evaluate semantic calibration of Qwen, Gemini, Mistral, and Llama-family models,
of varying sizes from 0.5B to 72B,
for base and instruct variants, using each of the 3 response styles,
on all 4 question-answer datasets.
This yields over 650 
evaluation experiments,
which we compile into \Cref{fig:ece-grid}
by overlaying their reliability diagrams.
The box-plots in the bottom row of \Cref{fig:ece-grid} show the
distribution of calibration errors in aggregate for each dataset and configuration.
We will use this condensed figure to discuss our experimental predictions.
The full list of models is in \Cref{app:list-of-llms}
and disaggregated results are 
reported
in \Cref{app:encyclopedia}.

\textbf{Prediction 1: Semantic calibration emerges from standard pretraining.}
Our theory predicts that base models, in non-CoT settings, should be semantically calibrated.
The top row of \Cref{fig:ece-grid} shows reliability diagrams for all such models we evaluated
(configurations \textcolor{AppleBlue5}{base-concise} and \textcolor{AppleGreen5}{base-sentence}),
and we observe nearly all of these experiments are well-calibrated.
Notably, semantic calibration does not depend significantly on model size for base models:
even small models ($\leq$~1B) are remarkably calibrated; 
see \Cref{app:additional_experimental_results} for a more in-depth look at this aspect.
Models are also well-calibrated regardless of the response style
(``sentence'' vs.\ ``concise''),
supporting our theory that semantic calibration depends not on
the specific phrasing of the answer,
but rather on whether the model ``knows''
its semantic class distribution before starting to generate.

\textbf{Prediction 2: Post-training can break calibration.}
The middle row of \Cref{fig:ece-grid} includes reliability diagrams for instruct post-trained models,
for all three response types.
Many of these settings are miscalibrated, typically overconfident (i.e.\ a curve below the diagonal), as expected from a reward-maximizing RL objective.
\Cref{fig:sft_dpo} takes a closer look at the effect of different types of
instruction-tuning on calibration.
We compare three models from the same lineage: 
a base model (Mistral-7B-v0.1),
a version of it post-trained via instruction supervised finetuning (SFT, zephyr-7b-sft-full) \citep{tunstall2023zephyr},
and a version post-trained via both SFT and Direct Preference Optimization (DPO, zephyr-7b-dpo-full) \citep{rafailov2024directpreferenceoptimizationlanguage}.
The DPO model (not trained with a proper loss) is significantly miscalibrated,
while the SFT-only model and the base model 
(both trained with proper losses) are better calibrated.

\textbf{Prediction 3: CoT reasoning can break calibration.}
The middle row of \Cref{fig:ece-grid} shows CoT with both \textcolor{ApplePurple5}{base} and \textcolor{AppleYellow5}{instruct} models,
which are poorly calibrated in the math settings (GSM8K and OpenMathInstruct).
\textcolor{ApplePurple5}{Base-cot} responses are underconfident
(above the diagonal), 
while \textcolor{AppleYellow5}{instruct-cot} are underconfident for GSM8K, but overconfident for OpenMathInstruct, 
see last row of \Cref{fig:scaling_laws}.
Notably, this miscalibration is not inherent to math:
base models are calibrated
when asked to provide the answer immediately
(\textcolor{AppleBlue5}{base-concise} and \textcolor{AppleGreen5}{base-sentence}),
but become miscalibrated when allowed to reason (\textcolor{ApplePurple5}{base-cot}). 

\begin{remark}[Underconfidence of CoT]
The systematic underconfidence (rather than overconfidece)
of \textcolor{ApplePurple5}{base-cot} models
may seem surprising.
It is important to recall that our definitions of semantic confidence and accuracy
involve plurality vote.
For say GSM8K with CoT, the underconfidence manifests as follows:
when we sample multiple chain-of-thoughts for a given question,
the \emph{plurality answer} is almost always correct, but it is a weak plurality.
Thus the semantic accuracy is nearly 100\%, since the argmax answer is almost always right,
but the semantic confidence is significantly less than 100\%.
\end{remark}

\begin{figure}[t]
\begin{tabular}{@{}p{.66\textwidth}p{.32\textwidth}@{}}
  \raisebox{-1.03\height}{\includegraphics[width=\linewidth]{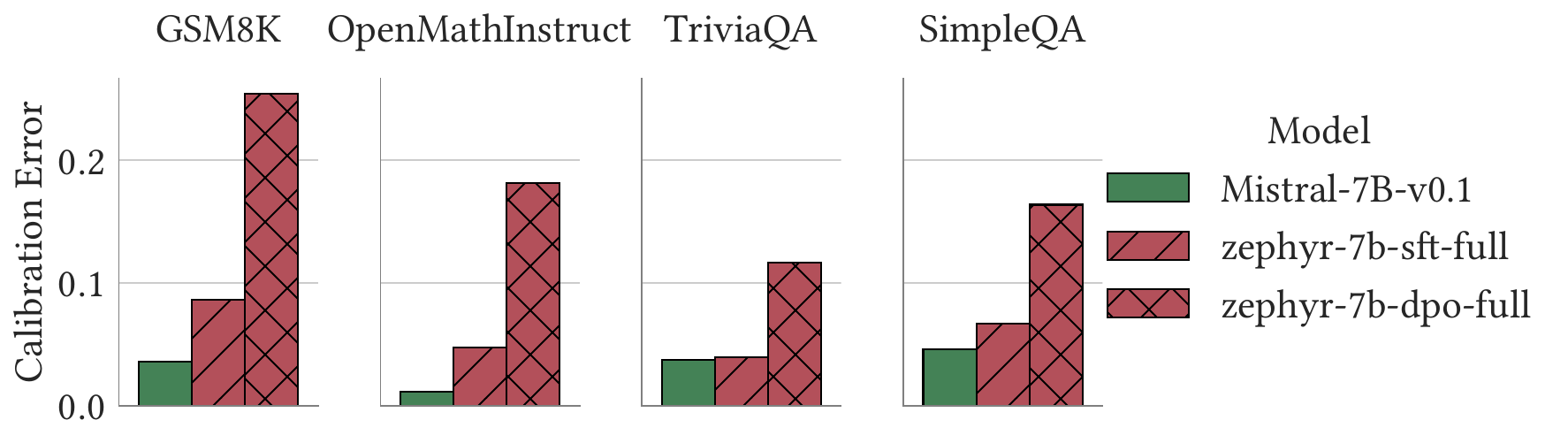}}
  &
  \raisebox{-\height}{\includegraphics[width=\linewidth]{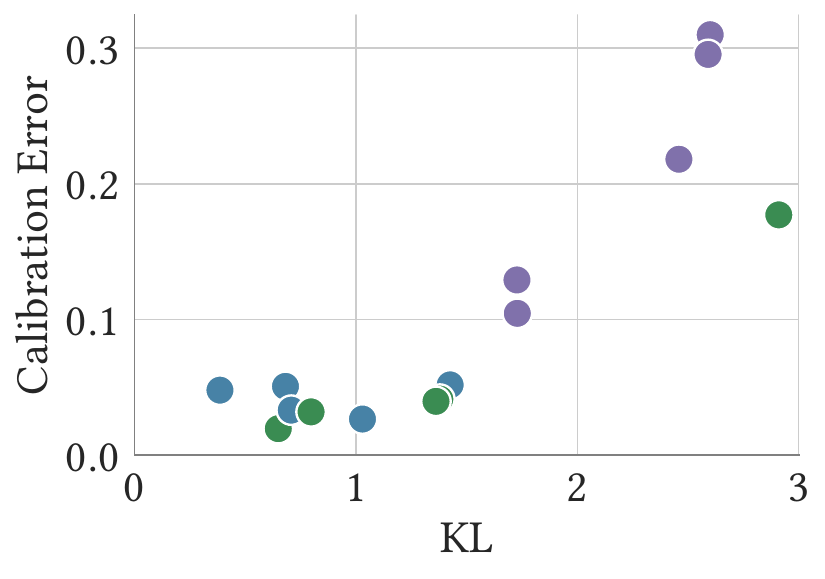}}
  \\
    \vspace{-10pt}
    \caption{
    Calibration error for three models based on Mistral-7B-v0.1:
    pretrained-only,
    instruction-supervised-finetuned, 
    and DPO-finetuned.
    Here, ``sentence'' response style, 
    see \Cref{fig:sft_dpo_full} for others.
    }
    \label{fig:sft_dpo}
  &
    \vspace{-10pt}
\caption{Testing \Cref{claim:main_heuristic} across Qwen2.5 models and response styles.
Colors per \Cref{fig:ece-grid}.
    }
    \label{fig:scatterplot}
\end{tabular}
\vspace{-10mm}
\end{figure}

\textbf{Quantitative Learnability Probe.}
Claim \ref{claim:main_heuristic} suggests an explicit experiment
to predict when a base model will be $B$-confidence-calibrated for a given choice of $B$:
can the model ``easily learn'' the function
$G: x \mapsto B_x \sharp p_x$ 
mapping a question $x$ to
the distribution over 
the model's own semantic answers for that question?
We can test this by training a small LoRA \citep{hu2022lora} on top of the model,
to directly generate the semantic class distribution $B_x \sharp p_x$ when prompted with the question~$x$.
For example, in CoT settings, this would require the LoRA to ``short-circuit''
the reasoning steps, and immediately generate the final answer that the model would have produced with CoT.
Notably, this does not require the model to produce the \emph{correct} semantic answer,
but just match its own generative distribution.
In \Cref{fig:scatterplot}, we train rank-8 LoRAs on 
Qwen2.5 base models of varying sizes (0.5B, 1.5B, 3B, 7B, 14B),
for all three response types on GSM8K.
We then compare each LoRA's KL gap to optimality (x-axis) to the underlying model's calibration error (y-axis).
The correlation agrees with our theory:
models which can easily predict their own semantic class distribution (low KL gap) 
are also well-calibrated.
Full details in \Cref{app:lora}.

\begin{figure}[t!]
    \centering
    \includegraphics[width=\linewidth]{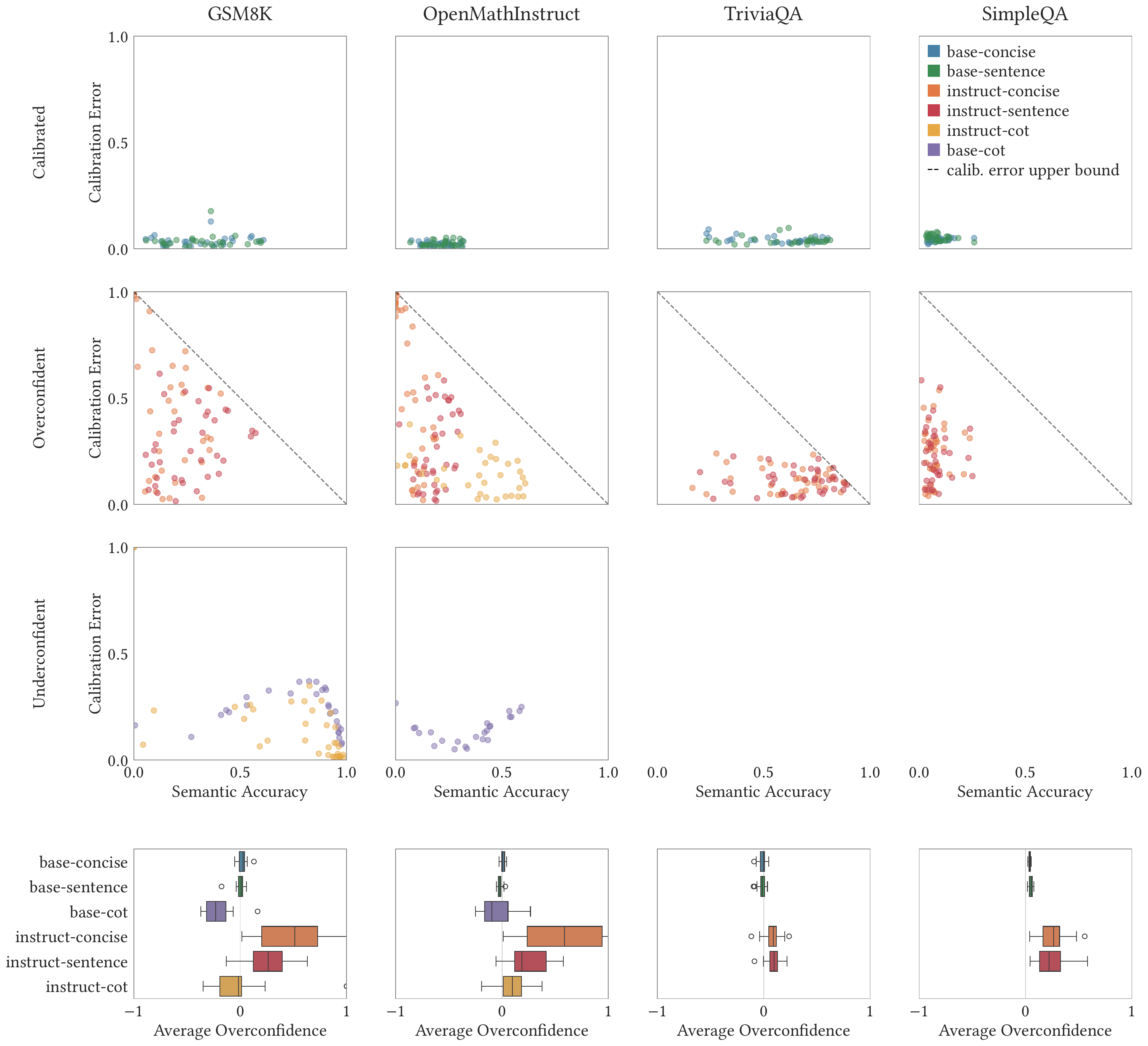}
    \caption{
    {\bf Effect of Scale:} We plot Calibration Error vs.\ Semantic Accuracy
    for all models in \Cref{fig:ece-grid}; 
    each dot represents a separate model.
    {\bf First row (predicted calibrated):} 
    In the settings our theory applies,
    we see no correlation between
    the model capability (semantic accuracy) and the calibration error.
    {\bf Second row (overconfident):}
    Configurations which we empirically observed to be mostly overconfident.
    The dashed line illustrates the upper bound on the calibration error w.r.t.\ the accuracy, for a maximally overconfident predictor.
    We see little correlation between semantic accuracy and the calibration error beyond what is dictated by the upper bound.
    {\bf Third row (underconfident):}
    We see little correlation between calibration and accuracy,
    except near the extreme when models approach perfect accuracy.
    TriviaQA and SimpleQA plots are empty because
    there are no underconfident configurations.
    {\bf Fourth row:}
    The distribution of \emph{average overconfidence}
    across models, 
    for each configuration;
    positive/negative values indicate over--/underconfidence.
    }
    \vspace{-2mm}
    \label{fig:scaling_laws}
\end{figure}

\textbf{Model Scaling Effects.}
Here, we aim to explore the effect of model scaling (parameter count, compute, data) on calibration.
Since information about training details are most often not publicly available, we use the model capability (measured with accuracy) as a proxy variable for model scale.
In \Cref{fig:scaling_laws}, we plot calibration error (smECE) vs.\ semantic accuracy.
For base models without chain-of-thought (first row), we see no correlation
between model capability (semantic accuracy) and calibration error.
This is consistent with our theoretical predictions,
which have no explicit dependency on model scale or capability.
It is worth noting that prior works have observed that calibration of base models can improve with model scale for other notion of calibration: 
next-token prediction in multiple-choice question setups \citep{kadavath2022language,zhu-etal-2023-calibration,plaut2025probabilities}.
We do not find such improvements for semantic calibration of \emph{base models}.

For instruct models and base models with chain-of-thought, 
we empirically observe that some configurations are overconfident, while other underconfident, 
and we divide those configurations into separate rows in \Cref{fig:scaling_laws}.
The dashed line illustrates the upper bound on the calibration error w.r.t.\ the accuracy for an overconfident configurations, 
which is dictated by the behavior of a maximally-overconfident predictor that puts its entire probability mass on a single choice.
For the overconfident configurations, 
we see little correlation between calibration error and accuracy beyond what is dictated by the upper bound.
For the underconfident configurations, 
we see also little correlation overall,
except for in the high-accuracy regime:
calibration error tends to decrease
when models approach perfect semantic accuracy\footnote{
Notably, this is not mathematically necessary.
Since semantic accuracy involves only the \emph{argmax}-probability class,
it is possible for a predictor to be perfectly semantically-accurate
while having high calibration error.}.
However, it is not clear whether this is a robust phenomenon.

\subsection{Discussion: Calibration in LLMs vs other deep networks}
One may wonder why the state of calibration in LLMs seems significantly
different from calibration in non-LLM deep networks (e.g. image classifiers).
Specifically, deep network classifiers are sometimes severely overconfident,
and sometimes well-calibrated,
depending on the specific network (e.g. \citet{guo2017calibration} vs. \citet{minderer2021revisiting}).
On the other hand, all base LLMs we tested were well-calibrated
(in non-CoT settings) --- there was no significant calibration difference between models.
This difference between LLMs and classifiers is due to
differences in training practices:
when training LLMs, practitioners monitor the test/validation loss closely,
and stop training before the test loss overfits (increases).
On the other hand, when training classifiers,
practitioners care about the test \emph{classification error},
and often continue training even as test \emph{loss} increases.
Most trained LLMs are therefore locally-loss-optimal w.r.t. test loss (and thus calibrated),
but trained classifiers might have high test loss (and thus be miscalibrated).
This perspective on the calibration of deep networks was articulated in Section 1.1 of \citet{blasiok2023when}.

\section{Conclusion}
We find that base LLMs,
despite being trained with a \emph{token-level syntactic} objective,
are remarkably calibrated 
with respect to the \emph{sequence-level semantics} of their generations.
Our central contribution is a principled mechanism behind this emergence,
building on recent theoretical connections between calibration and loss-optimality \citep{blasiok2023when,blasiok2024loss}.
This theory provides a unified lens through which to understand
the nuanced calibration behavior of models in practice,
distinguishing settings which are calibrated from
those which are not.
More generally, our work can be seen as a step towards 
understanding the structure of LLMs' output distribution:
$B$-calibration is one formal way of quantifying how close the LLM's
distribution is to the ground-truth pretraining distribution.

For the interested reader, 
we have an extended discussion and technical remarks
in \Cref{app:discuss}.

\subsection{Limitations}
\label{app:limits}

{\bf Types of Calibration.} 
One limitation of our paper is that we focus on a very specific
type of calibration, which is essentially a sampling-based notion
($B$-confidence-calibration).
It is possible that other types of calibration (e.g. verbalized calibration \citep{tian-etal-2023-just,mielke-etal-2022-reducing})
also emerge for certain types of LLM training; we consider this possibility interesting
but out-of-scope for the current work.

{\bf Practical Implications.} Our work is primarily scientifically motivated, 
and so we do not fully explore practical considerations or implications.
For example, we do not consider the computational efficiency of our confidence measurements.
This is a limitation to using such measures in practice,
since computing semantic confidence requires sampling an LLM
multiple times for the same question.
We consider translating our scientific results into
real-world improvements to be an important direction for future work.

{\bf Datasets.} Although we evaluate on a variety of different models,
we only evaluate on 4 selected datasets.
We chose these datasets to cover a diversity of domains and problem difficulties,
from questions about world-knowledge to mathematical reasoning problems.
Further, we chose datasets with \emph{open-ended} answers,
since calibration of multiple-choice datasets is already extensively studied
\citep{kadavath2022language, zhu-etal-2023-calibration}.
Although we do not expect our results to depend significantly
on the choice of dataset, it is possible that certain other datasets
have different calibration behavior; this is a limitation of our experiments.

\begin{remark}
Notably, there are some datasets which
we would expect to behave differently,
such as TruthfulQA \citep{lin2022truthfulqa},
which is a dataset containing common human misconceptions.
This dataset fails to satisfy the ``in-distribution'' requirement 
of our results (e.g. \Cref{corr:main_heuristic_semantic} and Footnote~\ref{footnote:multical}),
and so it is consistent with our theory for models to be miscalibrated.
\end{remark}

{\bf Theory Formalism.}
There remain several steps in our conjectured mechanism (Fig~\ref{fig:mechanism})
lacking formal definitions and proofs.
It is an open question to formalize these in meaningful and tractable ways.

\subsubsection*{Acknowledgments}
We are grateful for discussion and feedback from
Parikshit Gopalan,
Eran Malach,
Aravind Gollakota,
Omid Saremi,
Madhu Advani,
Etai Littwin,
Josh Susskind,
and
Russ Webb.

\applefootnote{ \textcolor{textgray}{\sffamily Apple and the Apple logo are trademarks of Apple Inc., registered in the U.S.~and other countries and regions.}}

\bibliography{refs}
\bibliographystyle{iclr2026_conference}

\newpage

\appendix
\crefalias{section}{appendix}

\startcontents[appendix]
\printcontents[appendix]{}{1}{\section*{Appendix Contents}}

\newpage
\section{Additional Related Works}

\paragraph{Recalibration Methods.}
A number of prior works study methods to improve the calibration of LLMs, ranging from temperature-scaling at inference-time \citep[e.g.,][]{xie-etal-2024-calibrating,pmlr-v235-shen24c}
to training calibration-specific probes that predict correctness \citep{mielke-etal-2022-reducing} or training with calibration-improving regularization terms \citep{wang-etal-2025-towards-objective}. Other approaches attempt to cluster questions and predict per-cluster accuracy \citep{lin2022teaching,ulmer-etal-2024-calibrating}, or make use of the fact that ensembling models tends to improve calibration \citep{jiang2023calibrating,hou2023decomposing}. Probabilistic approaches (such as Bayesian deep learning, or evidential deep learning) have been found to often yield better calibration \citep[e.g.,][]{li2025calibrating,yang2023bayesian}.

\paragraph{Sampling-based Confidences.}
A number of prior works have proposed sampling-based approaches to defining LLM uncertainty.
Both \citet{wang2023selfconsistency} and \citet{wei2024measuring}
sample multiple answers per-question, and define confidence
as the frequency of the most-common answer.
\citet{wei2024measuring} additionally groups answers together by string-matching,
which allows for some degree of semantic equivalence.
This approach was extended and popularized by
the notion of \emph{semantic entropy} \citep{farquhar2024detecting}.
Semantic entropy clusters sampled answers together
by semantic content, and then measures the empirical entropy of clustered answers.
Recently, \citet{lamb2025semantic} define Empirical Semantic Confidence,
which is essentially an empirical version of our notion of semantic confidence.
Note that one distinguishing aspect of our formalism is,
we parameterize the notion of calibration by the choice of collapsing function $B$.
This allows us to develop somewhat more general theoretical insights,
which are not tied to a fixed notion of semantics.

\paragraph{Factors which Harm LLM Calibration.}
Various factors have been observed in prior work to harm LLM calibration.
It is well-known that RLHF often harms calibration in multiple-choice QA settings \citep{kadavath2022language,openai2023gpt}.
Other RL post-training methods such as DPO have also been observed to harm calibration \citep{leng2024taming,xiao2025restoring}.
Some studies have also found chain-of-thought responses to harm calibration,
agreeing with our results \citep{zhang-etal-2024-study}.
However, we warn that not all of these works use the same notion of confidence and calibration as we do,
and so are not directly comparable.

\section{Extended Discussion and Remarks}
\label{app:discuss}

\subsection{Potential Extensions}

The theoretical framework described here is fairly general, and extends beyond the
setting of confidence-calibration in LLMs.
Briefly, since most of our theory
is stated in the language of \emph{weighted calibration} \citep{gopalan2024computationally},
it applies
to any property that can be written as weighted calibration.
This includes slightly stronger notions of calibration,
such as top-label calibration,
and also includes conformal-prediction type of guarantees (more details in \Cref{subsec:Conformal_Calibration}.
See \citet{gopalan2024computationally} for a number of properties which can be expressed
as weighted calibration, and \Cref{app:conformal} for the connection to conformal prediction.
Our general theoretical results appear in \Cref{app:theory}.

Intuitively, the high-level message of our results is that if a model is trained
with a max-likelihood / log-loss objective, then we should expect it to
satisfy weighted calibration 
for a ``simple'' family of weight functions.
The appropriate notion of simplicity depends on the model architecture;
simple weight functions should roughly correspond to easy-to-learn perturbations
to the model's output distribution.
At this level of generality, we expect some version of our results to apply
even for real-valued density models, such as continuous normalizing flows
(e.g. \citet{zhai2025normalizing}), which are also trained with
the log-likelihood objective. That is, we should expect such normalizing flows
to also exhibit certain (weak) types of calibration.
We believe this is a promising avenue for future work.\looseness=-1

\begin{remark}
This high-level message can be interpreted as
``models should match the true distribution with respect to all
easy-to-learn features.''
Interestingly, a very similar statement was conjectured for
\emph{interpolating} classifiers in \citet{nakkiran2020distributional},
called ``Feature Calibration.''
The exact relation to our present work is unclear.
\end{remark}

\arxiv{

\paragraph{Other choices of $B$}. 
\pnote{I'm working on this para.}

Talk about the Kalai OAI paper on Hallucinations.
$B(x, z)$ that is ``easy''
\pnote{
 if you ask it a question like “When was YOURNAME born?” , it will answer with a date, not like “Strawberry.” This is because it’s very easy for an LLM to be B-calibrated w.r.t. the predicate
B(x, z) := "is question x about a date, and is generation z a date"
And being B-calibrated implies that it should never answer a non-date to a question that’s obviously about a date.
In general, by the above reasoning, we should expect that when LLMs are wrong, they are wrong in “believable” ways (i.e. wrong in ways that are hard-to-distinguish-from-truth) --- and we call such things “hallucinations”
}

For the sake of clarity, we focused our exposition on the special case of semantic calibration. 
However, our theory can be applied more generally,
for other choices of function $B$.
To give one interesting example, consider ...
}

\arxiv{

\paragraph{Question-dependent Confidences vs. Q-and-Answer-dependent confidences}

\pnote{Importance of log loss: loss on tokens --> proper loss on entire sequence.
Fails for MSE loss.
From Tokens to Concepts.}

}

\subsection{Technical Remarks}
\label{subsec:remark_theory_section}
We collect several technical remarks regarding the theory of \Cref{sec:theory}.

\begin{remark}[The Distribution]
One detail of the theory worth discussing further is the role of the
\emph{distribution $\cD$}.
Technically, our theory only implies calibration if the base model
has been trained on the exact same distribution on which calibration is evaluated.
This is obviously not strictly true (e.g. models are not pretrained only on GSM8K).
However, we may imagine modern pretrained models to behave ``as if''
they were trained on the evaluation distributions, since the pretraining distribution
is large and diverse.
Notably, this reasoning requires the evaluation dataset and
prompt choice to be reasonably in-distribution for the pretraining
(see \citet{zhang-etal-2024-study} for settings where the
prompting scheme affects calibration).
We discuss a more formal way to think about the choice of distribution in \Cref{remark:multical} below.
\end{remark}

\begin{remark}[Multicalibration]
\label{remark:multical}
For clarity of exposition, we described the theory as if there is only one
distribution $\cD$ of interest, but in reality, we evaluate calibration across
multiple distributions (TriviaQA, GSM8K, etc),
and we pretrain on yet another distribution.
Moreover, we find that a single model can be
simultaneously calibrated across many evaluation distributions.
We touched upon this issue in Footnote~\ref{footnote:multical},
but there is a theoretically cleaner (though more involved)
way to think about multiple distributions, which we outline now.

Formally, requiring $B$-calibration across multiple distributions simultaneously
can be thought of as a
\emph{multi-calibration} property \citep{hebert2018multicalibration}.
Suppose for example that the pretraining distribution $\cD$
is some mixture of disjoint sub-distributions:
$\cD = \alpha_1D_1 + \alpha_2D_2 + \dots $.
Suppose we are interested in $B$-calibration simultaneously for distributions $D_1$
and $D_2$.
Then, it is possible to show a generalization of \Cref{thm:calibration_equivalences}:
\begin{quote}
\emph{A model is $B$-confidence-calibrated across both $D_1$ and $D_2$
if and only if it is locally-loss-optimal on $\cD$ w.r.t.
an expanded class of perturbations $\cW_B^*$.}
\end{quote}

Informally, the class of perturbations $\cW_B^*$ is essentially
the usual class $\cW_B$ (of \Cref{def:perturbation_classes})
augmented by indicator functions 
$\1\{x \in \cD_1\},~\1\{x \in \cD_2\}$
for membership in each sub-distribution.

We will not get into the technical details,
but using this version of \Cref{thm:calibration_equivalences},
it is possible to carry out the remaining steps
of the argument from \Cref{sec:theory} and \Cref{fig:mechanism}.
Applying the same heuristics, for example, we would conclude:
an LLM will be simultaneously $B$-confidence-calibrated
on distributions $\cD_1, \cD_2$ if
it is easy for the LLM to
(1) estimate its own distribution on $B$-classes
and (2) identify samples as either $x \in \cD_1$ or $x \in \cD_2$.

The second condition is likely to be satisfied in all our experiments,
since all our evaluation datasets are distinct and easy to identify.
Thus, the predictions of our theory remain unchanged,
justifying our choice to avoid discussing multicalibration in the main body.
\end{remark}

\begin{remark}[Full calibration]
At first glance, it may seem that a
minor generalization of our mechanism (\Cref{fig:mechanism})
would also imply \emph{full} $B$-calibration (i.e., canonical calibration of the $B$-induced classifier),
rather than just \emph{confidence}-calibration.
After all, \Cref{thm:calibration_equivalences} formally generalizes to 
arbitrary weight families $\cW$
(see \Cref{thm:bcal_calibration_equivalences}),
including the family corresponding to full $B$-calibration
(defined as $\WBcal$ in \Cref{def:bcal_perturbation_classes}).
However, full $B$-calibration is too strong a property to hold in general\footnote{
For example, when $K$ (the number of $B$-classes) is large, full $B$-calibration
would be computationally intractable to even estimate \citep{gopalan2024computationally}.}.
So, which part of our argument in \Cref{fig:mechanism} breaks for full calibration?
The culprit is the step (B) $\implies$ (C).
The weight family $\WBcal$ relevant for full calibration is, roughly speaking, ``too large''
for this step to hold.

To better understand why the heuristic fails, here is 
more general version of the (B) $\implies$ (C) step in \Cref{fig:mechanism},
which we believe is plausible for arbitrary weight families $\cW$.

\begin{quote}
\begin{claim}[assumption, informal]
\label{claim:w-general}
If a perturbation family
$\cW$ is easy-to-learn for a pretrained LLM, meaning:
for all perturbations $w \in \cW$,
the LLM $p_\theta: \cV^* \to \Delta(\cV^N)$ can be easily LoRA-fine-tuned to
match the distribution of a perturbed-model $G: \cV^* \to \Delta(\cV^N)$,
\[
G: x \mapsto p_x \star w_x \equiv p_\theta( \cdot \mid x) \star w(x, p_x)
\]
then $p_\theta$ will be $\cW$-locally-loss-optimal w.r.t. its pretraining loss.
\end{claim}
\end{quote}

In other words, if all perturbations in the family $\cW$
can be ``easily learnt,'' then we should expect the LLM to be loss-optimal w.r.t. $\cW$. 

If we believe \Cref{claim:w-general}, we can see why our mechanism would apply 
to confidence-calibration but not to full-calibration:
For confidence-calibration, the perturbation class $\WBconf$ (\Cref{def:perturbation_classes})
is simple enough to be learnable,
while for full calibration, the corresponding perturbation class
$\WBcal$ (\Cref{def:bcal_perturbation_classes}) is too large to be efficiently learnable from samples.
To gain intuition for this, it helps to directly compare \Cref{def:bcal_perturbation_classes} to \Cref{def:perturbation_classes}.
From this discussion, we can see it is likely possible to extend our results
to certain types of calibration which are weaker than full-calibration,
but stronger than confidence-calibration.
We leave this direction for future work.
\end{remark}

\section{Additional Experimental Results}
\label{app:additional_experimental_results}

Due to their volume, disaggregated reliability diagram results are reported separately in \Cref{app:encyclopedia}.

\begin{figure}[H]
    \centering
    \includegraphics[width=0.55\linewidth]{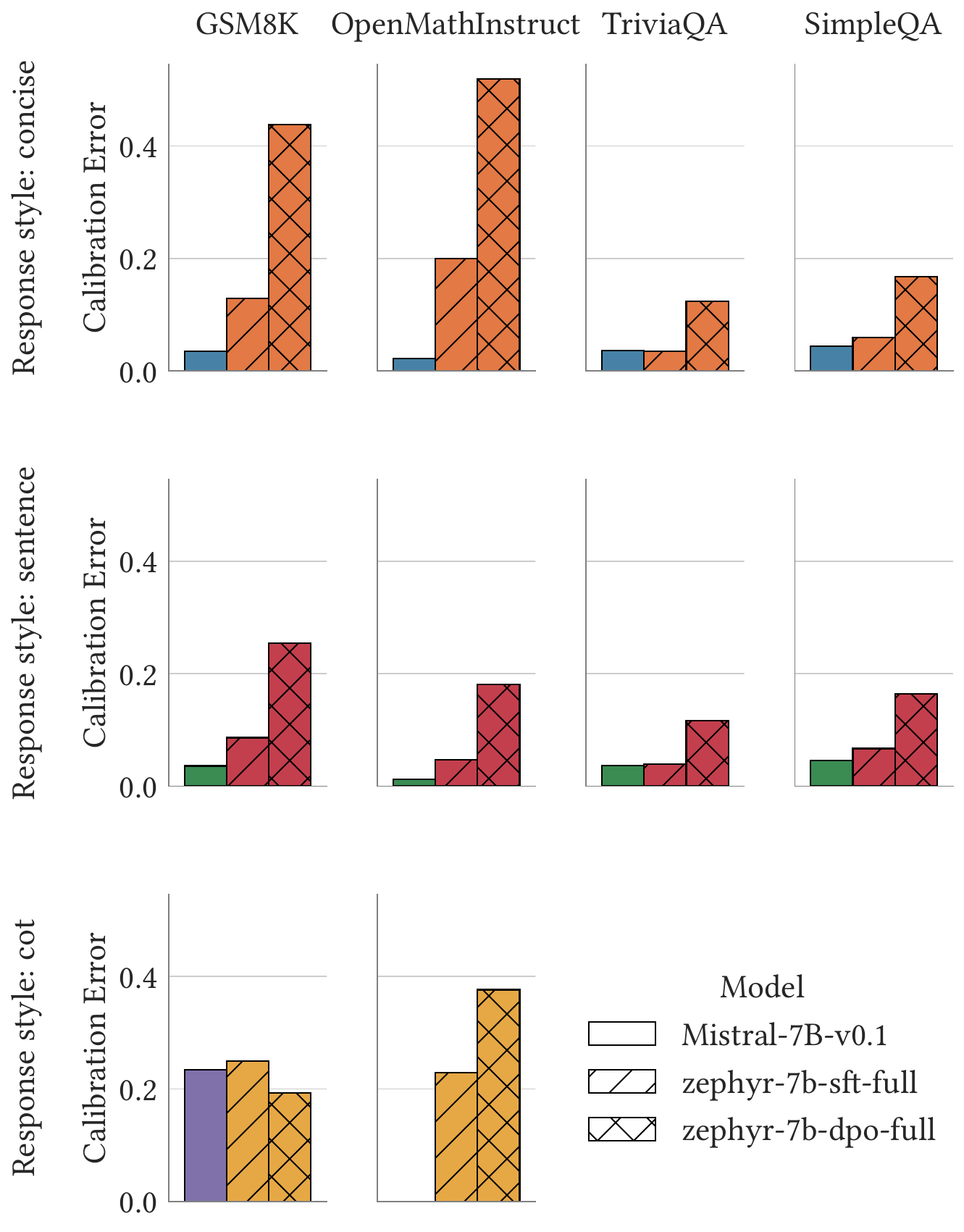}
    \caption{
    Calibration error for three models based on Mistral-7B-
    v0.1: pretrained-only, instruction-SFT model (zephyr-7b-sft-full),
    DPO model (zephyr-7b-dpo-full). 
    We did not evaluate TriviaQA and SimpleQA for the ``cot'' response style.
    The ``cot'' result for Mistral-7B-v0.1 for OpenMathInstruct is missing due to the model not terminating generation within its maximum context length.
    }
    \label{fig:sft_dpo_full}
\end{figure}

\section{Additional Experimental Details}
\label{app:experiments}

\paragraph{Datasets.}
We focus on open-ended question-answer (QA) settings,
since calibration for multiple-choice QA is
already well-studied \citep{kadavath2022language, zhu-etal-2023-calibration}, and a special case of our results.
We evaluate on:
GSM8K \citep{cobbe2021gsm8k},
OpenMathInstruct-2 \citep{toshniwal2024openmathinstruct},
TriviaQA \citep{joshi-etal-2017-triviaqa},
and SimpleQA \citep{wei2024measuring},
from Huggingface datasets \citep{wolf2019huggingface,lhoest2021datasets}.

\paragraph{Models.}
We evaluate on models including
the Qwen, Gemini, Mistral, and Llama family,
of sizes from 0.5B to 72B.
The full list of models we evaluate is in \Cref{app:list-of-llms}.
We use vLLM \citep{kwon2023efficient} for inference.

\paragraph{Prompt format.}
See \Cref{app:prompt-examples} for the exact phrasing used in prompts.
All of our experiments include $5$-shot examples in the prompt.
We use three different prompt types, designed to elicit three different styles
of responses from the model: ``sentence'', ``concise'', and ``chain-of-thought (cot)''. The few-shot examples are formatted in the desired style (e.g. for the ``sentence'' type, the few-shot examples have complete-sentence answers). For instruct models, in addition to formatted few-shot examples, the prompt also includes explicit formatting instructions.
The ``concise'' prompt type encourages the model to respond with just the final answer (a single word, phrase, or number).
The ``sentence'' prompt type asks the model to answer each question in a complete sentence (making it likely to phrase the same semantic answer in different ways, so the $B$-collapsing function is essential for a meaningful notion of semantic calibration).
The ``cot'' prompt type elicits chain-of-thought reasoning from the model; this prompt
type is only used for math datasets.

These prompts are typically successful in eliciting the desired type of responses from the model.
However, in some cases we observed models (especially Qwen models)
produce ``chain-of-thought'' responses
even when prompted to reply in a single word.
To exclude such cases, we exclude any responses for the
``concise'' prompt on math datasets which are too long
(heuristically, more than 15 characters before the first newline).

\paragraph{The semantic collapsing function.}
Recall, the function $B$ is intended to collapse semantically-equivalent
generations into a single class, an idea proposed by \citet{kuhn2023semantic}.
We implement the function $B$ with a two-stage procedure as follows.

The first stage is canonicalization: we extract a short ``canonical form'' answer
from the LLM's response. For ``concise'' and ``cot'' prompt types, this is done via simple
string parsing (for ``cot'', extracting only the final answer).
For the ``sentence'' type, we use a strong LLM (Qwen3-14B-Instruct) prompted
to extract a short-answer from the generation, given the question as context.
The prompts used for canonicalization are in \Cref{app:prompt-examples}:
Prompt~\ref{canonicalize} for non-math settings, and
Prompt~\ref{canonicalize-math} for math settings.
We also normalize strings at this stage, converting to lower-case and stripping spaces,
including a math-specific normalization for domains with LaTeX outputs.
Specifically, we use the MATH string-normalization
from Minerva, given in Listing 1, Appendix D.1 of \citet{lewkowycz2022solving}.

The second stage, used only for non-math settings, is semantic clustering:
we prompt an LLM judge (Qwen3-14B-Instruct) to assess whether two responses to a question are semantically equivalent, and use the output to cluster responses\footnote{This is a slight variation of the two-way entailment method used by \citet{farquhar2024detecting}.}.
This is necessary for non-math settings to handle irrelevant differences in canonical forms
(e.g. ``Seattle, WA'' vs ``Seattle'').
The prompt used for semantic equivalence is Prompt~\ref{prompt:sem-equiv} in \Cref{app:prompt-examples}.
For math settings, the second stage is unnecessary, since the first stage already outputs a
number or symbol that can be directly compared.

\paragraph{Measuring calibration.}
We first produce an LLM-induced semantic classifier, following the experimental procedure described in
\Cref{sec:semantic-calibration} and illustrated in \Cref{fig:B-collapse-intrp}.
For each dataset, we take 10K random evaluation samples
(or the entire dataset for those with fewer than 10K total samples).
For each question, we construct the appropriate $5$-shot prompt,
sample $M=50$ responses from the LLM at temperature $1$,
and then apply the semantic collapsing function (described above) to each response.
The semantic confidence is defined as the empirical frequency of the plurality semantic class,
and the semantic accuracy is the $0/1$ indicator of whether this
plurality class matches the ground-truth's semantic class.
This yields, for each question, a pair of
(semantic-confidence, semantic-accuracy) $\in [0, 1] \times \{0, 1\}$.
We then evaluate the calibration of the resulting classifier
over the entire dataset of questions using SmoothECE (smECE, \citet{blasiok2024smooth}),
a theoretically-principled version of the Expected Calibration Error (ECE),
as described below.

\subsection{Calibration Metric: SmoothECE}
\label{app:smECE}
We measure calibration error using SmoothECE (smECE) \citet{blasiok2024smooth},
which is essentially a kernel-smoothed version of Expected-Calibration-Error
with better theorhetical properties.

For interpretability reasons,
we chose to use a \emph{fixed} bandwidth of $\sigma=0.05$,
rather than the automatic bandwidth defined by SmoothECE.
This fixed choice makes smECE behave closer to a ``smoothed'' version of
BinnedECE with bin-width $=0.05$, which makes the metric more directly comparable to prior works.
Fixing the bandwidth comes at the cost of slightly weaker theoretical guarantees; however the smECE at scale $\sigma$
still bounds the distance-to-calibration \citep{blasiok2023unifying} in the following way:

\[
(\underline{\mathrm{dCE}} - \sigma) \leq  \mathrm{smECE}_\sigma \leq
\left(1 + \frac{1}{\sigma} \right) \underline{\mathrm{dCE}}.
\]
The LHS is \citet[Lemma 8]{blasiok2024smooth}
and the RHS is \citet[Lemma 9]{blasiok2024smooth}.

We use the SmoothECE implementation provided by: \url{https://github.com/apple/ml-calibration}.
Specifically, we use \verb|relplot.smECE_sigma| with $\sigma=0.05$.

\subsection{Visualizing calibration: reliability diagrams}
\label{app:rel-diagram}
We follow the guidance of \citet{blasiok2024smooth},
and visualize calibration using kernel-smoothed reliability diagrams.

{\bf Reading the Diagram.}
\Cref{fig:calib_works_grid} gives several examples of reliability diagrams.
The solid red line is the regression line, 
an estimate of 
$\mu(c) := \E[ \textrm{semantic accuracy} \mid \textrm{semantic confidence} = c]$.
The black cross is the point 
$(  \E[ \textrm{semantic confidence}] ,  \E[ \textrm{semantic accuracy}] ) \in [0, 1] \times [0, 1]$, that is, the average semantic confidence and accuracy.
The gray histograms at the bottom of the plot visualize the
density of semantic confidences.
We plot two overlaid histograms,
one for the confidence distribution of correct predictions
(i.e. the confidence of samples where semantic-accuracy=1),
and another for the confidence distribution of incorrect predictions.
The width of the red regression line varies with the overall density of
semantic-confidences.

{\bf Implementation Details.} For reliability diagrams, we use the implementation of 
\texttt{relplot} (\url{https://github.com/apple/ml-calibration}) 
with minor modifications:
we use a fixed kernel bandwidth $\sigma=0.05$ for the regression line,
and we visualize the density of confidences using histogram binning with 15 constant-width bins.

\subsection{LoRA Fine-Tuning}
\label{app:lora}

To test \Cref{claim:main_heuristic} more quantitatively, we  train a
LoRA version of the LLM to explicitly learn the function $G$ defined in \Cref{claim:main_heuristic}.
We do this as follows.
Let $p_\theta$ be the base model.
Instantiate a rank=8 LoRA adapter \citep{hu2022lora}
on top of the original model $p_\theta$,
which we denote $p_\phi$.

We want to train $p_\phi$ to behave as the ``semantically-collapsed'' version of $p_\theta$.
That is, when prompted with a question $x$, the model $p_\phi$ should generate
a distribution on answers $b$ which imitates the base model's semantic answers $B_x(z)$:
\begin{equation}
\label{eqn:phi-match}
\quad p_\phi(b \mid x) \approx 
\Pr_{z \sim p_\theta(\cdot \mid x)}[ B_x(z)=b ] \equiv (B_x \sharp p_x)(b)
\end{equation}
Since our implementation of the collapsing function $B$ produces string outputs (canonical answers),
we can train $p_\phi$ as a standard autoregressive model.
Explicitly:

\begin{enumerate}
    \item For each question in the dataset $x$,
    sample the original model 50 times, and apply the collapsing function $B$ to each generation.
    This produces 50 samples $\{(x, b_i)\}$ of question $x$ and canonical-answer $b_i$ for each 
    original question $x$, effectively expanding the original dataset size by 50 times.
    \item Train $p_\phi$ with the standard autoregressive objective,
    on the prompt-completion pairs $\{(x, b_i)\}$ from above.
    That is, train $p_\phi$ to complete prompt $x$ with generation $b_i$.
\end{enumerate}
Our training procedure is similar to the procedure used to train ``P(IK)'' in \citet{kadavath2022language},
in that we also train on an ``expanded'' training set defined by base model samples.
Similar to \citet{kadavath2022language}, we do this mainly for convenience.

For GSM8K, we hold-out 2000 questions for evaluation, and use the remainder for training as above.
We train all models on an 8xA100 node for 1 epoch on the expanded dataset, using the \verb|SFTTrainer| implementation
from Huggingface TRL \citep{vonwerra2022trl} with the following parameters in \Cref{tab:hparams}.
Note, we shuffle the expanded training set manually beforehand,
so we do not ask the dataloader to shuffle.

\begin{table}[H]
\centering
\caption{Hyperparameters for Supervised Fine-Tuning (SFT).}
\label{tab:sft_hyperparams}
\begin{tabular}{ll}
\toprule
\textbf{Parameter} & \textbf{Value} \\
\midrule
\multicolumn{2}{l}{\textit{Training \& Hardware}} \\ %
\quad \texttt{num\_train\_epochs} & 1 \\
\quad \texttt{per\_device\_train\_batch\_size} & 4 \\
\quad \texttt{gradient\_accumulation\_steps} & 2 \\
\quad (Effective Batch Size) & 64 (4 x 8 GPUs x 2) \\
\quad \texttt{bf16} & True \\
\midrule
\multicolumn{2}{l}{\textit{Optimizer \& Scheduler}} \\ %
\quad \texttt{optim} & \texttt{adamw\_torch\_fused} \\
\quad \texttt{learning\_rate} & 5e-5 \\
\quad \texttt{weight\_decay} & 0.0 \\
\quad \texttt{warmup\_ratio} & 0.05 \\
\midrule
\multicolumn{2}{l}{\textit{PEFT (LoRA) Configuration}} \\
\quad \texttt{use\_peft} & True \\
\quad \texttt{lora\_r} & 8 \\
\quad \texttt{lora\_alpha} & 16 \\
\quad \texttt{lora\_dropout} & 0.0 \\
\quad \texttt{lora\_target\_modules} & \texttt{all-linear} \\
\quad \texttt{task\_type} & \texttt{CAUSAL\_LM} \\
\quad \texttt{bias} & \texttt{none} \\
\midrule
\multicolumn{2}{l}{\textit{Data Handling}} \\
\quad \texttt{packing} & False \\
\quad \texttt{dataloader\_shuffle} & False \\
\bottomrule
\end{tabular}
\label{tab:hparams}
\end{table}

After training, we evaluate how closely \Cref{eqn:phi-match} holds,
by estimating the KL divergence between RHS and LHS of \Cref{eqn:phi-match}.
This KL measures how well our LoRA $p_\phi$ matches its training distribution.
Conveniently, the KL can be written as the difference between
the \emph{negative-log-loss} of $p_\phi$ and 
the \emph{semantic entropy} of the original model $p_\theta$:

\begin{align}
\textrm{Gap to optimality}
&:= 
KL(~ (B_x \sharp p_x) ~~\Vert~~ p_\phi( \cdot \mid x)   ~)  \\
&= 
\underbrace{
\E_{\substack{
x \sim \cD\\
z \sim p_\theta( z \mid x )
}}
[
-\log p_\phi( B(z) \mid x )
]}_{\textrm{Eval NLL loss of $p_\phi$}}
-
\underbrace{H( B_x \sharp p_x )}_{\textrm{Semantic entropy of $p_\theta$}}
\label{eqn:KL-gap}
\end{align}

This is particularly convenient because the eval log-loss is a
standard metric tracked during training.
Note that for our purposes, it is important to compute the
\emph{unnormalized} log-loss (i.e., not normalized by sequence-length).

In {\bf \Cref{fig:scatterplot}},
we plot the KL gap of \Cref{eqn:KL-gap}
on the x-axis, and the SmoothECE of the original model $p_\theta$ on the y-axis.
We evaluate base models: Qwen2.5-\{0.5B, 1.5B, 3B, 7B, 14B\},
with all three response styles:
\textcolor{AppleBlue5}{concise},
\textcolor{AppleGreen5}{sentence},
\textcolor{ApplePurple5}{cot}.
This results in 15 points plotted in \Cref{fig:scatterplot}, 
colored according to response style using the color scheme of \Cref{fig:ece-grid}.
We observe that, consistent with \Cref{claim:main_heuristic},
configurations where the semantic class distribution is easy-to-learn
(low KL gap) also have small calibration error.
The points with high KL (and high calibration error) are 
the \textcolor{ApplePurple5}{chain-of-thought} experiments,
as well as the small 0.5B model with the ``sentence'' response type.

\subsection{LLMs evaluated}
\label{app:list-of-llms}

Below, we list all models evaluated in this paper.
All were obtained from HuggingFace.

\begin{longtable}{@{}lll@{}}
\caption{Pretrained-only base models evaluated in this paper. Models sharing a prefix and reference are grouped.}
\label{tab:base-models-final} \\
\toprule
\textbf{Family Prefix} & \textbf{Model Suffix} & \textbf{Reference} \\
\midrule
\endfirsthead

\caption[]{} \\ 
\toprule
\textbf{Family Prefix} & \textbf{Model Suffix} & \textbf{Reference} \\
\midrule
\endhead

\bottomrule
\endfoot
\endlastfoot

\multirow{7}{*}{\texttt{google/}} 
& \texttt{gemma-2-2b} & \multirow{3}{*}{\citep{team2024gemma}} \\
& \texttt{gemma-2-9b} & \\
& \texttt{gemma-2-27b} & \\
\cmidrule(r){2-3}
& \texttt{gemma-3-1b-pt} & \multirow{4}{*}{\citep{team2025gemma}} \\
& \texttt{gemma-3-4b-pt} & \\
& \texttt{gemma-3-12b-pt} & \\
& \texttt{gemma-3-27b-pt} & \\
\midrule

\multirow{15}{*}{\texttt{Qwen/}} 
& \texttt{Qwen2.5-0.5B} & \multirow{7}{*}{\citep{Yang2024Qwen25TR}} \\
& \texttt{Qwen2.5-1.5B} & \\
& \texttt{Qwen2.5-3B} & \\
& \texttt{Qwen2.5-7B} & \\
& \texttt{Qwen2.5-14B} & \\
& \texttt{Qwen2.5-32B} & \\
& \texttt{Qwen2.5-72B} & \\
\cmidrule(r){2-3}
& \texttt{Qwen2.5-Math-1.5B} & \multirow{3}{*}{\citep{yang2024qwen2}} \\
& \texttt{Qwen2.5-Math-7B} & \\
& \texttt{Qwen2.5-Math-72B} & \\
\cmidrule(r){2-3}
& \texttt{Qwen3-0.6B-Base} & \multirow{5}{*}{\citep{yang2025qwen3}} \\
& \texttt{Qwen3-1.7B-Base} & \\
& \texttt{Qwen3-4B-Base} & \\
& \texttt{Qwen3-8B-Base} & \\
& \texttt{Qwen3-14B-Base} & \\
\midrule

\multirow{4}{*}{\texttt{mistralai/}} 
& \texttt{Mistral-7B-v0.1} & \multirow{2}{*}{\citep{Jiang2023Mistral7}} \\
& \texttt{Mistral-7B-v0.3} & \\
\cmidrule(r){2-3}
& \texttt{Mistral-Small-24B-Base-2501} & \citep{mistral2024small31} \\
\cmidrule(r){2-3}
& \texttt{Mixtral-8x7B-v0.1} & \citep{mistral2023mixtral} \\
\midrule

\multirow{2}{*}{\texttt{meta-llama/}} 
& \texttt{Llama-3.1-8B} & \multirow{2}{*}{\citep{grattafiori2024llama}} \\
& \texttt{Llama-3.1-70B} & \\
\end{longtable}

\newpage

\begin{longtable}{@{}lll@{}}
\caption{Instruction-tuned models evaluated in this paper. Models sharing a prefix and reference are grouped.}
\label{tab:instruct-models-final} \\
\toprule
\textbf{Family Prefix} & \textbf{Model Suffix} & \textbf{Reference} \\
\midrule
\endfirsthead

\caption[]{} \\
\toprule
\textbf{Family Prefix} & \textbf{Model Suffix} & \textbf{Reference} \\
\midrule
\endhead

\bottomrule
\endfoot
\endlastfoot

\multirow{7}{*}{\texttt{google/}}
& \texttt{gemma-2-2b-it} & \multirow{3}{*}{\citep{team2024gemma}} \\
& \texttt{gemma-2-9b-it} & \\
& \texttt{gemma-2-27b-it} & \\
\cmidrule(r){2-3}
& \texttt{gemma-3-1b-it} & \multirow{4}{*}{\citep{team2025gemma}} \\
& \texttt{gemma-3-4b-it} & \\
& \texttt{gemma-3-12b-it} & \\
& \texttt{gemma-3-27b-it} & \\
\midrule

\multirow{16}{*}{\texttt{Qwen/}}
& \texttt{Qwen2.5-0.5B-Instruct} & \multirow{7}{*}{\citep{Yang2024Qwen25TR}} \\
& \texttt{Qwen2.5-1.5B-Instruct} & \\
& \texttt{Qwen2.5-3B-Instruct} & \\
& \texttt{Qwen2.5-7B-Instruct} & \\
& \texttt{Qwen2.5-14B-Instruct} & \\
& \texttt{Qwen2.5-32B-Instruct} & \\
& \texttt{Qwen2.5-72B-Instruct} & \\
\cmidrule(r){2-3}
& \texttt{Qwen2.5-Math-1.5B-Instruct} & \multirow{3}{*}{\citep{yang2024qwen2}} \\
& \texttt{Qwen2.5-Math-7B-Instruct} & \\
& \texttt{Qwen2.5-Math-72B-Instruct} & \\
\cmidrule(r){2-3}
& \texttt{Qwen3-0.6B} & \multirow{6}{*}{\citep{yang2025qwen3}} \\
& \texttt{Qwen3-1.7B} & \\
& \texttt{Qwen3-4B} & \\
& \texttt{Qwen3-8B} & \\
& \texttt{Qwen3-14B} & \\
& \texttt{Qwen3-32B} & \\
\midrule

\multirow{4}{*}{\texttt{mistralai/}}
& \texttt{Mistral-7B-Instruct-v0.1} & \multirow{2}{*}{\citep{Jiang2023Mistral7}} \\
& \texttt{Mistral-7B-Instruct-v0.3} & \\
\cmidrule(r){2-3}
& \texttt{Ministral-8B-Instruct-2410} & \citep{mistral2024ministraux} \\
\cmidrule(r){2-3}
& \texttt{Mistral-Small-24B-Instruct-2501} & \citep{mistral2024small31} \\
\midrule

\multirow{2}{*}{\texttt{NousResearch/}}
& \texttt{Nous-Hermes-2-Mixtral-8x7B-SFT} & \citep{nousresearch2024hermes2mixtralsft} \\
\cmidrule(r){2-3}
& \texttt{Nous-Hermes-2-Mixtral-8x7B-DPO} & \citep{nousresearch2024hermes2mixtraldpo} \\
\midrule

\texttt{alignment-}
& \texttt{zephyr-7b-dpo-full} & \multirow{2}{*}{\citep{tunstall2023zephyr}} \\
\texttt{handbook/} & \texttt{zephyr-7b-sft-full} & \\
\midrule

\multirow{3}{*}{\texttt{meta-llama/}}
& \texttt{Llama-3.1-8B-Instruct} & \multirow{3}{*}{\citep{grattafiori2024llama}} \\
& \texttt{Llama-3.1-70B-Instruct} & \\
& \texttt{Llama-3.3-70B-Instruct} & \\
\midrule

\texttt{microsoft/} & \texttt{phi-4} & \citep{abdin2024phi} \\
\end{longtable}

\newpage
\subsection{Prompts}
\label{app:prompt-examples}

\newtcolorbox[ auto counter ]{prompt}[2][]{%
breakable,
fonttitle=\bfseries,
fontupper=\small,
fontlower=\small,
title=Prompt~\thetcbcounter: #2,
#1}

\lstset{
basicstyle=\tiny\ttfamily,
columns=flexible,
breaklines=true,
breakautoindent=false,
breakindent=0pt,  %
}

We use 3 different prompt styles: concise, sentence, and chain-of-thought (cot).
All prompts use 5 few-shot examples from the dataset.
We describe the prompt formatting here by way of example, using our prompts for the GSM8K dataset.
For base models, we use the full prompt text as context,
while for instruct models we format
the few-shot examples using the model-specific chat template (per Huggingface).

\texttt{Prompt~\ref{gsm8k-concise}} shows the \emph{``concise"} prompt for GSM8K.
This prompt style uses only the final answers provided by the dataset
(excluding any chain-of-thought).

\texttt{Prompt~\ref{gsm8k-sentence}} shows the \emph{``sentence"} prompt type.
This prompt formats the few-shot answers in complete sentences,
and also includes instructions to format answers accordingly.
Note that we intentionally varied the sentence structure of the few-shot examples,
to encourage the model to use a diversity of phrasings.
This makes the ``sentence'' responses more syntactically complex than the ``concise'' responses,
though not more \emph{semantically} complex --- thus testing the limits of our theory.

\texttt{Prompt~\ref{gsm8k-cot}} shows the \emph{``cot"} prompt type.
This includes reasoning and formatting instructions, as well
as few-shot examples that include reasoning-traces (provided by the dataset).

The prompt formatting for other datasets follow the same conventions as these GSM8K examples.
We exclude the ``cot'' prompt type for non-math datasets.

\begin{prompt}[label=gsm8k-concise]{GSM8K-concise}
\begin{lstlisting}
Question: Natalia sold clips to 48 of her friends in April, and then she sold half as many clips in May. How many clips did Natalia sell altogether in April and May?
Answer: 72

Question: Weng earns $12 an hour for babysitting. Yesterday, she just did 50 minutes of babysitting. How much did she earn?
Answer: 10

Question: Betty is saving money for a new wallet which costs $100. Betty has only half of the money she needs. Her parents decided to give her $15 for that purpose, and her grandparents twice as much as her parents. How much more money does Betty need to buy the wallet?
Answer: 5

Question: Julie is reading a 120-page book. Yesterday, she was able to read 12 pages and today, she read twice as many pages as yesterday. If she wants to read half of the remaining pages tomorrow, how many pages should she read?
Answer: 42

Question: James writes a 3-page letter to 2 different friends twice a week.  How many pages does he write a year?
Answer: 624

Question: {QUESTION}
Answer:
\end{lstlisting}
\end{prompt}

\begin{prompt}[label=gsm8k-sentence]{GSM8K-sentence}
\begin{lstlisting}
Answer the following question in a single brief but complete sentence.
Question: Natalia sold clips to 48 of her friends in April, and then she sold half as many clips in May. How many clips did Natalia sell altogether in April and May?
Answer: Natalia sold 72 clips in April and May combined.

Answer the following question in a single brief but complete sentence.
Question: Weng earns $12 an hour for babysitting. Yesterday, she just did 50 minutes of babysitting. How much did she earn?
Answer: Weng earned only $10 yesterday.

Answer the following question in a single brief but complete sentence.
Question: Betty is saving money for a new wallet which costs $100. Betty has only half of the money she needs. Her parents decided to give her $15 for that purpose, and her grandparents twice as much as her parents. How much more money does Betty need to buy the wallet?
Answer: Betty needs $5 more to buy the wallet.

Answer the following question in a single brief but complete sentence.
Question: Julie is reading a 120-page book. Yesterday, she was able to read 12 pages and today, she read twice as many pages as yesterday. If she wants to read half of the remaining pages tomorrow, how many pages should she read?
Answer: She would need to read 42 pages tomorrow.

Answer the following question in a single brief but complete sentence.
Question: James writes a 3-page letter to 2 different friends twice a week.  How many pages does he write a year?
Answer: James writes 624 pages per year.

Answer the following question in a single brief but complete sentence.
Question: {QUESTION}
Answer:
\end{lstlisting}
\end{prompt}

\begin{prompt}[label=gsm8k-cot]{GSM8K-cot}
\begin{lstlisting}
Answer the following question. To do that, first reason about it by saying 'Reasoning:' and then derive the answer. After that, when you are done, write 'My answer is: ' and write a short and concise answer to the question.Last, write <DONE>.
Question: Natalia sold clips to 48 of her friends in April, and then she sold half as many clips in May. How many clips did Natalia sell altogether in April and May?
Answer: Reasoning: Natalia sold 48/2 = <<48/2=24>>24 clips in May.
Natalia sold 48+24 = <<48+24=72>>72 clips altogether in April and May.
My answer is: 72<DONE>

Answer the following question. To do that, first reason about it by saying 'Reasoning:' and then derive the answer. After that, when you are done, write 'My answer is: ' and write a short and concise answer to the question.Last, write <DONE>.
Question: Weng earns $12 an hour for babysitting. Yesterday, she just did 50 minutes of babysitting. How much did she earn?
Answer: Reasoning: Weng earns 12/60 = $<<12/60=0.2>>0.2 per minute.
Working 50 minutes, she earned 0.2 x 50 = $<<0.2*50=10>>10.
My answer is: 10<DONE>

Answer the following question. To do that, first reason about it by saying 'Reasoning:' and then derive the answer. After that, when you are done, write 'My answer is: ' and write a short and concise answer to the question.Last, write <DONE>.
Question: Betty is saving money for a new wallet which costs $100. Betty has only half of the money she needs. Her parents decided to give her $15 for that purpose, and her grandparents twice as much as her parents. How much more money does Betty need to buy the wallet?
Answer: Reasoning: In the beginning, Betty has only 100 / 2 = $<<100/2=50>>50.
Betty's grandparents gave her 15 * 2 = $<<15*2=30>>30.
This means, Betty needs 100 - 50 - 30 - 15 = $<<100-50-30-15=5>>5 more.
My answer is: 5<DONE>

Answer the following question. To do that, first reason about it by saying 'Reasoning:' and then derive the answer. After that, when you are done, write 'My answer is: ' and write a short and concise answer to the question.Last, write <DONE>.
Question: Julie is reading a 120-page book. Yesterday, she was able to read 12 pages and today, she read twice as many pages as yesterday. If she wants to read half of the remaining pages tomorrow, how many pages should she read?
Answer: Reasoning: Maila read 12 x 2 = <<12*2=24>>24 pages today.
So she was able to read a total of 12 + 24 = <<12+24=36>>36 pages since yesterday.
There are 120 - 36 = <<120-36=84>>84 pages left to be read.
Since she wants to read half of the remaining pages tomorrow, then she should read 84/2 = <<84/2=42>>42 pages.
My answer is: 42<DONE>

Answer the following question. To do that, first reason about it by saying 'Reasoning:' and then derive the answer. After that, when you are done, write 'My answer is: ' and write a short and concise answer to the question.Last, write <DONE>.
Question: James writes a 3-page letter to 2 different friends twice a week.  How many pages does he write a year?
Answer: Reasoning: He writes each friend 3*2=<<3*2=6>>6 pages a week
So he writes 6*2=<<6*2=12>>12 pages every week
That means he writes 12*52=<<12*52=624>>624 pages a year
My answer is: 624<DONE>

Answer the following question. To do that, first reason about it by saying 'Reasoning:' and then derive the answer. After that, when you are done, write 'My answer is: ' and write a short and concise answer to the question.Last, write <DONE>.
Question: {QUESTION}
Answer:
\end{lstlisting}
\end{prompt}

\begin{prompt}[label=canonicalize]{Canonicalization}
\begin{lstlisting}
Question: "{QUESTION}"
Response: "{RESPONSE}"

Your task is to return **only** the core answer from this response.
Follow these rules:
- Keep only the core answer (e.g., a number, a name, or a short phrase).
- Remove all extra words and filler.
- Expand all abbreviations to their full form (e.g., 'USA' -> 'United States of America').
- Write all numbers with digits, not as words (e.g., 'eight' -> '8').
- For locations, output only the highest-precision part (e.g. 'Seattle, Washington' -> 'Seattle')
- For dates, unless otherwise specified, format as YYYY-MM-DD (e.g. "August 1, 1990" -> "1990-08-01"). If only a month or year is specified, leave as-is (e.g. "August" or "2003" or "July, 2000"). Do not make up unspecified information.
- No explaining or reasoning. Output the core answer only.
- If the response does not address the question, or if you are unsure what to do, return the response unchanged.
- Never alter the meaning of the response, even if it is incorrect.
- Do not infer missing information; only rephrase what is given in the response.
\end{lstlisting}
\end{prompt}

\begin{prompt}[label=canonicalize-math]{Canonicalization (math)}
\begin{lstlisting}
Response: "{RESPONSE}"

Your task is to return **only** the core answer from this response.
Follow these rules:
- Keep only the core answer, as a raw number or LaTeX string (e.g. '0.5' or '\frac{1}{2}').
- If the answer is the value of a variable, only output the value itself (e.g. 'x=10' -> '10').
- Write all numbers with digits, not as words (e.g., 'eight' -> '8').
- Remove all extra words and filler.
- No explaining or reasoning. Output the core answer only.
- If the response does not contain a numeric value, or if you are unsure what to do, return the response unchanged.
- Never alter the value of the response, even if it is incorrect.
- Do not infer missing information; only extract what is given in the response.
\end{lstlisting}
\end{prompt}

\begin{prompt}[label=prompt:sem-equiv]{Semantic Equivalence}
\begin{lstlisting}
You will be given a question, and two possible responses. Your task is to determine whether the two answers are semantically consistent, i.e., whether the two responses agree on what the answer to the question is. 

Question: {QUESTION}
Response 1: {RESPONSE1}
Response 2: {RESPONSE2}

Are these two responses semantically aligned responses to the question? Respond only with either the string "Yes" or the string "No".

\end{lstlisting}
\end{prompt}

\newpage
\section{Theory}
\label{app:theory}

\subsection{Quick Reference}

In this section, we provide proofs of the theorems presented in the main text.
\begin{itemize}
    \item \Cref{thm:calibration_equivalences}
    is proved in \Cref{app:pf-of-calib-equiv}.
    \item \Cref{thm:ar-b-cal}
    is formally restated and proved as \Cref{thm:simple_conf_circuit}
    in \Cref{app:ar-conf-proofs}.
\end{itemize}

Proving these theorems involves some additional theoretical machinery, particularly \emph{weighted calibration},
which we develop here.
We restate some of the notation
and definitions from the main body for convenience.
We also give more general versions of several of our results in this section:
\begin{itemize}
    \item \Cref{app:full_b_cal} gives \emph{full calibration} analogs of our confidence-calibration results.
    \item \Cref{app:calib_gap_proper} extends from cross-entropy loss to general \emph{proper losses}, providing quantitative bounds between post-processing and calibration gap.
\end{itemize}

\subsection{Weighted Calibration}

A key object in our theory is the notion of \emph{weighted calibration},
from \citet{gopalan2024computationally},
which is capable of expressing many different types of calibration.
We will use a version of this definition suitable for our LLM setting,
stated below.

\begin{definition}[Weighted Calibration, \citet{gopalan2024computationally}]
\label{def:weighted-cal}
For a set $\cW$ of weight functions
$w: \cV^* \x \Delta(\cV^N) \to \R^N$,
and a distribution $\cD$ over pairs $(x, y) \in \cV^* \times \cV^N$,
a model $p_\theta$ is \emph{perfectly $\cW$-weighted-calibrated on $\cD$} 
if:
\begin{align*}
\E_{(x,y) \sim \cD} \left[ 
\left\langle \tilde{y} - p_x, w(x, p_x) \right\rangle \right] \equiv 0
\end{align*}
where $p_x := p_\theta(\cdot \mid x) \in \Delta(\cV^N) \subset \R^{|\cV^N|}$
is the model's output distribution on input $x$,
and $\tilde{y} \in \{0, 1\}^{|\cV^N|}$ is the one-hot-encoding of $y$.
\end{definition}

\begin{remark}
For the reader familiar with multi-calibration \citep{hebert2018multicalibration}:
note that in our definition above, the weight functions $w$ are allowed to depend
on the prompt $x$. This allows weighted calibration to capture various kinds of multi-calibration.
\end{remark}

\subsection{Equivalence between Weighted Calibration and Local Loss Optimality}
\label{subsec:Calib_loss_optimality}
Weighted calibration is equivalent to local loss optimality w.r.t.~perturbations in the weight class. In this section, we prove this in a special case relevant to our framework; a more general result presented in \Cref{app:calib_gap_proper}.

For the log-loss $\ell(y, f) := -\sum_i y_i\log(f_i)$, we can analyze perturbations more easily through its dual representation. The dual loss, which operates on a logit vector $z$ is defined as
$$\ell^{\star}(y, z) = \log\left(\sum_{j=1}^K e^{z_j}\right) - y^T z 
\text{ and }
\nabla_z \ell^{\star}(y, z) = \text{softmax}(z) - y = f - y
$$
The primal and dual views are connected by the variable mapping $z = \log f$, which provides the key equality $\ell(y, f) = \ell^{\star}(y, z)$. (This is a special case of a more general primal/dual framework for proper losses; c.f. \Cref{tab:xent_duality}.) The relationship allows us to translate complex perturbations in the probability space into simple ones in the logit space. A multiplicative re-weighting of the probabilities, defined as $f \star w := \text{softmax}(\log f + w) = \text{softmax}(z + w)$, is equivalent to a simple additive perturbation $w$ on the logits. Therefore, the loss of the perturbed model can be expressed in either world:
\[
\underbrace{\ell(y, f \star w)}_{\text{Loss on perturbed probabilities}} \quad = \quad \underbrace{\ell^{\star}(y, z + w)}_{\text{Loss on perturbed logits}}
\]

\begin{theorem}[Equivalence of Calibration and Local Loss Optimality]
\label{thm:bcal_calibration_equivalences}
For all models $p_\theta$, distributions $\cD$, 
proper losses $\ell$
and families of weight functions $\cW$ (Definition~\ref{def:weighted-cal}):
the model $p_\theta$ is perfectly \text{$\cW$-weighted-calibrated}
on $\cD$
if and only if
it is $\cW$-locally loss-optimal on $\cD$ w.r.t. loss $\ell$.
\end{theorem}

\begin{proof}

We apply the first-order optimality condition to the dual loss $\ell^{\star}(y, z)$ with a simple additive perturbation $w$ on the logits $z$. With the perturbed loss function, for $\epsilon>0$, 
$$\mathcal{L}(\epsilon) = \ell^{\star}(y, z + \epsilon w) \text{ and } \frac{d \mathcal{L}}{d \epsilon}(\epsilon) = \langle \nabla_z \ell^{\star}(y, z + \epsilon w), w \rangle$$
By local loss optimality
\begin{align*}
0 &\leq \frac{\mathcal{L}(\epsilon) - \mathcal{L}(0)}{\epsilon} = \frac{d \mathcal{L}}{d \epsilon}(0) + \frac{o(\epsilon)}{\epsilon} \longrightarrow \langle \nabla_z \ell^{\star}(y, z), w \rangle
\end{align*}
The same reasoning replacing $w$ by $-w$, we also have $\langle \nabla_z \ell^{\star}(y, z), w \rangle \leq 0$. Thus 

$$
\ell^{\star}(y, z) \leq \ell^{\star}(y, z + \epsilon w) \Longrightarrow \langle \nabla_z \ell^{\star}(y, z), w \rangle = 0
$$

The opposite implication follow from convexity, we have:
$$
\ell^{\star}(y, z + w) \geq \ell^{\star}(y, z) + \langle \nabla_z \ell^{\star}(y, z), w \rangle.
$$
Thus, if $\langle \nabla_z \ell^{\star}(y, z), w \rangle = 0$ holds, the inequality simplifies to:
$
\ell^{\star}(y, z + w) \geq \ell^{\star}(y, z).
$

Taking the expectation on both side
\begin{align*}
    \E_{\substack{(x, y) \sim \cD}}
    [ \ell(y, f) ] \leq
    \E_{\substack{(x, y) \sim \cD}}
    [ \ell(y,  f \star w ) ]
&\Longleftrightarrow
    \E_{\substack{(x, y) \sim \cD}}
    [ \ell^\star(y, z) ] \leq
    \E_{\substack{(x, y) \sim \cD}}
    [ \ell^\star(y,  z + w ) ] \\
&\Longleftrightarrow
\E_{\substack{(x, y) \sim \cD}} \langle f - y, w \rangle = \E_{\substack{(x, y) \sim \cD}} \langle \nabla_z \ell^{\star}(y, z), w \rangle = 0 
\end{align*}

A model is calibrated under the log-loss if and only if its expected prediction error $f - y$ is orthogonal to any systematic perturbation $w$ of its logits.
\end{proof}

\subsection{Equivalence between $B$-confidence-calibration and weighted calibration}
\label{app:b_cal}

In this section we prove that $B$-confidence-calibration can be characterized in terms of weighted calibration (\Cref{def:weighted-cal}).

\paragraph{Notation and Setup}
There are two relevant output spaces:
the space $\cV^N$ of long-form answer strings,
and the space $[K]$ of semantic answer classes.
Let $M := |\cV^N|$.
It will be convenient to identify strings $z \in \cV^N$
with an index in $[M]$, and we will abuse notation by writing $z \in [M]$.

To simplify some of the proofs, we will rely on an explicit one-hot representation.
For a string $y \in \cV^N$, we denote its one-hot representation as
$\tilde y \in \{0,1\}^M$.
For a given prompt $x \in \cV^*$, the model's distribution over completions is 
$p_\theta(\cdot \mid x) \in \Delta(\cV^N) \subset \R^{M}$,
which we treat as a vector embedded in $\R^M$.
We write $p_x := p_\theta(\cdot \mid x)$ for convenience.

A collapsing function $B: \cV^* \times \cV^N \to [K]$
assigns to each prompt $x \in \cV^*$ and long-answer $y \in \cV^N$
a $B$-class $B_x(y) \in [K]$.
Moreover, the function $B$ along with the model $p_\theta$
induces a distribution on classes $[K]$ as follows.
For a given input $x \in \cV^*$, we take the model's distribution $p_\theta(\cdot \mid x)$ and push it forward through $B_x$ to obtain a categorical distribution
$\pi_x$ defined as
\[
\pi_x := B_x \sharp p_\theta(\cdot \mid x) \;\in\; \Delta_K.
\]
Explicitly, the probability assigned to a category $c \in [K]$ is:
\begin{equation}
\pi_x(c) = \left(B_x \sharp p_x\right)(c)
= \Pr_{z \sim p_\theta(\cdot \mid x)}\left[B_x(z) = c\right]
= \sum_{z : B_x(z) = c} p_\theta(z \mid x).
\label{eq:pi_c_explicit}
\end{equation}

\paragraph{Definitions}
In the main text, we defined confidence calibration and $B$-confidence calibration via  \Cref{def:confidence_calibration} and \Cref{def:B-conf-cal}. We formally restate these definitions below.

\begin{definition}[Confidence Calibration]
\label{def:app-confidence_calibration}
A distribution $\cD$ over prediction-output pairs $(c, y) \in \Delta_K \times \eK$ is perfectly confidence-calibrated if, conditioned on the model's top predicted probability, that probability matches the expected outcome. Formally, 
\begin{equation}
    \E_{(c, y) \sim \cD} \left[ y_{k^\star} - c_{k^\star} \mid c_{k^\star} \right] \equiv 0 \text{ where } k^\star = \argmax_{k \in [K]} c_k.
\end{equation}
\end{definition}
Applying to our LLM setting, we say that a model is $B$-confidence-calibrated if the categorical distribution it induces is confidence-calibrated:

\begin{definition}[$B$-Confidence-Calibration]
\label{def:app-B-conf-cal}
A model $p_\theta$ is $B$-confidence-calibrated on a distribution $\cD$
if the induced distribution over pairs $(\pi_x, B_x(y))$ is perfectly
confidence-calibrated according to Definition~\ref{def:app-confidence_calibration}.
This requires that, for $k^\star = \argmax_{k \in [K]} \pi_x(k)$,
\begin{equation}
\E_{(x, y) \sim \cD} \left[
   \mathds{1}\{B_x(y) = k^\star\} - \pi_x(k^\star) 
   \;\middle|\; \pi_x(k^\star)
\right] = 0.
\end{equation}
\end{definition}

We also restate \Cref{def:perturbation_classes} here for convenience:
\begin{definition}[Semantic Perturbation Function Classes]
Given an arbitrary collapsing function $B_x(z) \in [K]$, we define the class
$\WBconf$ of perturbation functions $w(x, p_x) \in \R^{|\cV^N|}$ as follows.
These functions generate a perturbation vector based on the prompt $x$ and the
model's predictive distribution $p_x$:
\begin{align*}
\WBconf := 
\Big\{ w \;\Big|\;& \exists \tau: [0, 1] \to [-1, 1] \;\; 
\forall z \in \cV^N:\;
w(x,p_x)[z] = 
\tau\big(\pi_x(k^\star)\big)\cdot \mathds{1}\{B_x(z) = k^\star\} 
\Big\}, \\
&\text{where } \pi_x := B_x \sharp p_x, 
\quad k^\star := \argmax_{k \in [K]} \pi_x(k).
\end{align*}
\end{definition}

\paragraph{Equivalence Theorem}
Using the above definitions, we have the following equivalence.

\begin{theorem}[B-Confidence-Calibration as Weighted Calibration]
\label{thm:conf_cal_equivalence}
A model $p_\theta$ is perfectly $B$-confidence-calibrated
if and only if it is perfectly $\WBconf$-weighted-calibrated.
\end{theorem}

\begin{proof}

The model is $\WBconf$-weighted-calibrated if, for all $w \in \WBconf$, the following holds:
\begin{equation*}
    \E_{(x,y) \sim \cD} \left[ \langle \tilde{y} - p_x, w(x, p_x) \rangle \right] = 0.
\end{equation*}
For a given $w$ defined by a function $\tau: [0,1] \to [-1,1]$, since $\tilde y$ is a one-hot vector with a $1$ in the coordinate $z=y$, the first term evaluates to
\[
\langle \tilde y, w(x,p_x)\rangle
= \sum_z \tilde y[z] \, w(x,p_x)[z]
= w(x,p_x)[y],
\]

Substituting the definition of $w$:
    \begin{equation*}
    w(x, p_x)[y] = \tau\left(v_x^\star\right) \cdot \mathds{1}_{\{B_x(y) = k^\star\}} \text{ where } v_x^\star:= \pi_x(k^\star).
    \end{equation*}

The second term is $\langle p_x, w(x, p_x) \rangle = \sum_z p_x(z) w(x, p_x)[z]$. Substituting the definition of $w$:
    \begin{align*}
    \sum_z p_x(z) w(x, p_x)[z] &= \sum_z p_x(z) \left( \tau\left(v_x^\star\right) \cdot \mathds{1}_{\{B_x(z) = k^\star\}} \right) \\
    &= \tau\left(v_x^\star\right) \cdot \sum_z p_x(z) \mathds{1}_{\{B_x(z) = k^\star\}}  \\
    &= \tau\left(v_x^\star\right) \cdot \Pr[B_x(z)=k^\star] 
    = \tau\left(v_x^\star\right) \cdot v_x^\star
    \end{align*}

Putting these together, the weighted calibration condition becomes:
\begin{equation*}
\E_{(x,y) \sim \cD} \left[ \tau\left(v_x^\star\right) \cdot \mathds{1}_{\{B_x(y) = k^\star\}} - \tau\left(v_x^\star\right) \cdot v_x^\star \right] = 0
\Longleftrightarrow
\E_{(x,y) \sim \cD} \left[ \tau\left(v_x^\star\right) \cdot \left( \mathds{1}_{\{B_x(y) = k^\star\}} - v_x^\star \right) \right] = 0.
\end{equation*}
This condition must hold for all functions $\tau: [0,1] \to [-1,1]$. By the properties of conditional expectation, this is true if and only if the term being multiplied by the arbitrary function of $v_x^\star$ has a conditional expectation of zero. This gives us:
\begin{equation*}
\E \left[ \mathds{1}_{\{B_x(y) = k^\star\}} - v_x^\star \mid v_x^\star \right] = 0,
\end{equation*}
which is precisely the definition of $B$-confidence-calibration.
\end{proof}

\subsection{Proof of \Cref{thm:calibration_equivalences}}
\label{app:pf-of-calib-equiv}

We can now combine the above ingredients to 
directly prove \Cref{thm:calibration_equivalences} from the main body.

\begin{proof}
Recall we have a model $p_\theta$, a collapsing function $B$, and a distribution $\cD$.

We have the following equivalences:
\begin{align*}    
\textrm{$p_\theta$ is $B$-confidence-calibrated on $\cD$}
&\iff
\textrm{$p_\theta$ is $\WBconf$-weighted-calibrated on $\cD$}
\tag{by \Cref{thm:conf_cal_equivalence}} \\
&\iff
\textrm{$p_\theta$ is $\WBconf$-locally-loss-optimal on $\cD$}
\tag{by \Cref{thm:bcal_calibration_equivalences}} 
\end{align*}
\end{proof}

\subsection{Proof of \Cref{thm:ar-b-cal}:
A Simple Circuit for B-Confidence-Perturbations}
\label{app:ar-conf-proofs}

Recall \Cref{def:perturbation_operator} of the perturbation operator:
\begin{align}
\forall z \in \cV^N: \quad
(f \star w)[z] &:= \mathrm{softmax}\big(w[z] + \log f[z] \big) =
\frac{f[z] \exp(w[z])}{\sum_{z' \in \cV^N} f[z'] \exp(w[z'])} 
\end{align}
which highlights that this transformation is a multiplicative reweighting of the reference distribution $f$ by $e^{w[z]}$, followed by a renormalization to get a valid distribution. %
We will show that perturbations of this form can be implemented autoregressively via a small, efficient arithmetic circuit.
The key is to define two ``intermediate top-1 confidence'' vectors that can be tracked during generation.

\begin{definition}[Intermediate Top-1 Confidence]
\label{def:ar-b-conf-scalar}
Given a model $p_x$ and mapping $B_x$, let $\pi_x = B_x \sharp p_x$ be the initial categorical distribution, and let $k^\star:= \operatorname{argmax}_{k \in [K]} (\pi_x)_k$ be the single most likely category. We define:
\begin{enumerate}
    \item The top confidence value $v_x^\star \in [0,1]$, which is the model's confidence in this top category:
    \begin{equation}
        v_x^\star:= (\pi_x)_{k^\star}.
    \end{equation}
    \item The conditional probability of hitting the top category, $g_i^{(\mathrm{conf})}(x, z_{\leq i}) \in [0,1]$, which is the probability of eventually generating a sequence in category $k^\star$, given the prefix $z_{\leq i}$:
    \begin{equation}
        g_i^{(\mathrm{conf})}(x, z_{\leq i}):= \Pr_{z' \sim p_x(\cdot \mid z_{\leq i})}[B_x(z_{\leq i}, z') = k^\star].
    \end{equation}
\end{enumerate}
\end{definition}

With these scalars, the autoregressive update becomes a simple linear transformation.

\begin{theorem}
\label{thm:simple_conf_circuit}
For any perturbation $w \in \WBconf$ (defined by a function $\tau$), the perturbed next-token probability is proportional to the original probability modified by a simple scalar circuit $C_w$:
\begin{equation}
(p_x \star w_x)(z_i \mid z_{<i}) \propto p_x(z_i \mid z_{<i}) \cdot C_w(v_x^\star, g_i^{(\mathrm{conf})}(x, z_{\leq i})),
\end{equation}
where the circuit $C_w$ is a linear function of $g_i^{(\mathrm{conf})}$:
\begin{equation}
    C_w(v, g):= 1 + \left(\exp(\tau(v)) - 1\right) \times g.
\end{equation}
\end{theorem}

The following helper lemma will assist with the proof of \Cref{thm:simple_conf_circuit}:

\begin{lemma}[Autoregressive Decomposition of the Perturbation]
\label{lem:autoregressive_decomposition}
For any position $i$, the perturbed conditional probability of the next token is the original conditional probability multiplied by a ratio of
``lookahead expectations'':
\begin{equation}
(p_x \star w_x)(z_i \mid z_{<i})
= p_x(z_i \mid z_{<i}) \cdot 
\frac{\E_{z_{>i} \sim p_x(\cdot \mid z_{\leq i})} \!\big[\exp(w_x(z_{\le i}, z_{>i}))\big]}
{\E_{z_{\ge i} \sim p_x(\cdot \mid z_{<i})} \!\big[\exp(w_x(z_{<i}, z_{\ge i}))\big]}.
\end{equation}
\end{lemma}

\begin{proof}
Let $Z := \sum_{z} p_x(z) \, e^{w_x(z)}$. By definition of conditional probability,
\[
(p_x \star w_x)(z_i \mid z_{<i}) 
= \frac{(p_x \star w_x)(z_{\leq i})}{(p_x \star w_x)(z_{<i})}.
\]

Expanding the perturbation operator and applying 
$p_x(z_{\leq i}, z_{>i}) = p_x(z_{\leq i}) p_x(z_{>i}\mid z_{\leq i})$,
\begin{align*}
(p_x \star w_x)(z_{\leq i})
&= \frac{1}{Z} \sum_{z_{>i}} p_x(z_{\leq i}, z_{>i}) \, e^{w_x(z_{\leq i}, z_{>i})} \\
&= \frac{p_x(z_{\leq i})}{Z} \,
   \E_{z_{>i} \sim p_x(\cdot \mid z_{\leq i})}[e^{w_x(z_{\leq i}, z_{>i})}].
\end{align*}

Similarly,
\[
(p_x \star w_x)(z_{<i})
= \frac{p_x(z_{<i})}{Z} \,
  \E_{z_{\ge i} \sim p_x(\cdot \mid z_{<i})}[e^{w_x(z_{<i}, z_{\ge i})}].
\]

Taking the ratio and canceling $Z$,
\begin{align*}
(p_x \star w_x)(z_i \mid z_{<i})
&= \frac{p_x(z_{\leq i})}{p_x(z_{<i})}
   \cdot \frac{\E_{z_{>i} \sim p_x(\cdot \mid z_{\leq i})}[e^{w_x(z_{\leq i}, z_{>i})}]}
              {\E_{z_{\ge i} \sim p_x(\cdot \mid z_{<i})}[e^{w_x(z_{<i}, z_{\ge i})}]} \\
&= p_x(z_i \mid z_{<i})
   \cdot \frac{\E_{z_{>i} \sim p_x(\cdot \mid z_{\leq i})}[e^{w_x(z_{\leq i}, z_{>i})}]}
              {\E_{z_{\ge i} \sim p_x(\cdot \mid z_{<i})}[e^{w_x(z_{<i}, z_{\ge i})}]}.
\end{align*}
\end{proof}

Now we can proceed with the proof of \Cref{thm:simple_conf_circuit}.

\begin{proof}(\Cref{thm:simple_conf_circuit})
By Lemma~\ref{lem:autoregressive_decomposition},
\[
(p_x \star w_x)(z_i \mid z_{<i})
\;\propto\; p_x(z_i \mid z_{<i}) \cdot 
\E_{z \sim p_x(\cdot \mid z_{\leq i})}\big[\exp(w_x(z))\big].
\]

For $w \in \WBconf$ we have
\begin{align*}
w_x(z) &= c_x \cdot \mathds{1}\{B_x(z)=k^\star\}, 
\quad\text{with } c_x := \tau(v_x^\star). \\
\exp(w_x(z)) &= 1 + (\exp(c_x)-1)\cdot \mathds{1}\{B_x(z)=k^\star\}.
\end{align*}

Taking expectation under $z \sim p_x(\cdot \mid z_{\leq i})$ yields
\[
1 + (\exp(c_x)-1)\,\Pr[B_x(z)=k^\star \mid z_{\leq i}]
= 1 + (\exp(\tau(v_x^\star))-1)\, g_i^{(\mathrm{conf})}(x,z_{\leq i}).
\]

By Lemma~\ref{lem:autoregressive_decomposition}, the perturbed conditional
probability is the original $p_x(z_i \mid z_{<i})$
scaled by the ratio of this term to an analogous denominator depending only on
the prefix $z_{<i}$. Since the denominator is independent of $z_i$, it can be absorbed
into the overall proportionality constant.
\end{proof}

\subsection{Full calibration}
\label{app:full_b_cal}
In this section, we provide \emph{full calibration} analogs of our confidence-calibration results.
Confidence calibration is a weaker form of calibration that focuses only on the model's top prediction, while full calibration is a stronger notion that considers the probability placed on all classes. We begin by defining full calibration and applying it to the LLM setting to define \emph{$B$-calibration}.

\begin{definition}[Full Calibration]
\label{def:canonical_calibration}
A distribution $\cD$ over prediction-output pairs $(c, y) \in \Delta_K \times \eK$ is perfectly calibrated if the expected error, conditioned on the prediction, is the zero vector:
\begin{equation}
    \E_{(c, y) \sim \cD} \left[ y - c \mid c \right] \equiv 0.
\end{equation}
Note that since $y$ and $c$ are both vectors in $\mathbb{R}^K$, this subtraction is well-defined.
\end{definition}

Now, we apply this template to our LLM setting. We say a model is $B$-calibrated if the distribution it induces over the collapsed, semantic categories is itself perfectly calibrated.

\begin{definition}[$B$-Calibration]
\label{def:B-cal}
A model $p_\theta$ is $B$-calibrated on a distribution $\cD$ if the induced distribution over pairs $(\pi_x, B_x(y))$ is perfectly calibrated according to Definition~\ref{def:canonical_calibration}. Here, $\pi_x = B_x \sharp p_x$ takes the role of the prediction $c$, and the ground-truth category $B_x(y) \in [K]$ takes the role of the outcome $y$. Formally,
\begin{equation}
\E_{(x, y) \sim \cD} \left[ B_x(y) - \pi_x \mid \pi_x \right] \equiv 0.
\end{equation}
Following our convention, the scalar $B_x(y) \in [K]$ is identified with its one-hot vector in $\eK$ to perform the vector subtraction.
\end{definition}

\paragraph{Collapsing matrix}
To help with the remaining statements and proofs in this section we introduce a matrix representation of the collapsing function $B$. Recall from \Cref{eq:pi_c_explicit} that $\pi_x$ assigns the explicit probabilities
\begin{equation*}
\pi_x(c) = \left(B_x \sharp p_x\right)(c)
= \Pr_{z \sim p_\theta(\cdot \mid x)}\left[B_x(z) = c\right]
= \sum_{z : B_x(z) = c} p_\theta(z \mid x), \quad \text{for each $c \in [K]$.}
\end{equation*}

This push-forward operation can be written in matrix form.
Define the collapsing matrix $\mathbf{B}_x$ as:
\begin{equation}\label{eq:collapsing_matrix}
  \mathbf{B}_x \in \{0,1\}^{K \times M},
  \qquad
  [\mathbf{B}_x]_{k,z} = \mathds{1}_{\{B_x(z) = k\}}.
\end{equation}

Then the pushforward distribution and ground-truth semantic class can be
expressed as
$$
\pi_x = \mathbf{B}_x p_x \in \Delta_K,
\qquad
\mathbf{B}_x \tilde y = e_{B_x(y)} \in \eK.
$$

Thus, matrix-vector multiplication exactly implements the pushforward operation:
\[
(\pi_x)_k = \sum_{z : B_x(z) = k} p_\theta(z \mid x)
= [\mathbf{B}_x p_x]_k.
\]

\subsubsection{Equivalence between $B$-calibration and weighted calibration}

In this section, we provide a result analogous to \Cref{thm:conf_cal_equivalence} connecting $B$-calibration with weighted calibration.

\begin{definition}[Semantic Perturbation Function Classes; Full Calibration]
\label{def:bcal_perturbation_classes}
Given an arbitrary function $B_x(z) \in [K]$, which we think of as a semantic 
collapsing function,
we define the $B$-induced weighted function class (a class of perturbation functions $w(x, p_x)$ that generate a perturbation vector based on the context $x$ and the model's predictive distribution $p_x$):
\begin{equation}
    \WBcal= \left\{ w_\tau \mid w_\tau(x, p_x)[z] = \tau(\pi_x)[B_x(z)]
    \text{ for some } \tau: \Delta^K \to [-1, 1]^K\right\}.
\end{equation}
\end{definition}

Intuitively, every sequence $z$ is assigned a weight based on its semantic category $B_x(z) \in [K]$, and the weighting scheme itself can adapt based on the model's overall categorical prediction $\pi_x$.

\begin{lemma}
\label{lem:weight_function_vector_form}
Let $w \in \WBcal$ be a weight function defined by $w(x, p_x)[z] = \tau(\pi_x)[B_x(z)]$. Its corresponding vector representation is given by $\mathbf{B}_x^\top \tau(\pi_x)$.
\end{lemma}

\begin{proof}
We will prove the equivalence by showing that for any sequence $z \in \cV^N$, the $z$-th component of the vector $\mathbf{B}_x^\top \tau(\pi_x)$ is equal to $\tau(\pi_x)[B_x(z)]$. Let $u = \tau(\pi_x)$, which is a vector in $\mathbb{R}^K$. 

Now, we want to analyze the components of the vector $v = \mathbf{B}_x^\top u$.

For any $z \in \cV^N$, the $z$-th component of $v$ is given by the definition of matrix-vector multiplication:
\begin{align*}
[v]_z &= [\mathbf{B}_x^\top u]_z 
= \sum_{k=1}^{K} [\mathbf{B}_x^\top]_{z,k} \cdot u_k  
= \sum_{k=1}^{K} [\mathbf{B}_x]_{k,z} \cdot u_k  
= \sum_{k=1}^{K} \mathds{1}_{\{B_x(z) = k\}} \cdot u_k
\end{align*}
where the last equality is by definition of $\mathbf{B}_x$; see \Cref{eq:collapsing_matrix}.
The indicator function $\mathds{1}_{\{B_x(z) = k\}}$ is non-zero for only one value of $k$ in the sum, namely when $k$ is equal to the category of the sequence $z$, i.e., $k = B_x(z)$. Therefore, the sum collapses to a single term:
\begin{align*}
[v]_z &= 1 \cdot u_{B_x(z)} + \sum_{B_x(z) \neq k} 0 \cdot u_k = u_{B_x(z)}.
\end{align*}
Substituting back the definition of $u = \tau(\pi_x)$, we get:
$[v]_z = \tau(\pi_x)[B_x(z)]$.
This expression matches the definition of $w(x, p_x)[z]$ exactly.

Since this holds for all sequences $z$, the vector $\mathbf{B}_x^\top \tau(\pi_x)$ is the vector representation of the function $w(x, p_x)$.
\end{proof}

With the definition of the weighted class and its vector representation, we can state the main equivalence theorem (analogous to \Cref{thm:conf_cal_equivalence}).

\begin{theorem}[B-Calibration as Weighted Calibration]
\label{thm:bcal_equivalence}
A model $p_\theta$ is perfectly $B$-calibrated if and only if it is perfectly $\WBcal$-weighted-calibrated.
\end{theorem}

\begin{proof}
We start from the definition of $B$-calibration, which (as established in Definition~\ref{def:canonical_calibration}) is formally expressed as a vector condition:
\begin{equation*}
\E \left[ e_{B_x(y)} - \pi_x \mid \pi_x \right] = 0.
\end{equation*}
By the properties of conditional expectation, this holds if and only if for all functions $\tau: \Delta_K \to [-1,1]^K$, it holds
\begin{equation}
\label{eq:proof_start}
\E \left[ \langle e_{B_x(y)} - \pi_x, \tau(\pi_x) \rangle \right] = 0.
\end{equation}
Substituting the matrix representation into \Cref{eq:proof_start}:
\begin{align*}
\E \left[ \langle e_{B_x(y)} - \pi_x, \tau(\pi_x) \rangle \right] = 0 &\iff 
\E \left[ \langle \mathbf{B}_x \tilde{y} - \mathbf{B}_x p_x, \tau(\mathbf{B}_x p_x) \rangle \right] = 0 \\
&\iff \quad \E \left[ \langle \mathbf{B}_x (\tilde{y} - p_x), \tau(\mathbf{B}_x p_x) \rangle \right] = 0  \\
&\iff \quad \E \left[ \langle \tilde{y} - p_x, \mathbf{B}_x^\top \tau(\mathbf{B}_x p_x) \rangle \right] = 0
\end{align*}
From \Cref{{lem:weight_function_vector_form}}, the term $\mathbf{B}_x^\top \tau(\mathbf{B}_x p_x)$ is precisely the vector representation of the function $w(x, p_x)$ from \Cref{def:bcal_perturbation_classes}. Thus, the condition is equivalent to:
\begin{equation*}
\E \left[ \langle \tilde{y} - p_x, w(x, p_x) \rangle \right] = 0, \quad \text{for all } w \in \WBcal,
\end{equation*}
which is exactly the definition of $\WBcal$-weighted-calibration ; see \Cref{def:weighted-cal}.

\end{proof}

\subsubsection{A Simple Circuit for $B$-Perturbations}

Given a model $p_x$ and a semantic mapping $B_x$, we define two ``intermediate B-confidence'' vectors as follows:
\begin{enumerate}
    \item The initial B-confidence $g_0(x) \in \Delta_K$, which is the model's overall predicted distribution on the $K$ categories before generation begins. This corresponds to the $B$-induced pushforward distribution $\pi_x = B_x \sharp p_x$:
    \begin{equation}
        g_0(x)[b] := \Pr_{z \sim p_x}[B_x(z) = b].
    \end{equation}
    \item The conditional B-confidence $g_i(x, z_{\leq i}) \in \Delta_K$, which is the model's predicted distribution on categories, conditioned on having generated the prefix $z_{\leq i}$:
    \begin{equation}
        g_i(x, z_{\leq i})[b] := \Pr_{z' \sim p_x(\cdot \mid z_{\leq i})}[B_x(z_{\leq i}, z') = b].
    \end{equation}
\end{enumerate}

\begin{theorem}[Simple Circuit for B-Perturbations]
\label{thm:conf_circuit}
For any perturbation $w \in \cW_B$ (defined by a scaling function $\tau$), 
the perturbed next-token probability is proportional to the original conditional probability 
multiplied by a simple circuit $C_w$:
\begin{equation}
(p_x \star w_x)(z_i \mid z_{<i}) 
\propto p_x(z_i \mid z_{<i}) \cdot C_w(g_0(x), g_i(x, z_{\leq i})),
\end{equation}
where the constant of proportionality does not depend on $z_i$, and
\begin{equation}
C_w(g_0, g_i) 
= \sum_{b=1}^{K} \exp(\tau(g_0)[b]) \cdot g_i[b].
\end{equation}
This circuit has constant depth and width linear in $K$.
\end{theorem}

\begin{proof}
From \Cref{lem:autoregressive_decomposition}, we know that
\[
(p_x \star w_x)(z_i \mid z_{<i})
= p_x(z_i \mid z_{<i}) \cdot
  \frac{\E_{z \sim p_x(\cdot \mid z_{\leq i})}[e^{w_x(z_{\leq i},z)}]}
       {\E_{z \sim p_x(\cdot \mid z_{<i})}[e^{w_x(z_{<i},z)}]}.
\]

For $w \in \cW_B$, by definition,
$
w_x(z) = \tau(g_0(x))[B_x(z)]
\quad\text{where}\quad g_0(x) = B_x \sharp p_x.
$

Expanding the expectation,
\begin{align*}
\E_{z \sim p_x(\cdot \mid z_{\leq i})}[e^{w_x(z_{\leq i},z)}]
&= \E_{z \sim p_x(\cdot \mid z_{\leq i})}[e^{\tau(g_0(x))[B_x(z_{\leq i},z)]}] \\
&= \sum_{b=1}^K \Pr[B_x(z_{\leq i},z)=b] \cdot e^{\tau(g_0(x))[b]} \\
&= \sum_{b=1}^K g_i(x,z_{\leq i})[b] \cdot e^{\tau(g_0(x))[b]}.
\end{align*}

The denominator is an expectation over $z\sim p_x(\cdot | z_{<i})$,
which depends only on the prefix $z_{<i}$ and not on the choice of $z_i$. 
Hence it is a constant with respect to $z_i$ and can be absorbed into the proportionality.
Therefore,
$
(p_x \star w_x)(z_i \mid z_{<i})
\propto p_x(z_i \mid z_{<i}) \cdot \langle \exp(\tau(g_0(x))), g_i(x, z_{\leq i}) \rangle.
$
\end{proof}

\subsection{Quantitative Bounds on Multi-Class Calibration and Post-Processing Gap for Proper Losses}
\label{app:calib_gap_proper}

Beyond cross-entropy loss, we provide in this section a generalization for the class of proper loss functions and quantitative bounds relating post-processing and calibration gap.
The main result in this section, \Cref{thm:multi_class_cal} should be interpreted
as a generalization of Theorem E.3 in \citet{blasiok2023when}
to the multi-class setting, and a robust version of \Cref{thm:bcal_calibration_equivalences}:
it essentially states that a model is ``close to'' $\cW$-weighted-calibrated
if it is ``close to'' $\cW$-loss-optimal.

First, we recall a standard result on convex representation of proper losses \citep{savage1971elicitation, schervish1989general, gneiting2007strictly}.

\begin{definition}[Savage representation]
\label{def:proper_cvx_rep}
A loss function $\ell: \{e_1, \dots, e_K\} \times \Delta_K \to \mathbb{R}$ is \emph{proper} iff there exists a convex function $\phi: \Delta_K \to \mathbb{R}$ such that
\[
    \ell(y, v) \;=\; -\phi(v) + \langle v-y, \nabla \phi(v)\rangle.
\]
\end{definition}

Next, define the convex conjugate $\psi = \phi^*$, a dual variable, and the dual form of the loss.

\begin{definition}[Dual loss]
\label{def:dual_loss}
For a proper loss $\ell$ with potential $\phi$ as in \Cref{def:proper_cvx_rep}, define:
\begin{align*}
    \text{Convex conjugate:} \quad &\psi(u) := \phi^*(u) := \sup_{v \in \Delta_K} \big(\langle u,v\rangle - \phi(v)\big), \\
    \text{Dual variable:} \quad &\text{dual}(v) := \nabla \phi(v), \\
    \text{Dual loss:} \quad &\ell^{(\psi)}(y, z) := \psi(z) - \langle y, z\rangle.
\end{align*}
\end{definition}

\begin{remark}
The dual parameterization of \Cref{def:dual_loss} satisfies:
\begin{enumerate}
    \item Agreement between primal and dual losses:
    \(\ell^{(\psi)}(y, \text{dual}(v)) = \ell(y,v).\)
    \item Probability $\to$ dual map: \(\text{dual}(v) = \nabla \phi(v)\) for all \(v\in \Delta_K\).
    \item Dual $\to$ probability map: \(v = \nabla \psi(\text{dual}(v))\) for all \(v\in \Delta_K\).
\end{enumerate}
\end{remark}

\begin{definition}[Generalized dual calibration and post-processing gap]
\label{def:gen_dual_post_gap}
Let $\mathcal{W}$ be a class of functions $w: \mathcal{X}\times \mathbb{R}^K \to \mathbb{R}^K$, 
and let $\mathcal{D}$ be a distribution over $\mathcal{X}\times \{e_1,\dots,e_K\}$. 

For a predictor $f:\mathcal{X}\to \Delta_K$, let $g:\mathcal{X}\to \mathbb{R}^K$ be its dual representation such that
\[
f(x) = \nabla \psi(g(x)) \quad\forall x\in\mathcal{X}.
\]

Define for shorthand
\[
\Delta(w) := \mathbb{E}_{(x,y)\sim\mathcal{D}}
    \big[\langle y - f(x), w(x,g(x))\rangle\big], 
\qquad
\mathcal{L}(h) := \mathbb{E}_{(x,y)\sim\mathcal{D}}
    [\ell^{(\psi)}(y,h(x))].
\]

\begin{itemize}
\item The \emph{dual calibration error} of $g$ with respect to $\mathcal{W}$ is
\begin{equation}
\label{eq:gen_ce}
\mathrm{CE}(g; \mathcal{W}) := \sup_{w\in\mathcal{W}} |\Delta(w)|.
\end{equation}

\item The \emph{dual post-processing gap} of $g$ with respect to a function class $\mathcal{H}$ is
\begin{equation}
\label{eq:gen_gap}
\mathrm{Gap}(g; \mathcal{H}) := \mathcal{L}(g) - \inf_{h\in\mathcal{H}} \mathcal{L}(h).
\end{equation}
\end{itemize}
\end{definition}

\begin{theorem}[General relationship between calibration and post-processing]
\label{thm:multi_class_cal}
Let $\psi: \mathbb{R}^K \to \mathbb{R}$ be differentiable and $\lambda$-smooth, i.e.\ $\nabla \psi$ is $\lambda$-Lipschitz. 
Let $\mathcal{W}$ be a class of bounded functions $w:\mathcal{X}\times\mathbb{R}^K \to \mathbb{R}^K$ 
with $\|w_x\|\leq 1$.  
For $w\in\mathcal{W}$ and $\beta\in[-1/\lambda,1/\lambda]$, define the perturbed dual predictor
\begin{equation}
\label{eq:perturb_def}
g_w(x) := g(x) + \beta\,w(x,g(x)).
\end{equation}
Let $\mathcal{G}_{\mathcal{W}} := \{ g_w : w\in\mathcal{W},\; \beta\in[-1/\lambda,1/\lambda]\}$.
Then, for every $g:\mathcal{X}\to\mathbb{R}^K$ and distribution $\mathcal{D}$,
\begin{equation}
\label{eq:main_theorem_bounds}
\frac{1}{2}\,\Big(\mathrm{CE}(g; \mathcal{W})\Big)^2 
\;\leq\; \lambda \cdot \mathrm{Gap}(g; \mathcal{G}_{\mathcal{W}}) 
\;\leq\; \mathrm{CE}(g; \mathcal{W}).
\end{equation}
\end{theorem}

\begin{proof}
By the definition of $\ell^{(\psi)}$,
\begin{align*}
\mathcal{L}(g) - \mathcal{L}(g_w) 
&= \mathbb{E}\!\big[\psi(g(x)) - \langle y,g(x)\rangle - \psi(g_w(x)) + \langle y, g_w(x)\rangle\big] \\
&= \mathbb{E}\!\big[\psi(g(x)) - \psi(g_w(x)) + \beta \langle y, w(x,g(x))\rangle\big].
\end{align*}

By convexity and $\lambda$-smoothness of $\psi$, for $z=g(x),\,z' = g_w(x)$ and $w_x = w(x,g(x))$
\[
\langle \nabla\psi(z), \beta w_x\rangle 
\;\leq\; \psi(z') - \psi(z) 
\;\leq\; \langle \nabla\psi(z), \beta w_x\rangle + \frac{\lambda\beta^2}{2}\|w_x\|^2.
\]

Since $f(x)=\nabla\psi(g(x))$ and $\|w_x\|\leq 1$, this yields
\[
\beta\,\Delta(w) - \frac{\lambda\beta^2}{2}
\;\leq\; \mathcal{L}(g) - \mathcal{L}(g_w) 
\;\leq\; \beta\,\Delta(w).
\]

\emph{Lower bound.} For $w\in\mathcal{W}$, set $\beta=\Delta(w)/\lambda$ (which lies in $[-1/\lambda,1/\lambda]$). Then
\[
\frac{1}{2\lambda}\,\Delta(w)^2 \;\leq\; \mathcal{L}(g) - \mathcal{L}(g_w).
\]
Taking $\sup_{w\in\mathcal{W}}$ yields
\[
\frac{1}{2}\,\big(\mathrm{CE}(g; \mathcal{W})\big)^2
\;\leq\; \lambda \cdot \mathrm{Gap}(g; \mathcal{G}_{\mathcal{W}}).
\]

\emph{Upper bound.} For $g_w\in\mathcal{G}_{\mathcal{W}}$, since $|\beta|\leq 1/\lambda$
\[
\mathcal{L}(g) - \mathcal{L}(g_w) \;\leq\; \beta\,\Delta(w)
\leq\; \frac{1}{\lambda}\,|\Delta(w)|.
\]
Taking $\sup_{w\in\mathcal{W}}$ gives
\[
\lambda \cdot \mathrm{Gap}(g; \mathcal{G}_{\mathcal{W}}) \;\leq\; \mathrm{CE}(g; \mathcal{W}).
\]

Combining the upper and lower bounds proves \Cref{eq:main_theorem_bounds}.
\end{proof}

\begin{remark}[Tighter exponent under strong convexity]
If, in addition, $\psi$ is $\mu$-strongly convex for some $\mu > 0$ i.e.
$$\psi(z') \ge \psi(z) + \langle \nabla \psi(z), z'-z \rangle 
+ \tfrac{\mu}{2}\|z'-z\|^2,$$
then one obtains matching upper and lower bounds. 
In this case, both inequalities in \Cref{thm:multi_class_cal} become quadratic in the calibration error:
\[
  \frac{\mu}{2\lambda^2}\,\big(\mathrm{CE}(g; \mathcal{W})\big)^2
  \;\leq\; \mathrm{Gap}(g; \mathcal{G}_{\mathcal{W}})
  \;\leq\; \frac{1}{2\mu}\,\big(\mathrm{CE}(g; \mathcal{W})\big)^2.
\]
That is, the dual post-processing gap and the squared dual calibration error 
are equivalent up to constants determined by $(\mu,\lambda)$.
\end{remark}

\subsubsection{Specialization to cross-entropy loss}
\label{subsec:cross_entropy}

For completeness, we summarize the standard facts about the dual parametrization of the negative log-loss in \Cref{tab:xent_duality}.

\begin{table}[h!]
\centering
\caption{Duality relationships for the Negative Log-Loss (Cross-Entropy) proper scoring rule. \\}
\label{tab:xent_duality}
\renewcommand{\arraystretch}{1.8}
\begin{tabular}{ll}
\toprule
Primal Proper Loss ($\ell_{\text{nll}}$) & $\ell(y,v) = - \sum_{i=1}^K y_i \log v_i$ \\
\addlinespace
Convex Function ($\phi$) & $\phi(v) = \sum_{i=1}^K v_i \log(v_i) \quad$ (Negative Entropy) \\
\addlinespace
Convex Conjugate ($\phi^*$) & $\phi^*(z) = \log\left(\sum_{i=1}^K \exp(z_i)\right) \quad$ (Log-Sum-Exp) \\
\addlinespace
Dual Loss ($\ell^*_{\text{nll}}$) & $\ell^*(y,z) = \phi^*(z) - y^T z$ \\
\addlinespace
Dual Mapping ($\nabla \phi^*$) & $\nabla\phi^*(z) = \text{softmax}(z)$ \\
\bottomrule
\end{tabular}
\end{table}

The log-sum-exp function $\phi^*(z) = \log\left(\sum_{i=1}^K \exp(z_i)\right)$ is $1/4$-smooth, as shown in \citet{beck2003mirror} and \citet{nesterov2005smooth}, so \Cref{thm:multi_class_cal} applies with $\lambda = 1/4$.
Moreover, to translate the result into the notation of our main theorems, recall the relationship between the primal prediction $f(x)$ and its dual representation $g(x)$:
\begin{align*}
    f(x) &= \nabla\phi^*(g(x)) = \text{softmax}(g(x)) \\
    g(x) &= \log(f(x))
\end{align*}
The perturbed loss can then be expressed in terms of the dual variables. The dual loss on perturbed logits $g+w$ is equivalent to the primal loss on the perturbed probability distribution $f \star w$:
$$
\ell^*_{\text{nll}}(y, g + w) = \ell_{\text{nll}}(y, \text{softmax}(g+w)) = \ell_{\text{nll}}(y, f \star w)
$$
where $f \star w = \text{softmax}(\log(f) + w)$.

\newpage

\subsection{Conformal Prediction via Weighted Calibration}
\label{app:conformal}

Here we observe that conformal prediction guarantees
can be expressed as a type of \emph{weighted calibration} \citep{gopalan2024computationally},
for a particular weight family.

Recall conformal prediction asks for a model $F(x)$ which outputs a \emph{set} of
labels, with the guarantee that this set contains the true label with high probability.
Specifically, a conformal predictor has \emph{coverage $\alpha$} if:
\[
\Pr_{x, y \sim \cD}[ y \in F(x) ] \geq 1-\alpha.
\]
For an introduction to conformal prediction, see \citet{angelopoulos2023conformal}
or the lecture notes of \citet{tibshirani2023conformal}.

\subsubsection{Conformal Prediction from Full Calibration}
\label{subsec:Conformal_Calibration}
Given a standard predictor $f$, which outputs a distribution on labels, one natural
way to construct a conformal predictor $F_\alpha$ is: given input $x$, and prediction $f(x)$,
output the set of highest-predicted-probability labels which sum to total probability
$1-\alpha$. This means, outputting the $K$ most-likely classes according to $f(x)$,
where $K$ is chosen per-sample based on the predicted probabilities.

The first observation (which is folklore) is:
if the predictor $f$ is perfectly calibrated, in the sense of full-calibration,
then the induced conformal predictor $F_\alpha$ is correct (i.e. has coverage $\alpha$).
This statement is not very relevant in practice, since full calibration is often too strong to hold.
However, we can achieve the same result with a weaker notion of calibration.
This is a straightforward result; we sketch the argument below.

\subsubsection{Conformal Prediction from Weighted Calibration}

\begin{lemma}
\label{lemma:conformal}
Suppose $f: \cX \to \Delta_N$ is perfectly weighted-calibrated (in the sense of \citet{gopalan2024computationally}) with respect to the
following family of weight functions $w(f) \in \R^N$:
\begin{align}
\label{eqn:W-conformal}
\cW := \{ w(f) = \sigma\1_{T_\alpha(f)}  \mid \alpha \in [0, 1], \sigma \in \{\pm 1\} \}
\end{align}

Where $\1_T \in \{0, 1\}^N$ is the indicator-vector for set of indices $T$, and the set $T$ contains
the highest-probability labels, defined as:
\begin{align}
    t_\alpha^*(f) &:= \max \{t: \left( \sum_{i \in [N]} f_i \1\{f_i \geq t\} \right) \geq 1-\alpha \} \tag{the threshold probability, given $f$}\\
    T_\alpha(f) &:= \{i : f_i \geq t^*_\alpha(f) \} \tag{The set of top-class indices, for given level $\alpha$}
\end{align}

That is, suppose:
\begin{align*}
\E_{(x,y) \sim \cD} \left[ 
\left\langle y - f(x), w(f(x)) \right\rangle \right] \equiv 0
\end{align*}

Then, the induced conformal predictor $F_\alpha$ of $f$ is valid at all coverage levels $\alpha$.
\end{lemma}

\begin{proof}(Sketch)
Notice that by construction, 
$\langle f, \1_{T_\alpha(f)} \rangle \geq 1-\alpha$.
Therefore by calibration we must have:
$\langle y, \1_{T_\alpha(f)} \rangle \geq 1-\alpha$.

Moreover, the set $T_\alpha(f)$ is exactly the output of the induced conformal predictor
$F_\alpha$, given base prediction $f$.
Therefore 
\begin{align}
    \Pr[y \in T_\alpha(f(x))] &= \E[ \langle y, \1_{T_\alpha(f)} \rangle ] \\
    &\geq 1-\alpha
\end{align}
\end{proof}

By the general connection of Theorem~\ref{thm:bcal_calibration_equivalences},
if a model $f$ is $\cW$-locally-loss-optimal w.r.t. the weight class of Equation~\eqref{eqn:W-conformal},
then the induced conformal predictor $F_\alpha$ has coverage $\alpha$ for all $\alpha \in [0, 1]$.

\section{Disaggregated Reliability Diagram Results}
\label{app:encyclopedia}

In this section, we report disaggregated reliability diagram results for individual configurations we evaluated. The plots are displayed as follow: 
\begin{itemize}
\item the right three columns present results for instruct models, 
\item the left three columns present results for the corresponding base models.
\end{itemize}
In some cases, there are multiple instruct models trained from a single base models, hence for some base models, their results are being presented multiple times.

Some instruct models do not have a public corresponding base model---in those cases, the left three columns of the row are empty.

As discussed in the \Cref{sec:experiments}, TriviaQA and SimpleQA were not evaluated for the CoT response style.

The figures start on the next page. For a quick references:

\begin{itemize}
\item \texttt{GSM8K} in \Cref{subsec:GSM8K}
\item \texttt{OpenMathInstruct} in \Cref{subsec:OpenMathInstruct}
\item \texttt{TriviaQA} in \Cref{subsection:TriviaQA}
\item \texttt{SimpleQA} in \Cref{subsection:SimpleQA}
\end{itemize}
\clearpage

\subsection{GSM8K}
\label{subsec:GSM8K}

\begin{figure}[!htb]
\centering
\includegraphics[width=\linewidth]{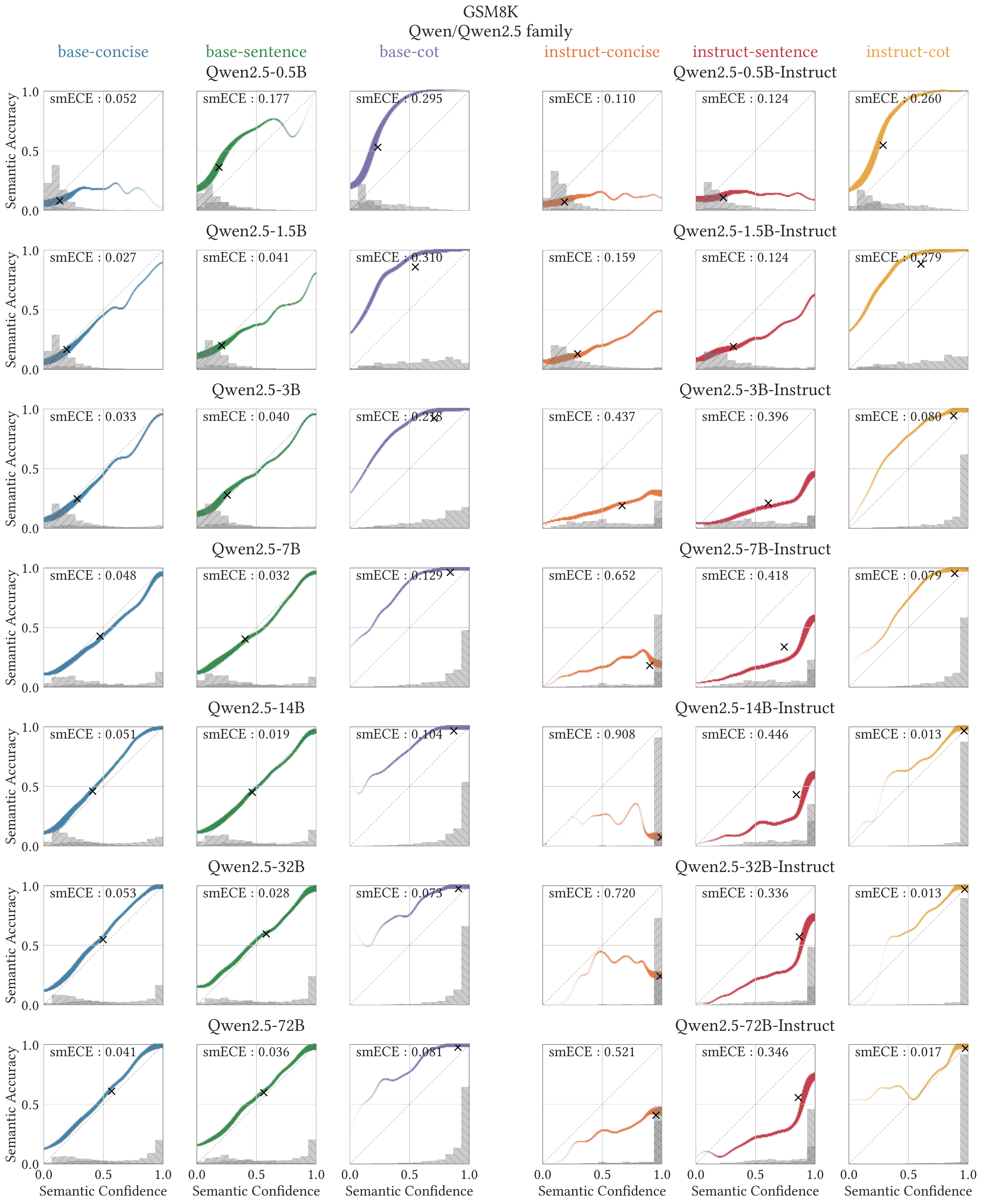}
\end{figure}
\begin{figure}[!htb]
\centering
\includegraphics[width=\linewidth]{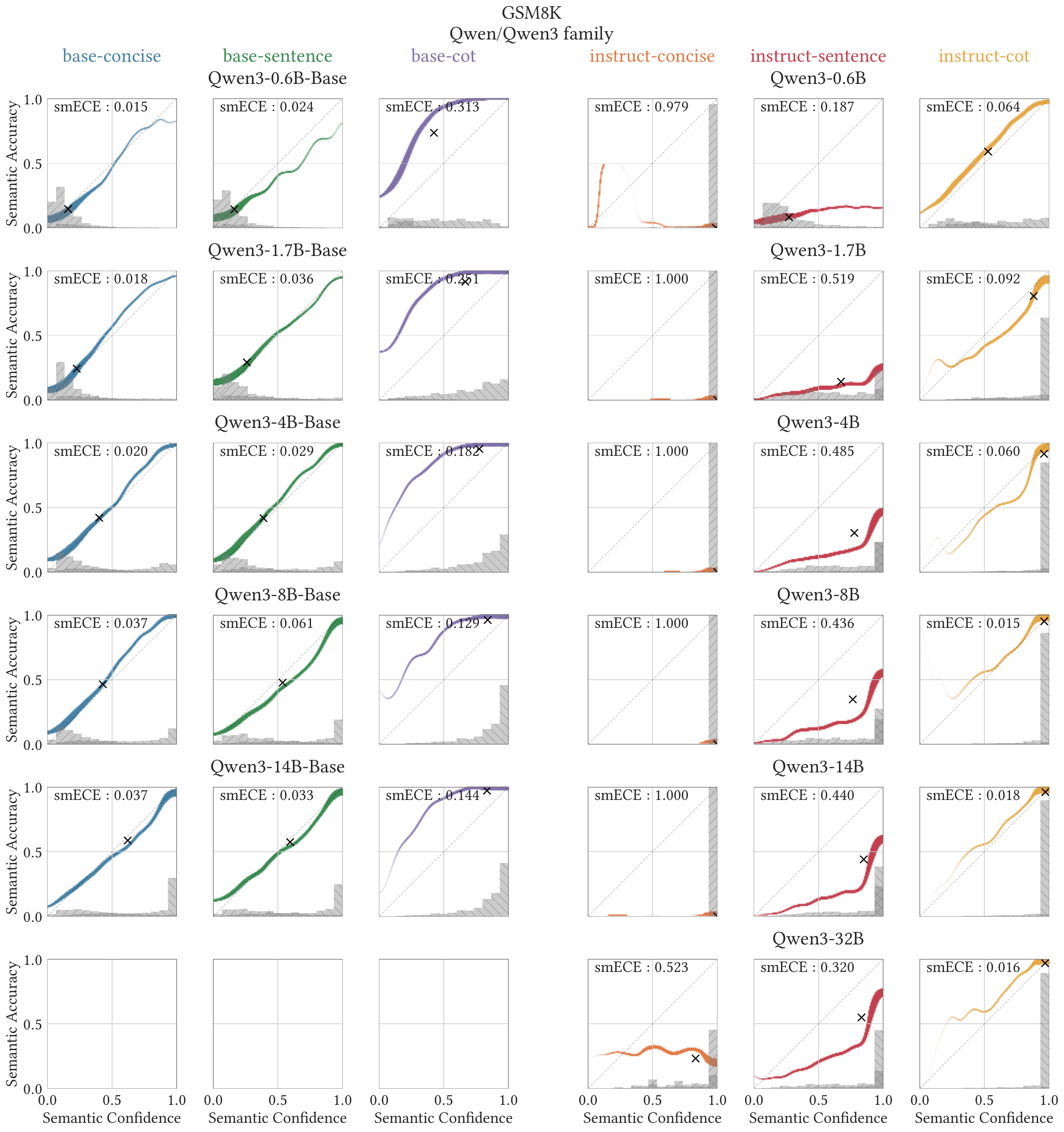}
\end{figure}
\begin{figure}[!htb]
\centering
\includegraphics[width=\linewidth]{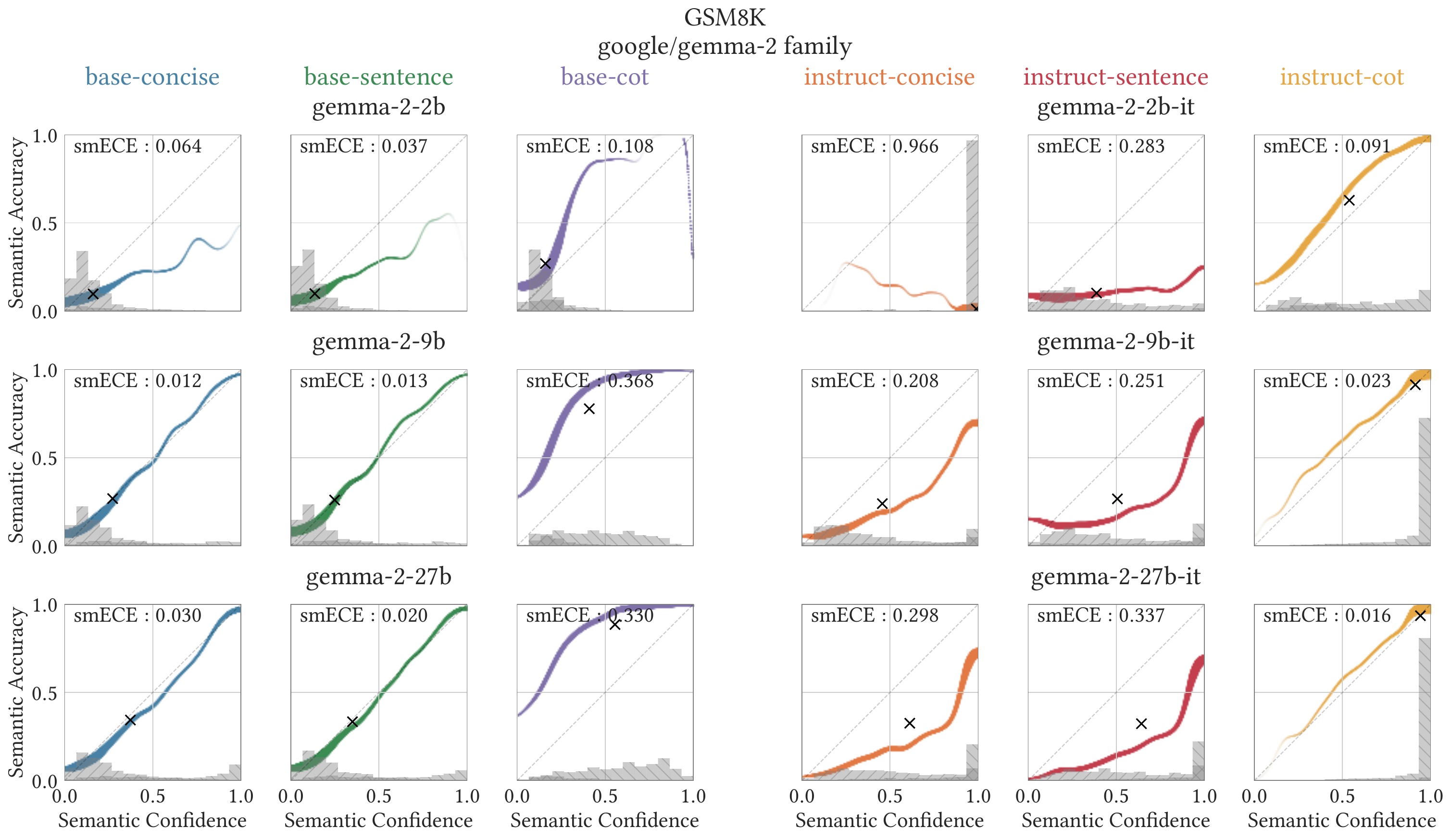}
\end{figure}
\begin{figure}[!htb]
\centering
\includegraphics[width=\linewidth]{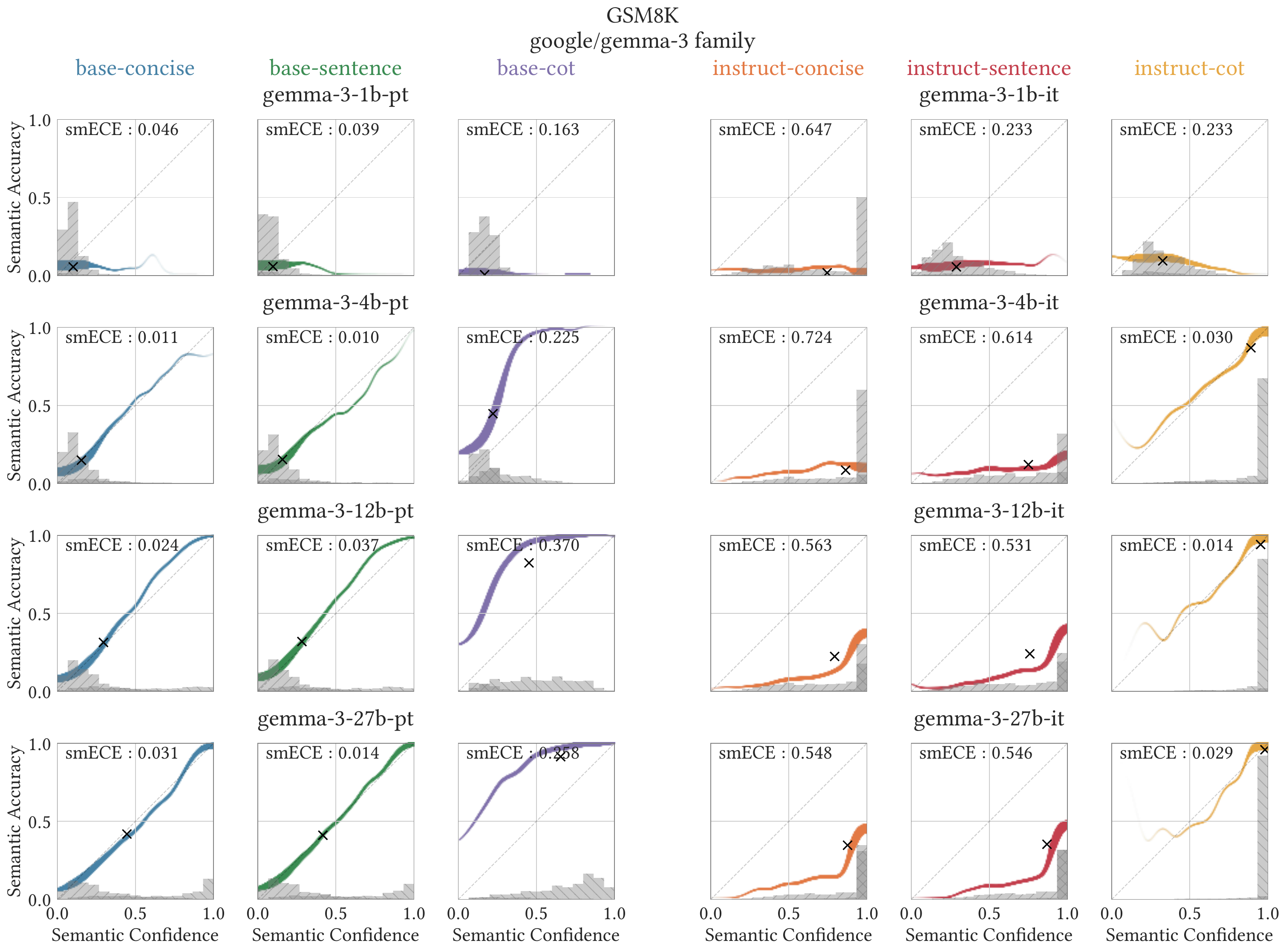}
\end{figure}
\begin{figure}[!htb]
\centering
\includegraphics[width=\linewidth]{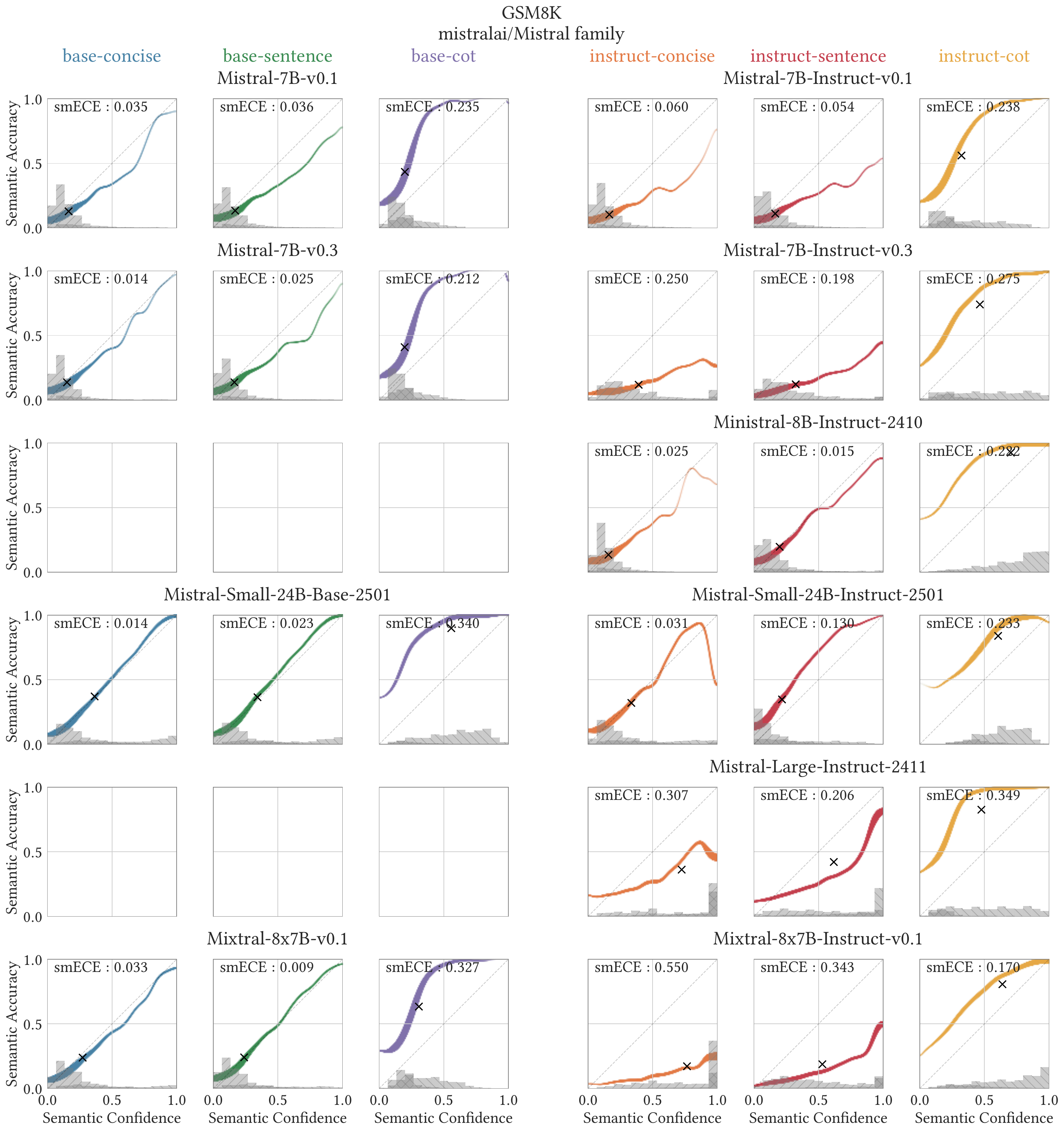}
\end{figure}
\begin{figure}[!htb]
\centering
\includegraphics[width=\linewidth]{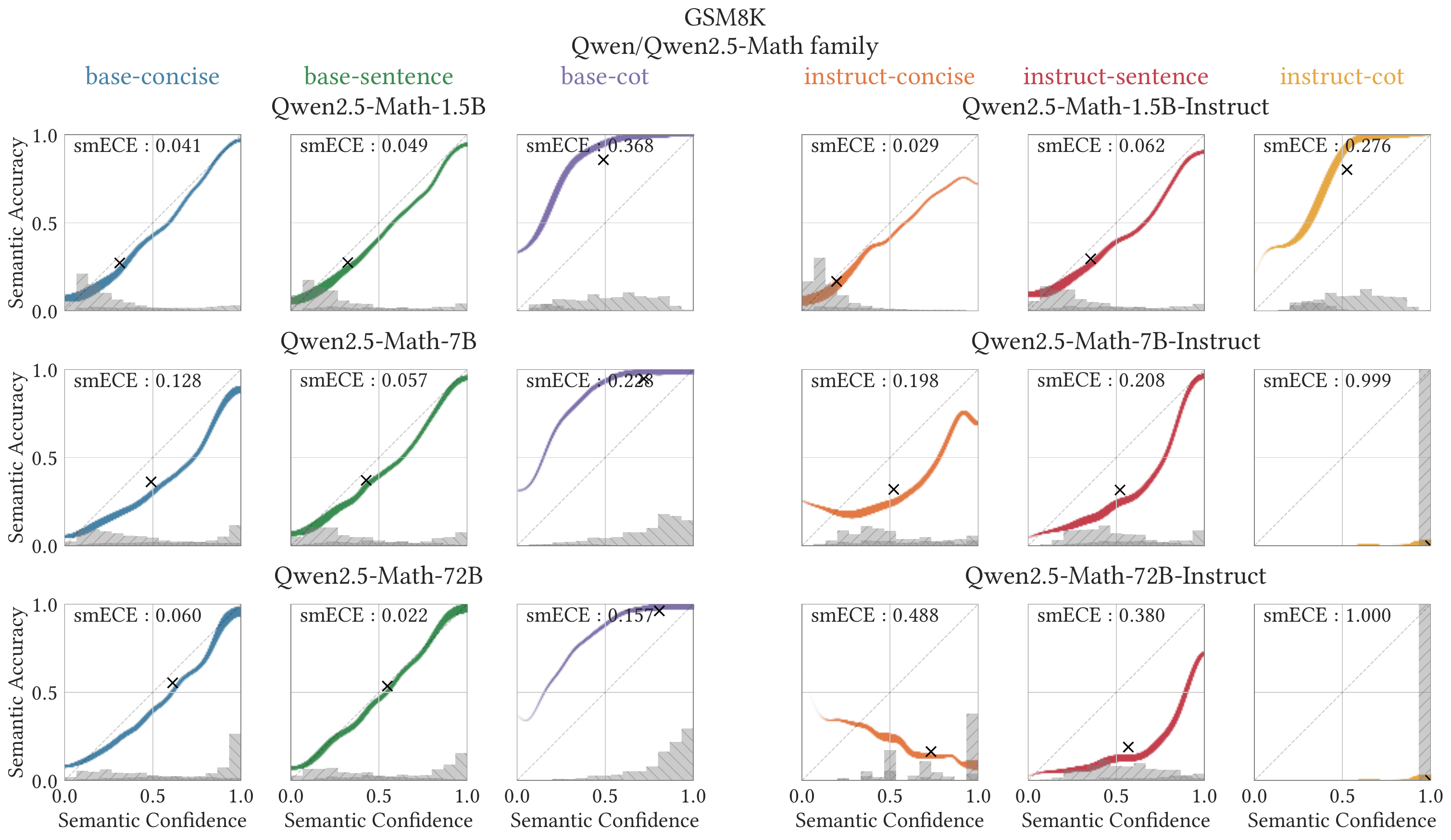}
\end{figure}
\begin{figure}[!htb]
\centering
\includegraphics[width=\linewidth]{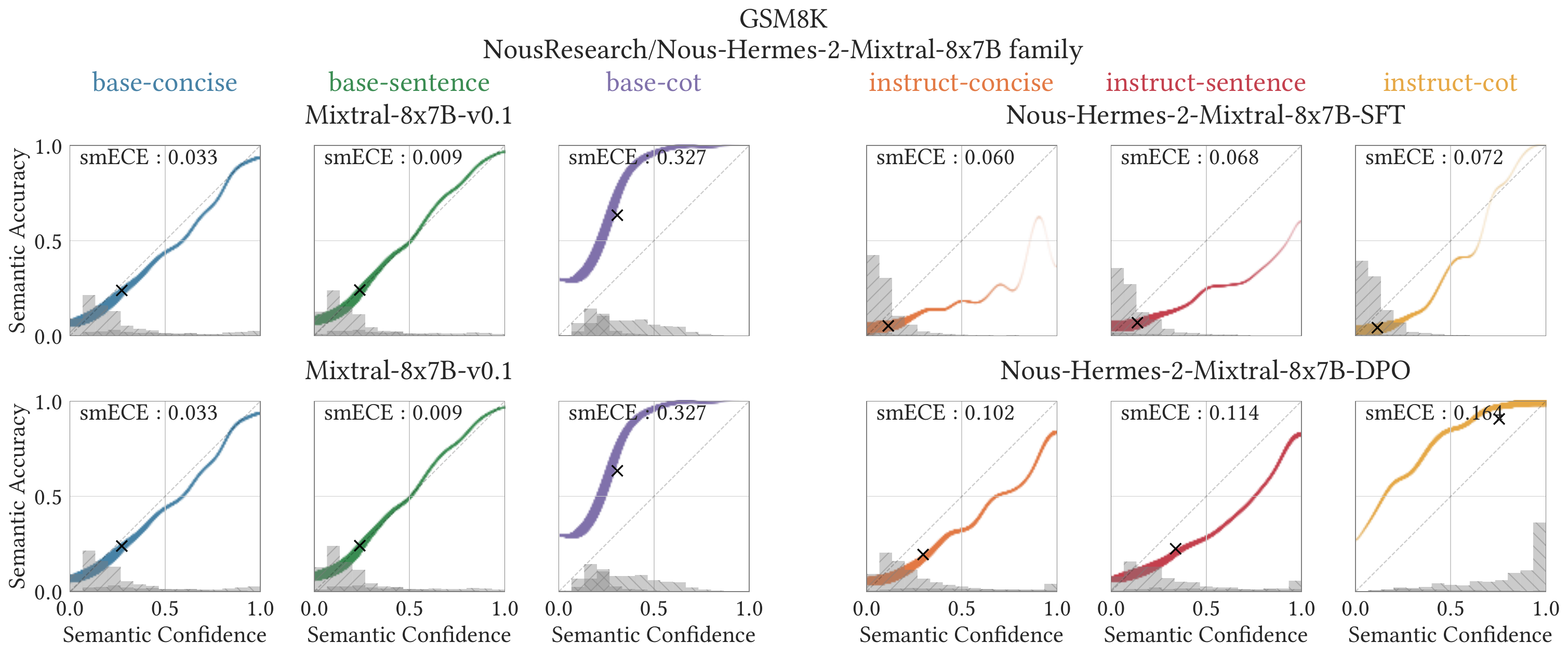}
\end{figure}
\begin{figure}[!htb]
\centering
\includegraphics[width=\linewidth]{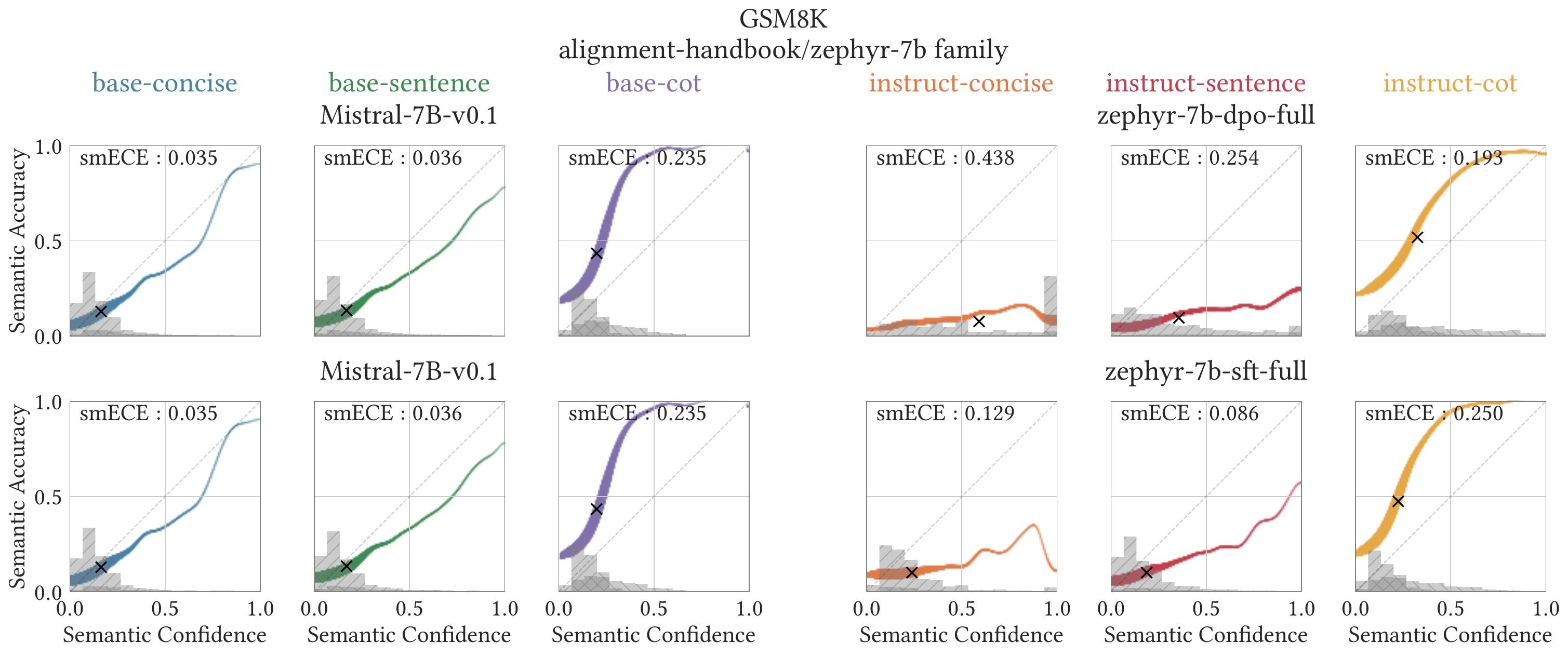}
\end{figure}
\begin{figure}[!htb]
\centering
\includegraphics[width=\linewidth]{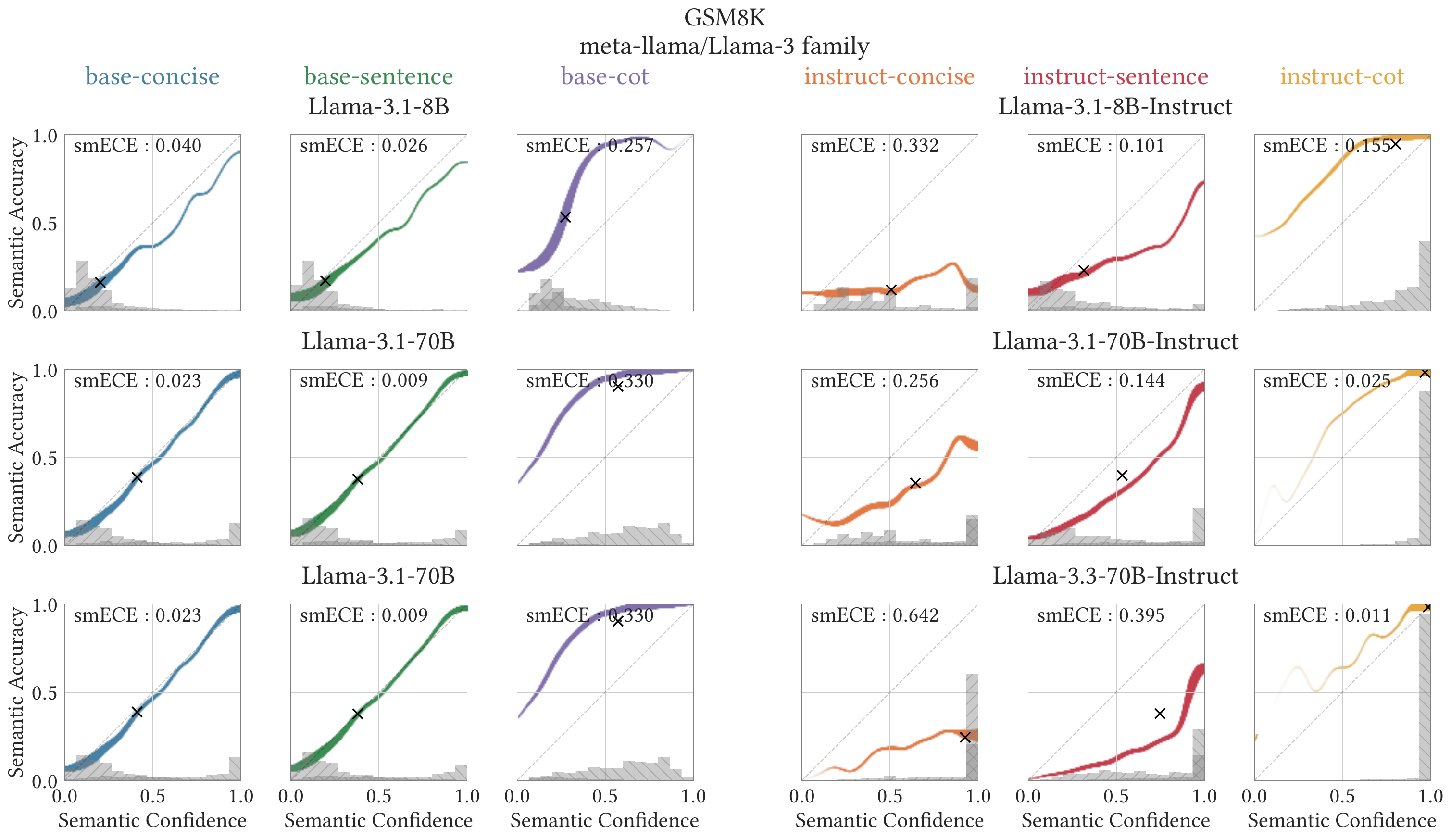}
\end{figure}
\begin{figure}[!htb]
\centering
\includegraphics[trim={0 10.4cm 0 0},clip,width=\linewidth]{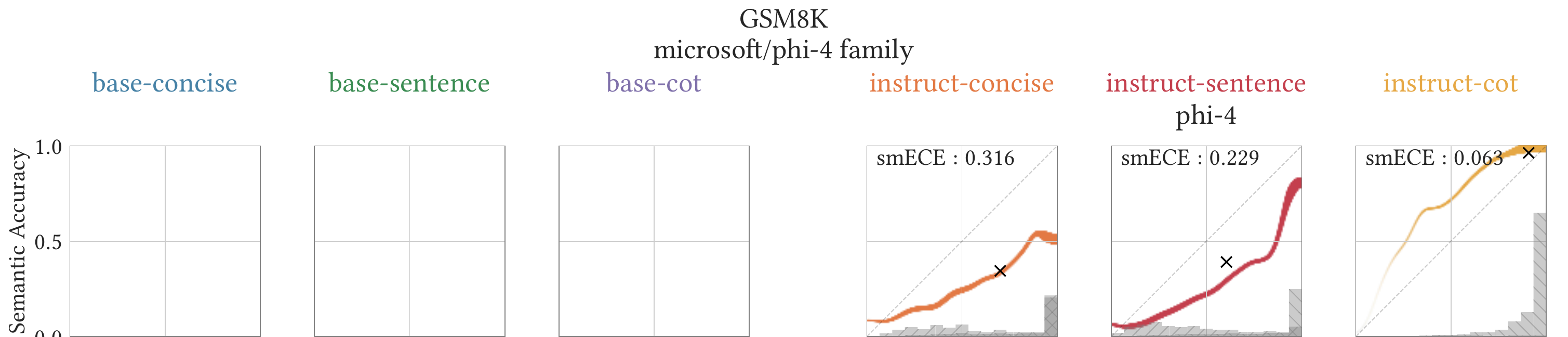}
\includegraphics[trim={0 0 0 19.5cm},clip,width=\linewidth]{figures/fig_appenidces_0_9.pdf}
\end{figure}

\FloatBarrier
\subsection{OpenMathInstruct}
\label{subsec:OpenMathInstruct}

\begin{figure}[!htb]
\centering
\includegraphics[width=\linewidth]{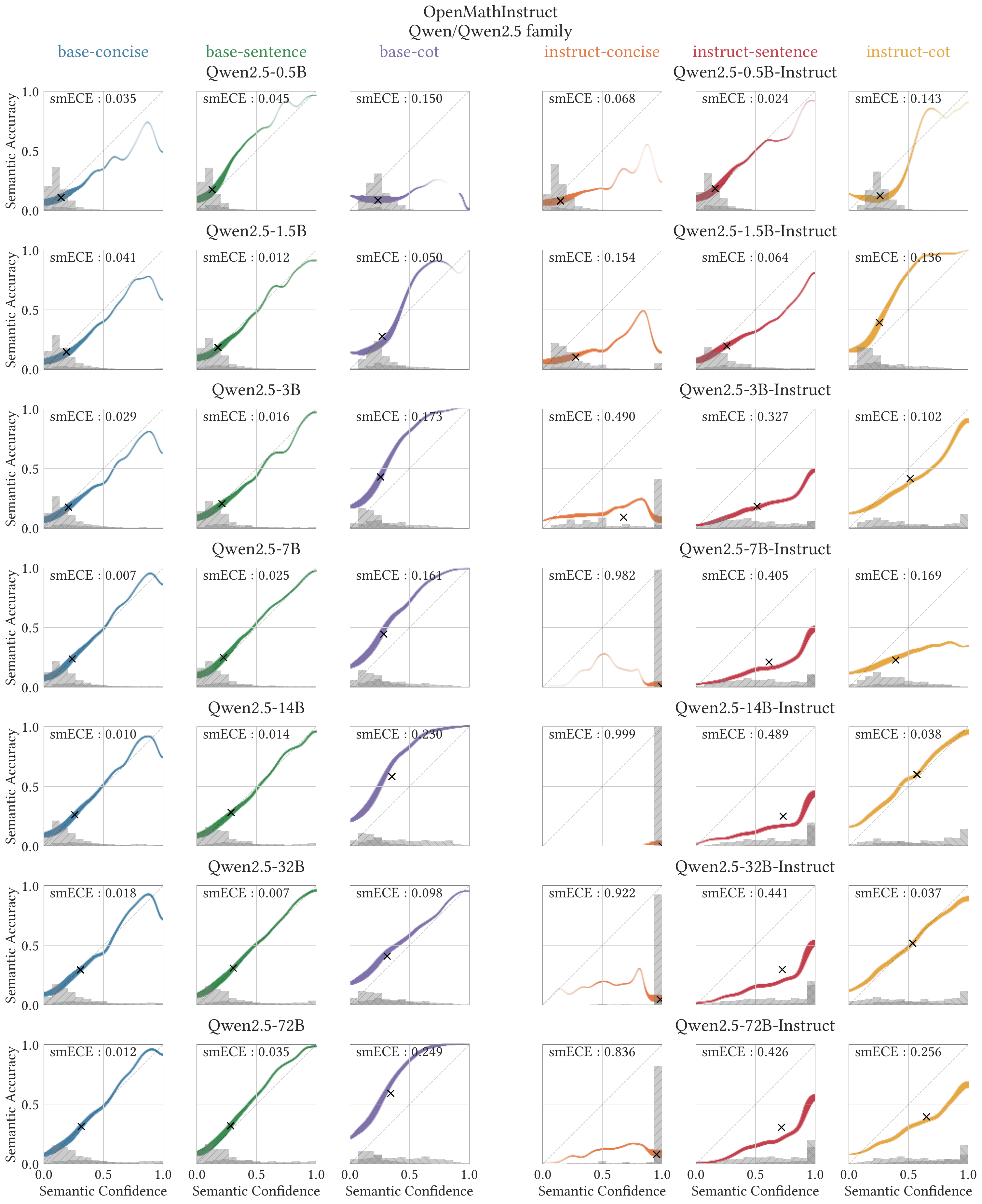}
\end{figure}
\begin{figure}[!htb]
\centering
\includegraphics[width=\linewidth]{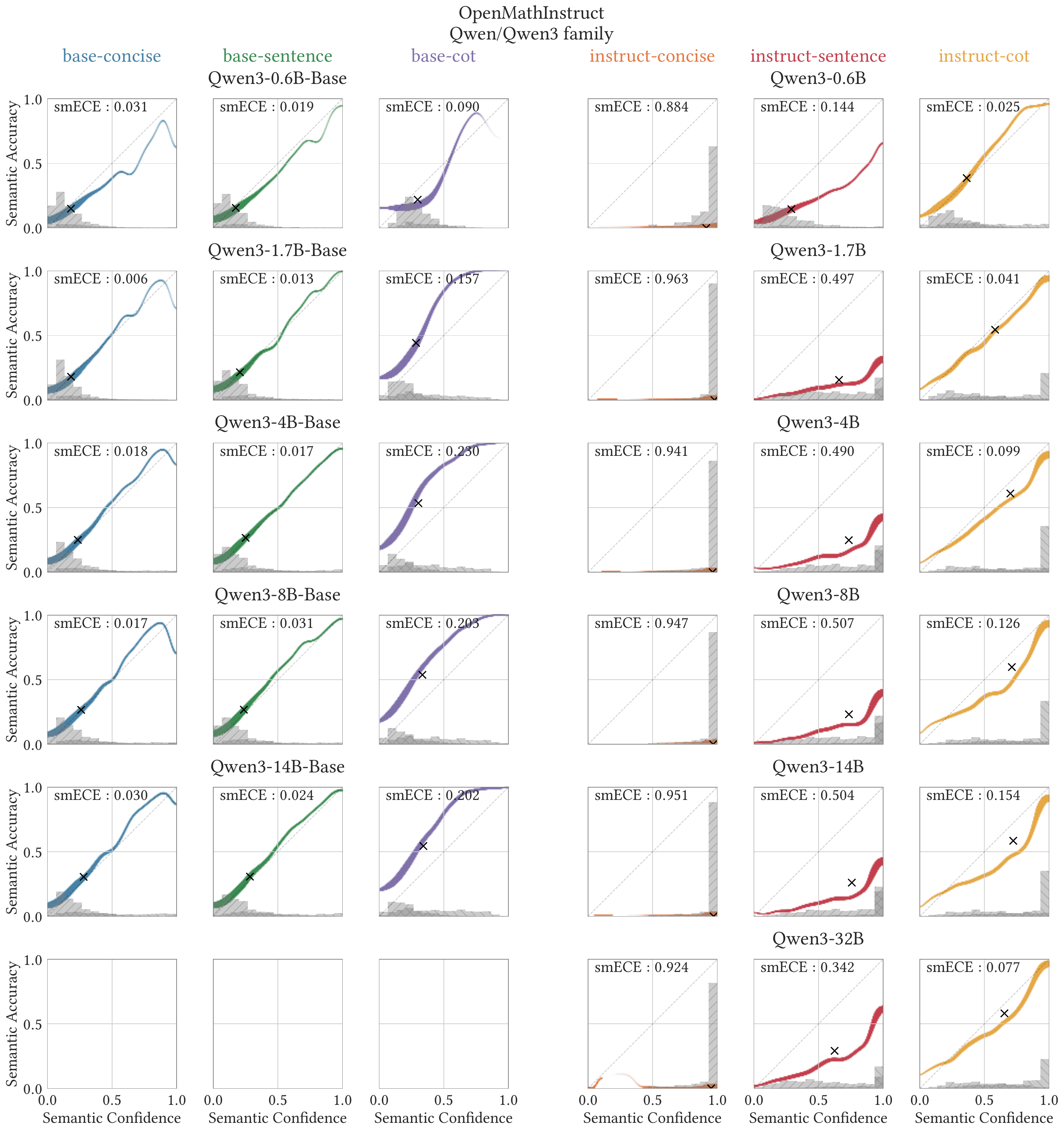}
\end{figure}
\begin{figure}[!htb]
\centering
\includegraphics[width=\linewidth]{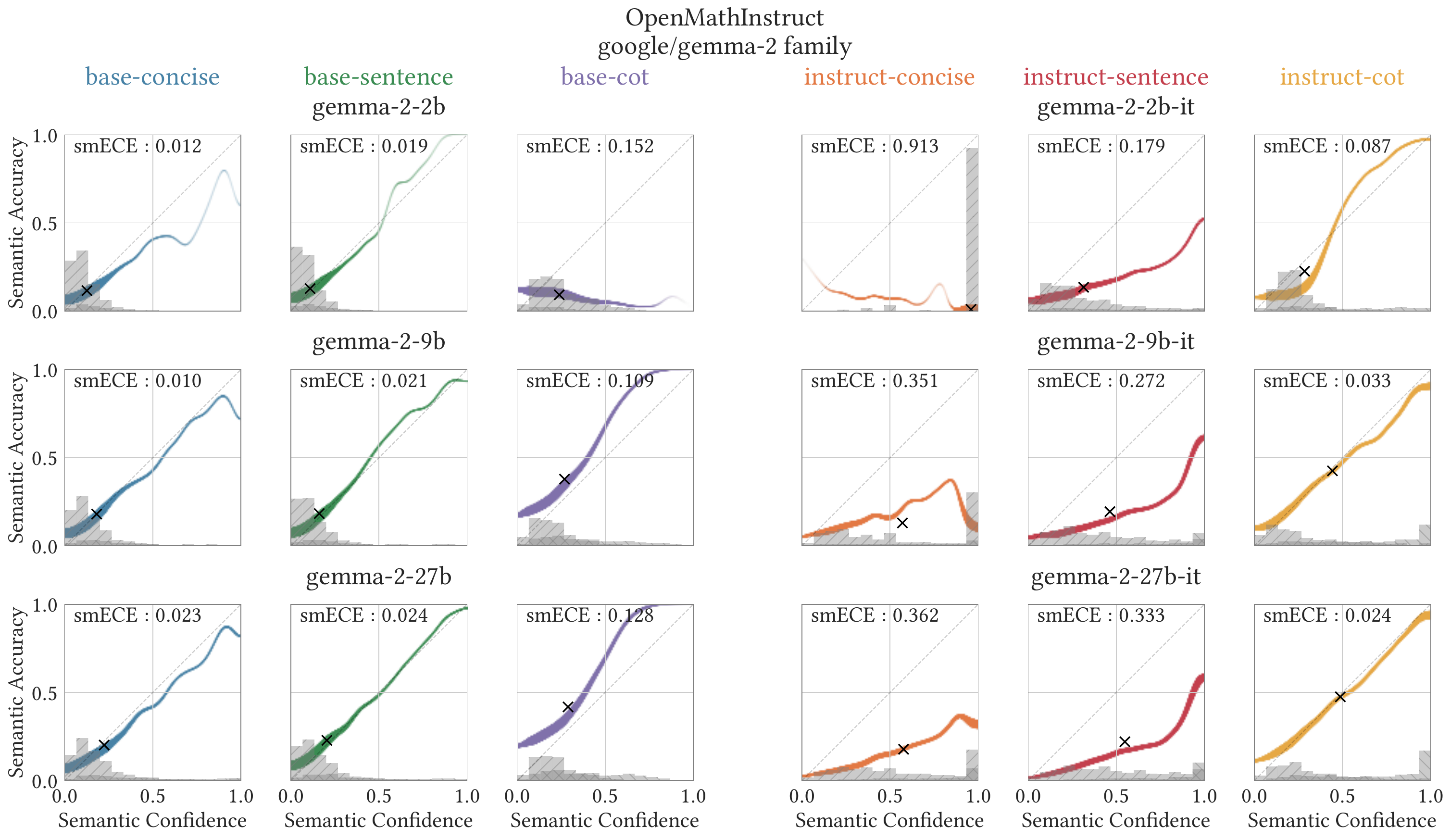}
\end{figure}
\begin{figure}[!htb]
\centering
\includegraphics[width=\linewidth]{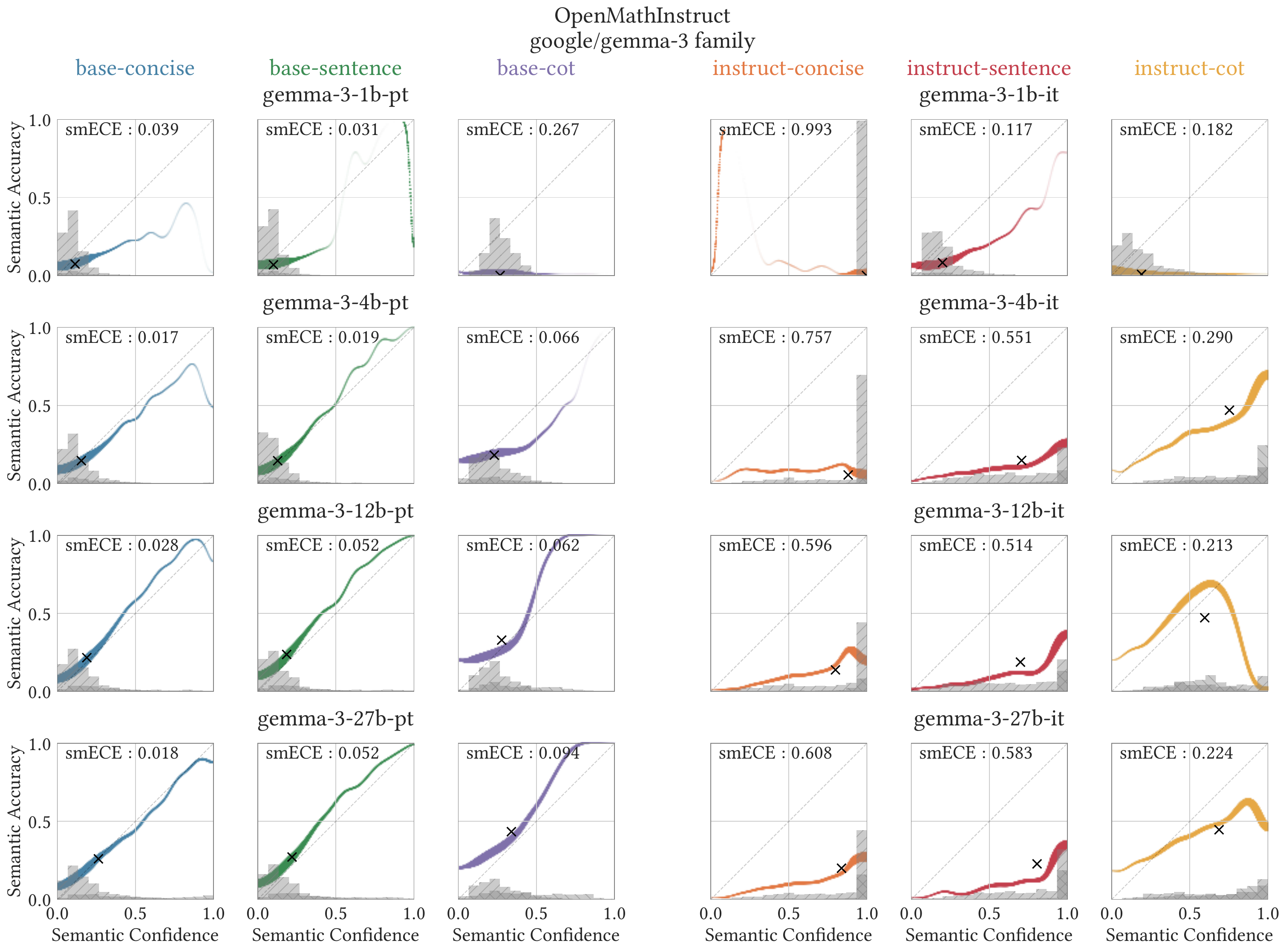}
\end{figure}
\begin{figure}[!htb]
\centering
\includegraphics[width=\linewidth]{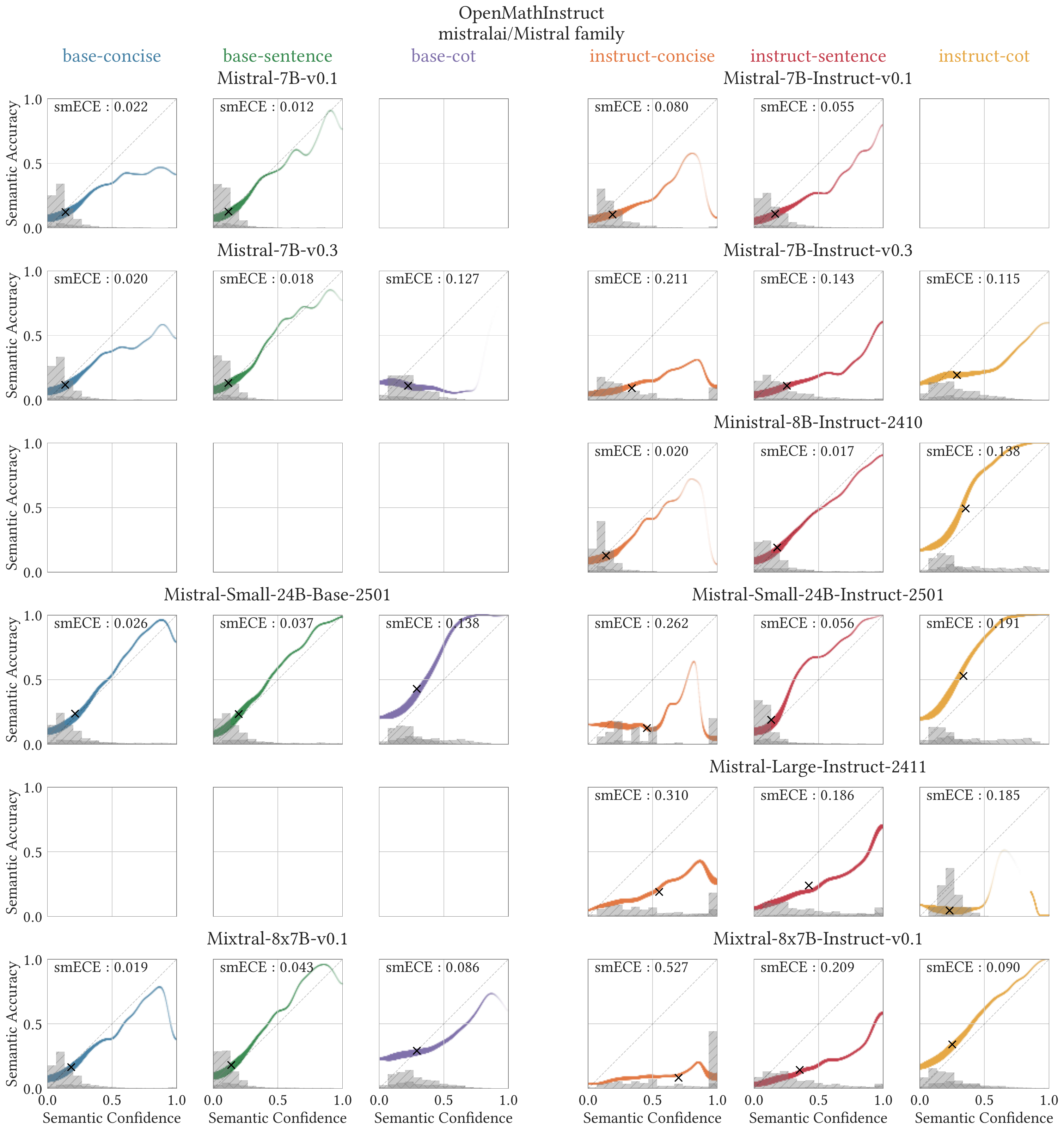}
\end{figure}
\begin{figure}[!htb]
\centering
\includegraphics[width=\linewidth]{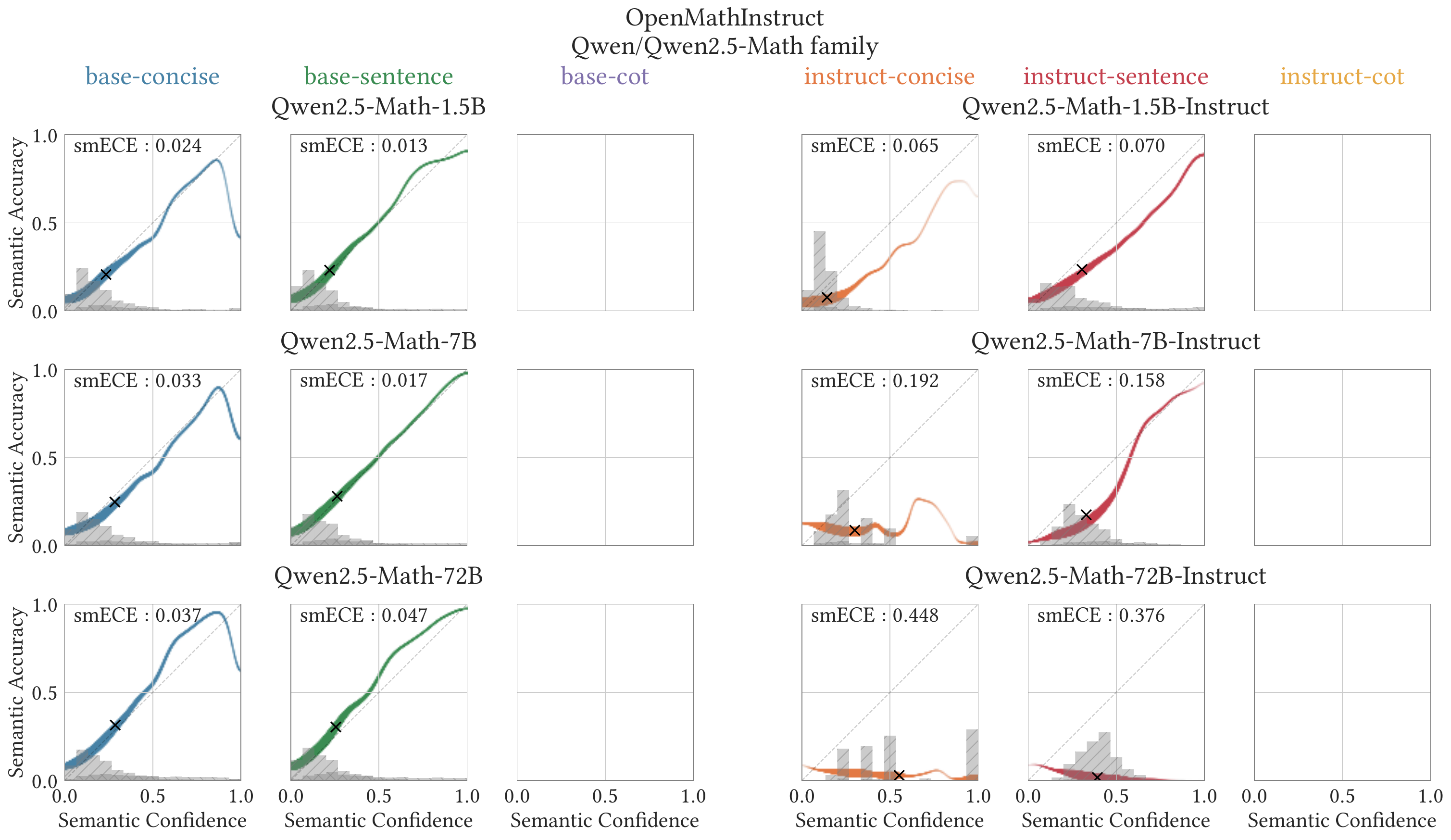}
\end{figure}
\begin{figure}[!htb]
\centering
\includegraphics[width=\linewidth]{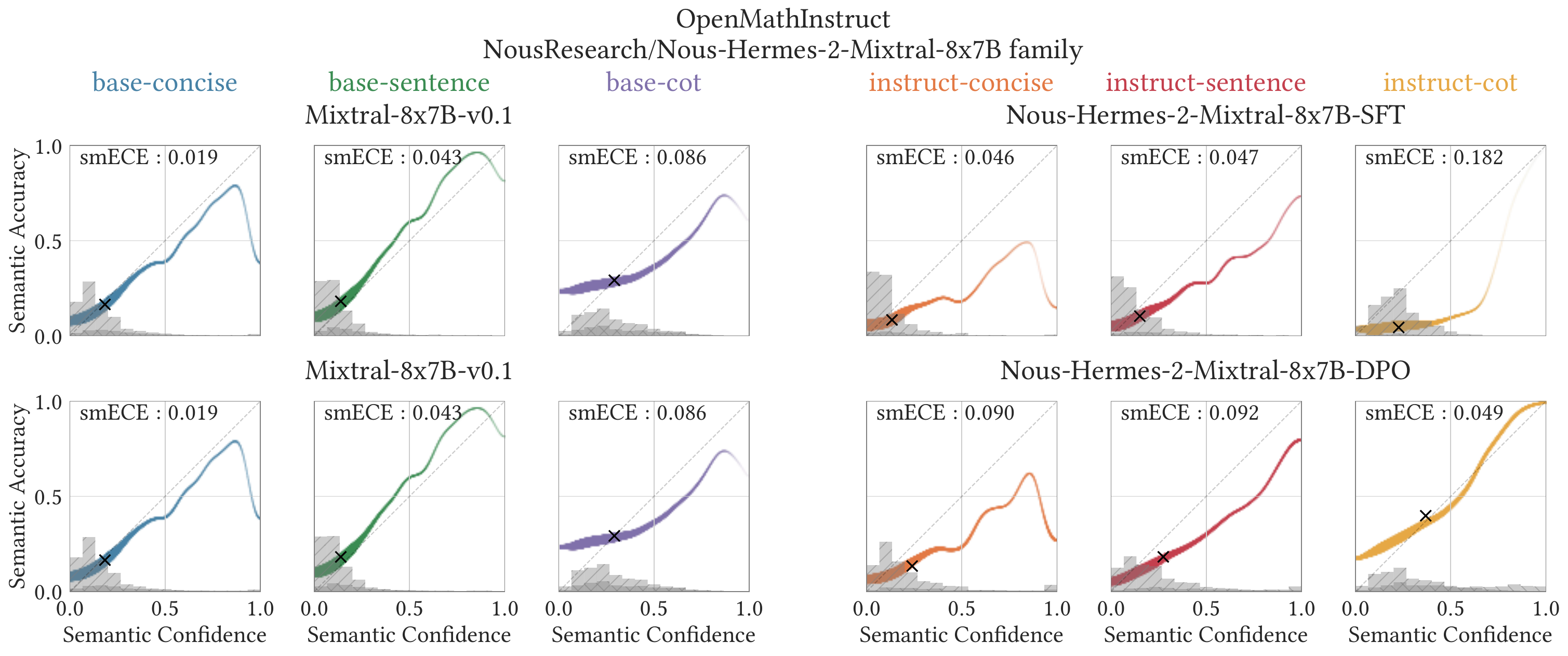}
\end{figure}
\begin{figure}[!htb]
\centering
\includegraphics[width=\linewidth]{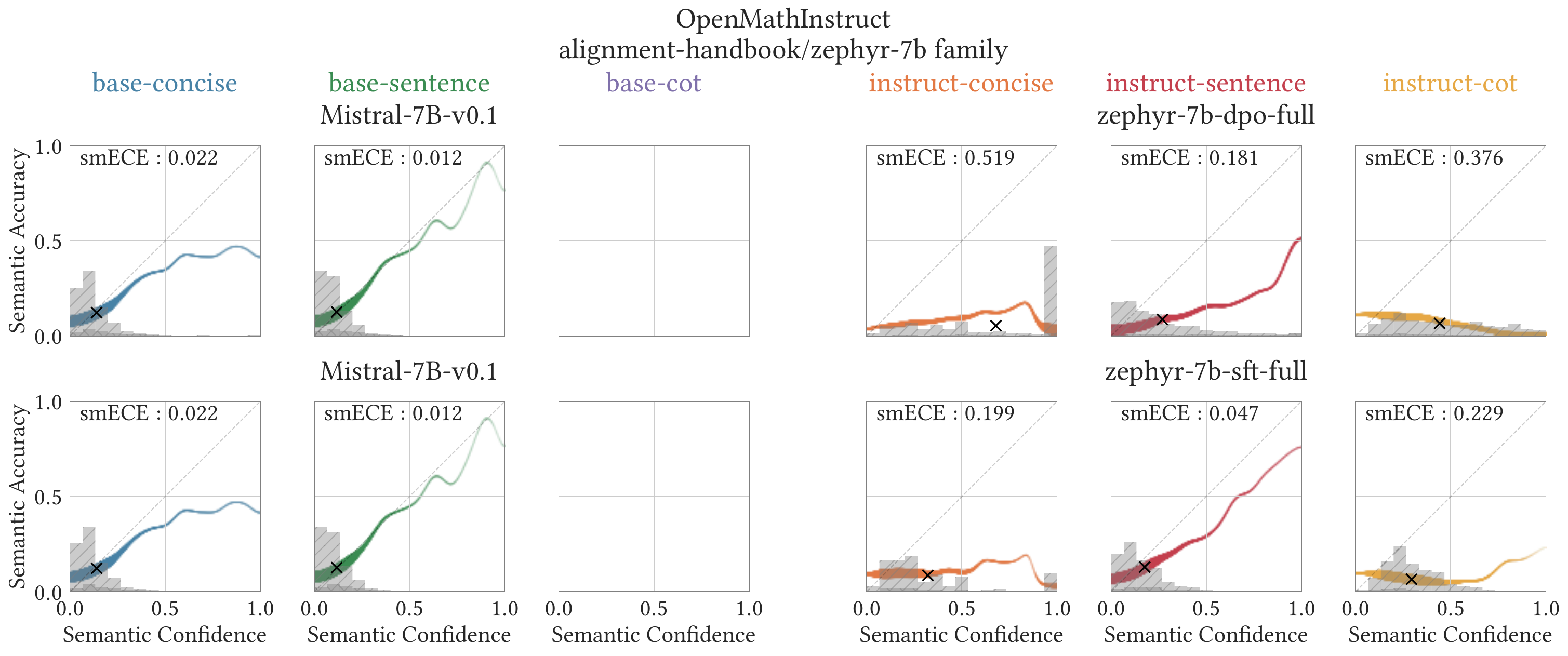}
\end{figure}
\begin{figure}[!htb]
\centering
\includegraphics[width=\linewidth]{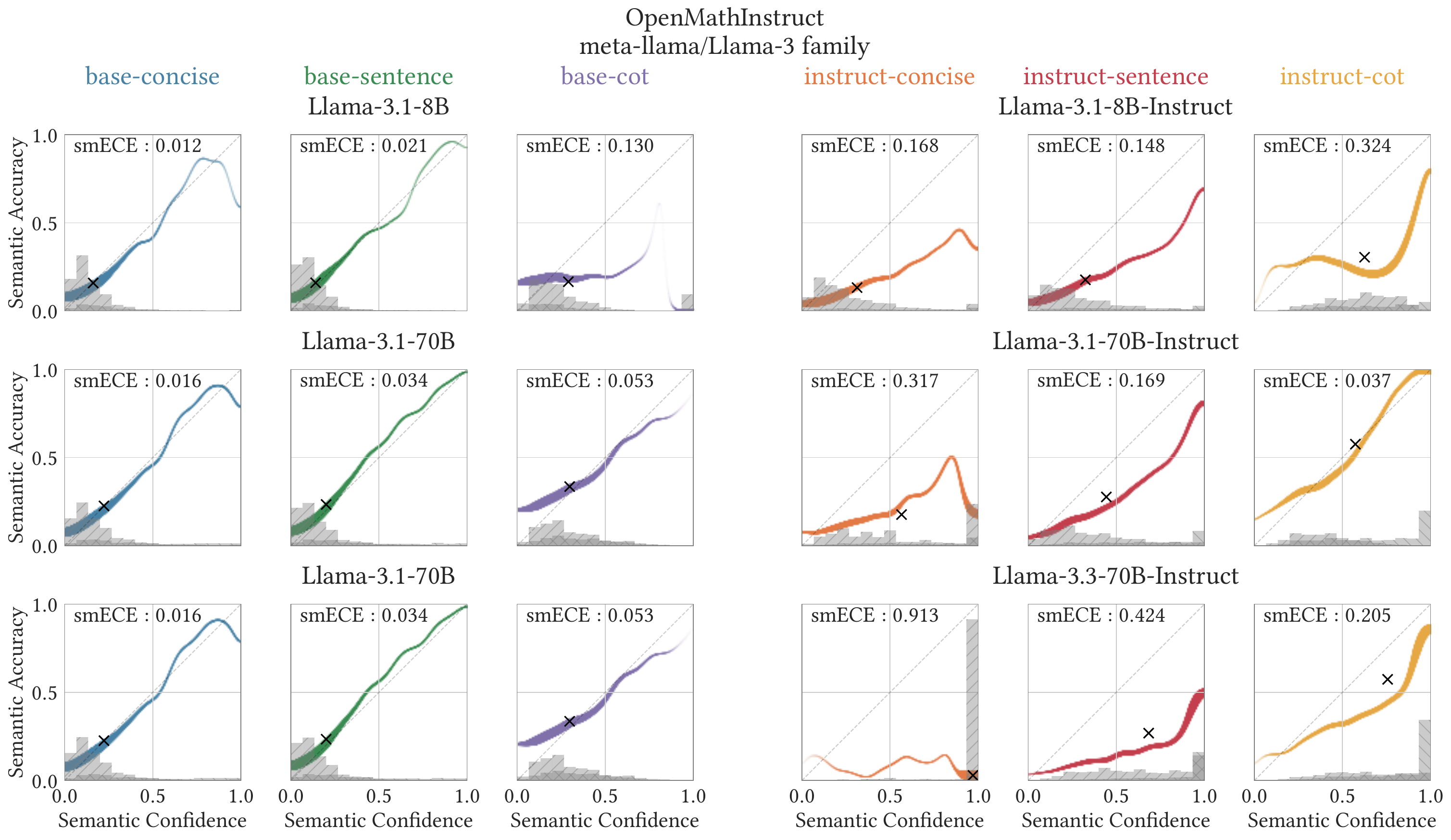}
\end{figure}
\begin{figure}[!htb]
\centering
\includegraphics[trim={0 10.4cm 0 0},clip,width=\linewidth]{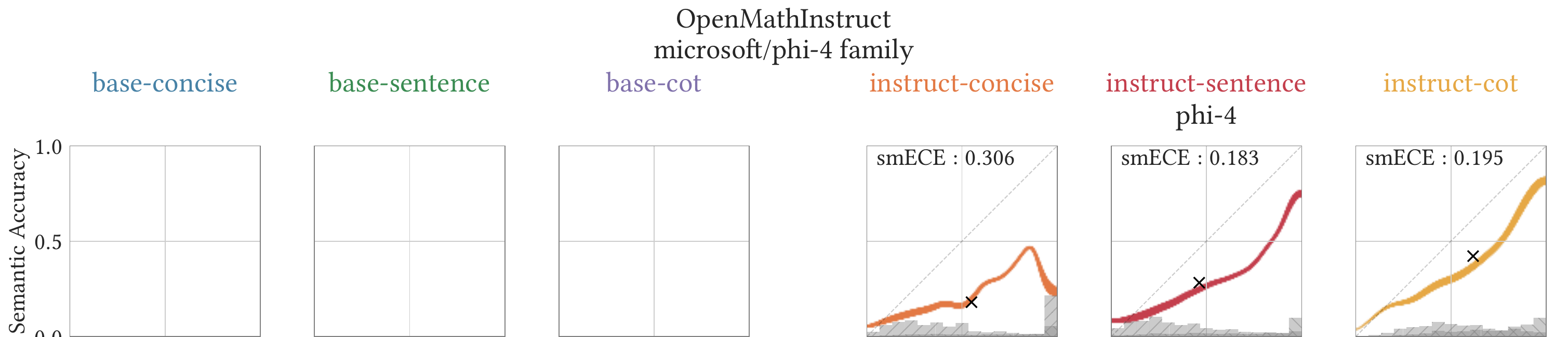}
\includegraphics[trim={0 0 0 19.5cm},clip,width=\linewidth]{figures/fig_appenidces_1_9.pdf}
\end{figure}

\FloatBarrier
\subsection{TriviaQA}
\label{subsection:TriviaQA}

\begin{figure}[!htb]
\centering
\includegraphics[width=\linewidth]{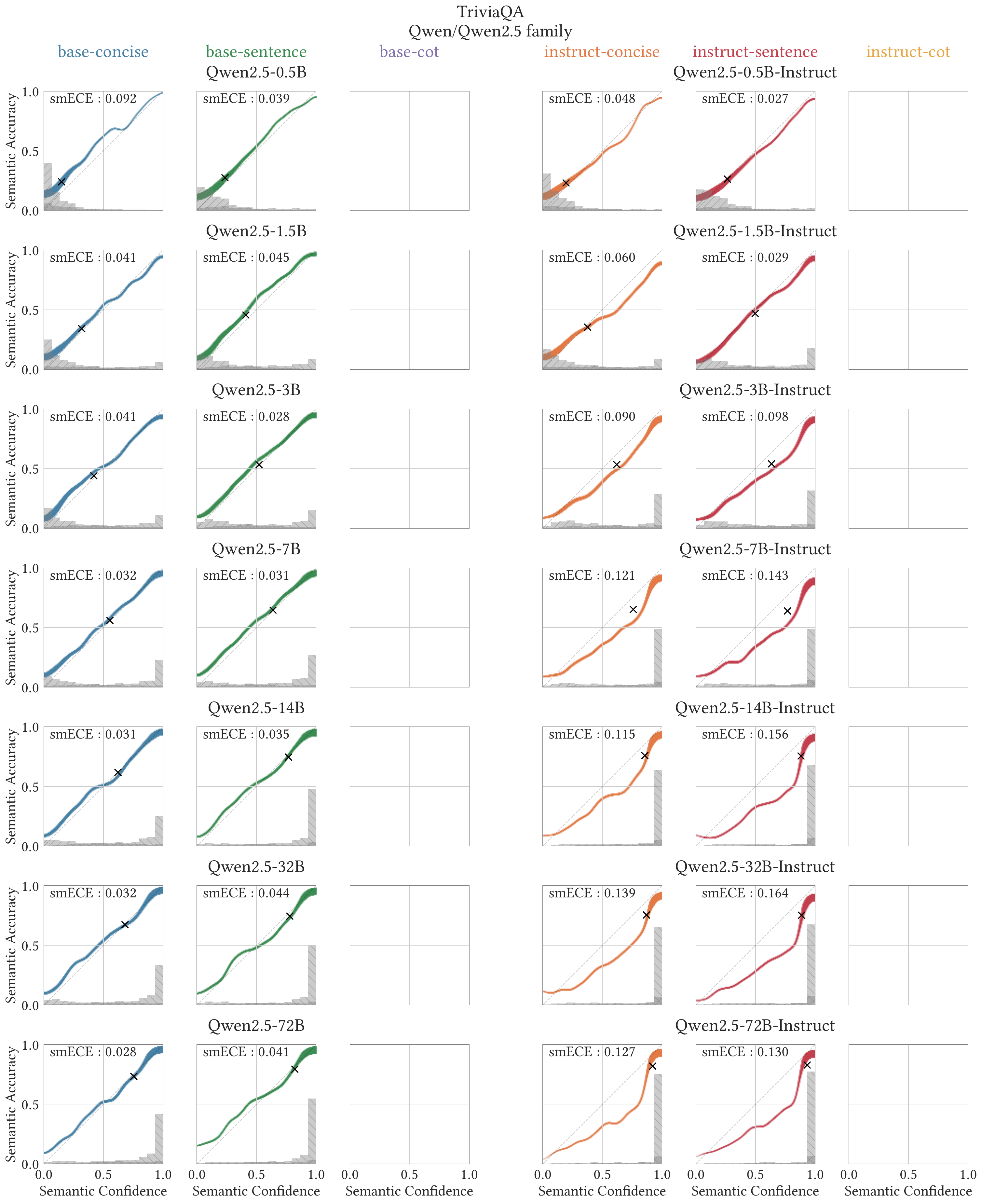}
\end{figure}
\begin{figure}[!htb]
\centering
\includegraphics[width=\linewidth]{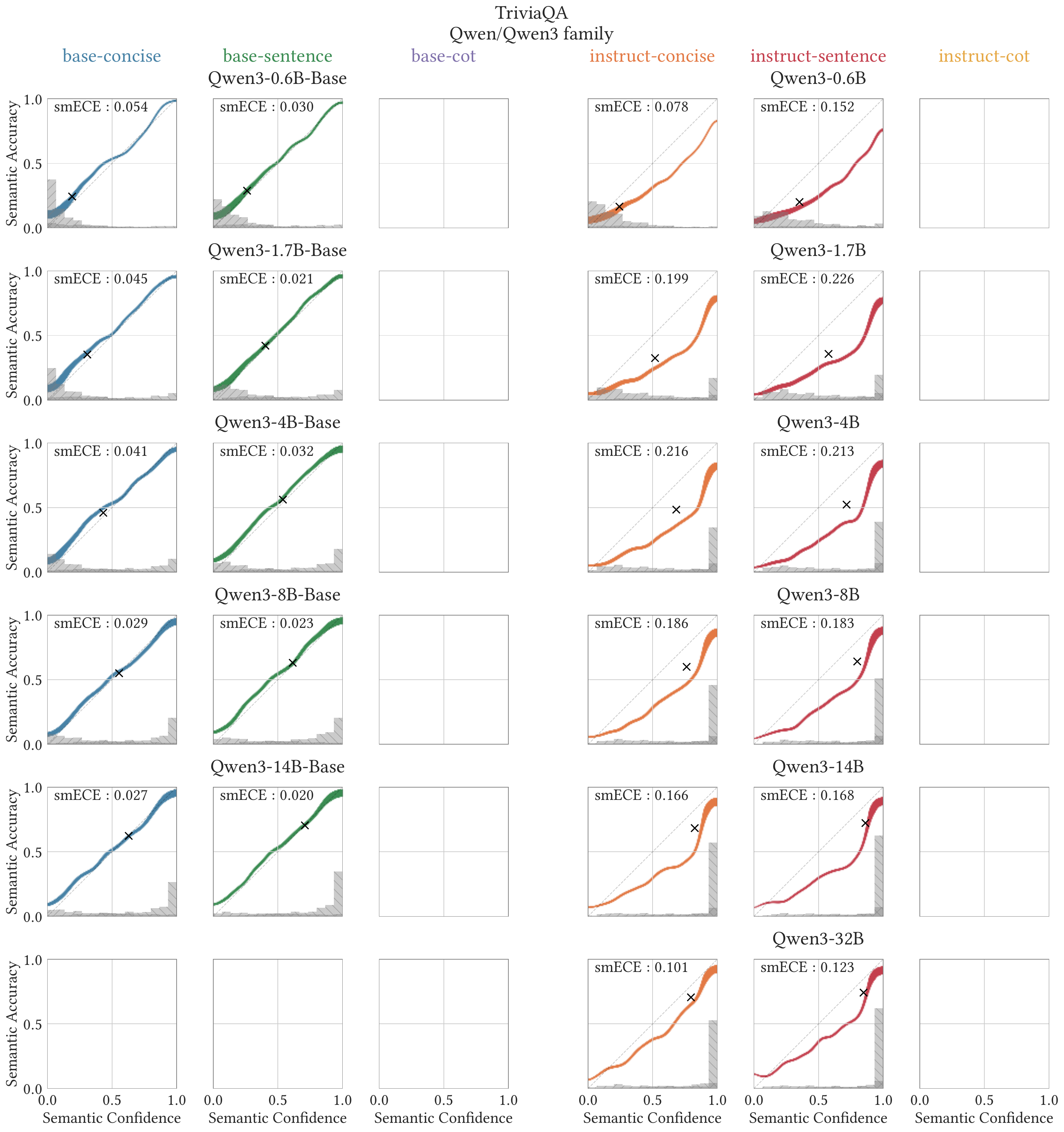}
\end{figure}
\begin{figure}[!htb]
\centering
\includegraphics[width=\linewidth]{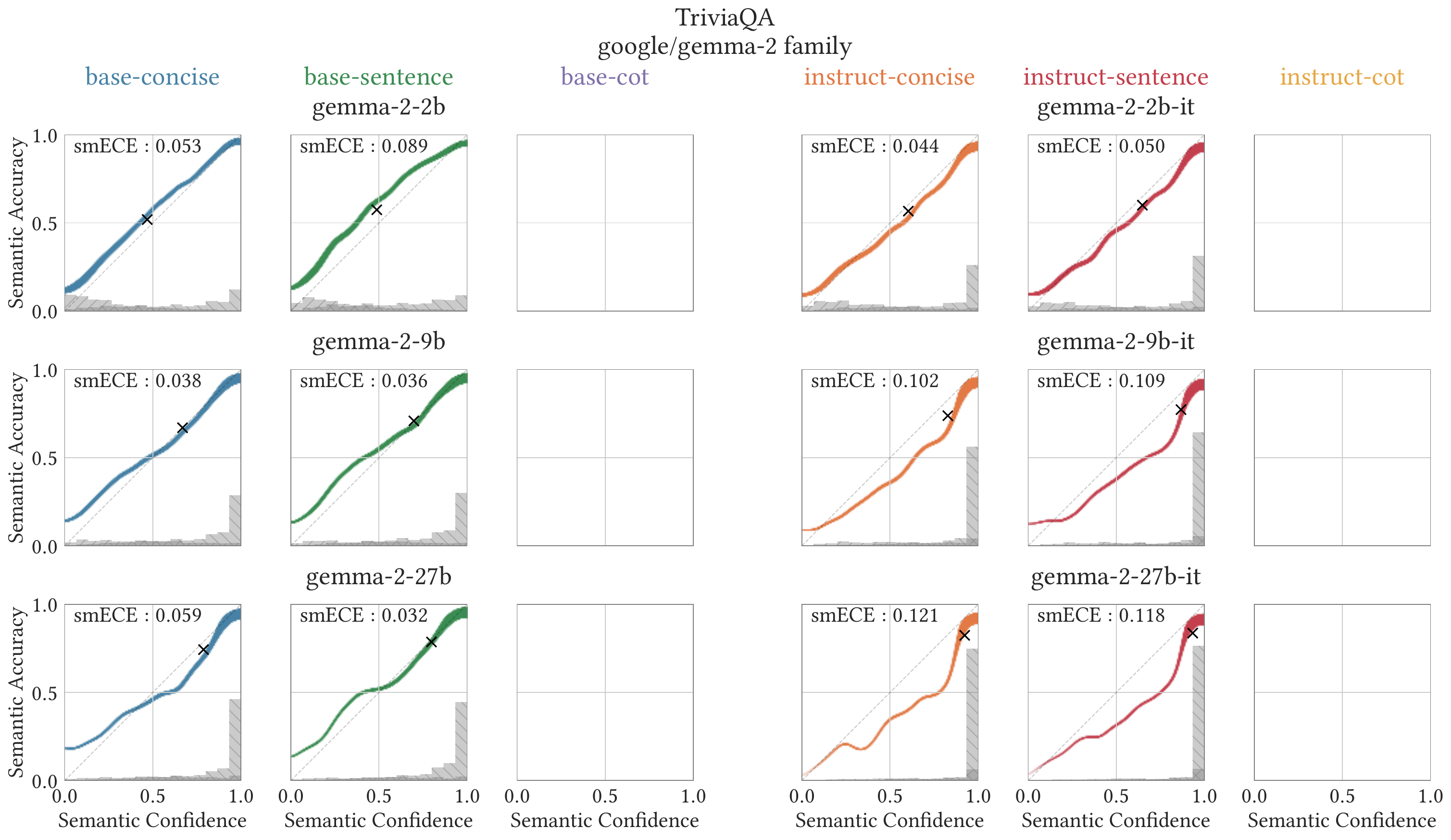}
\end{figure}
\begin{figure}[!htb]
\centering
\includegraphics[width=\linewidth]{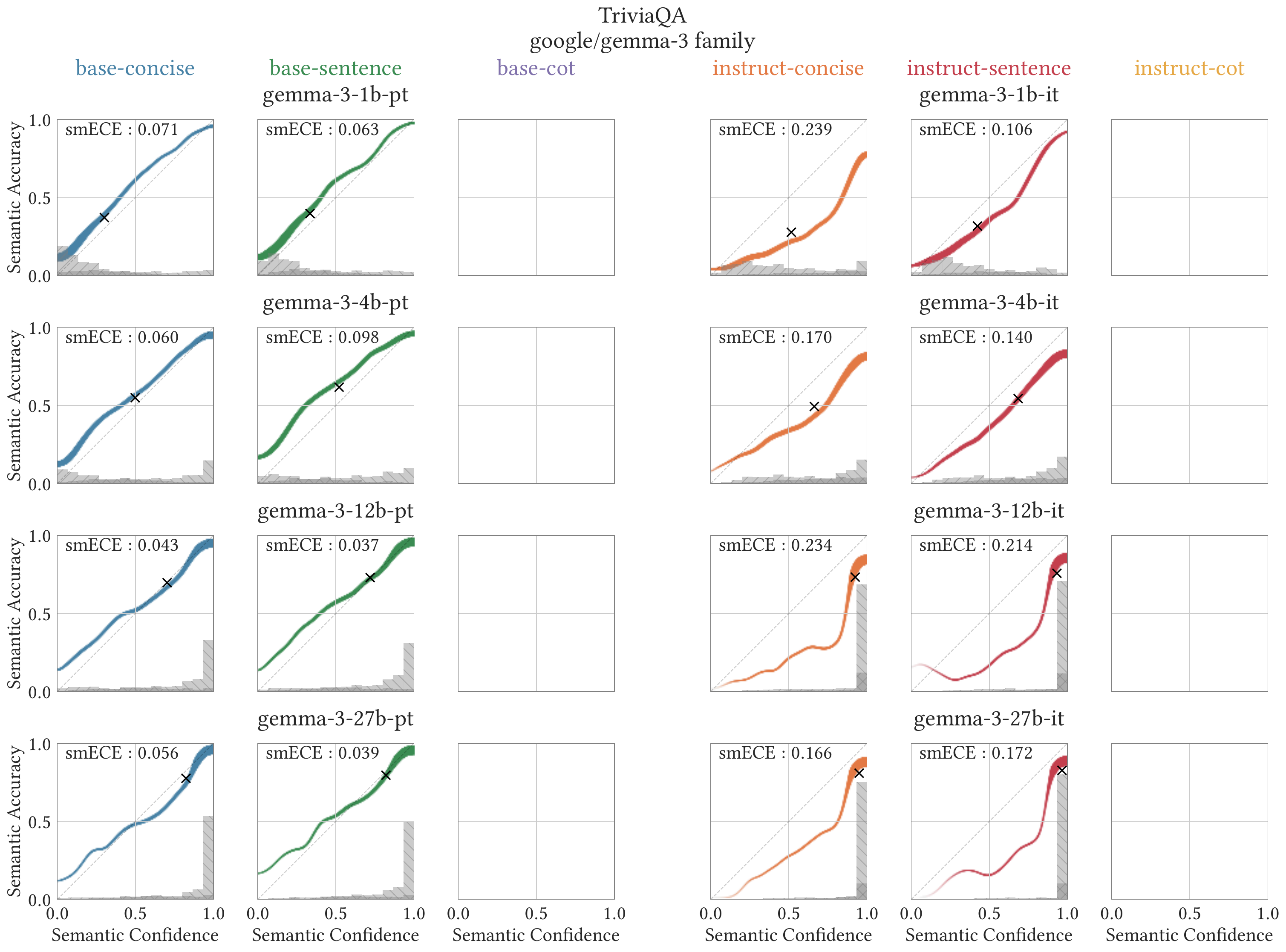}
\end{figure}
\begin{figure}[!htb]
\centering
\includegraphics[width=\linewidth]{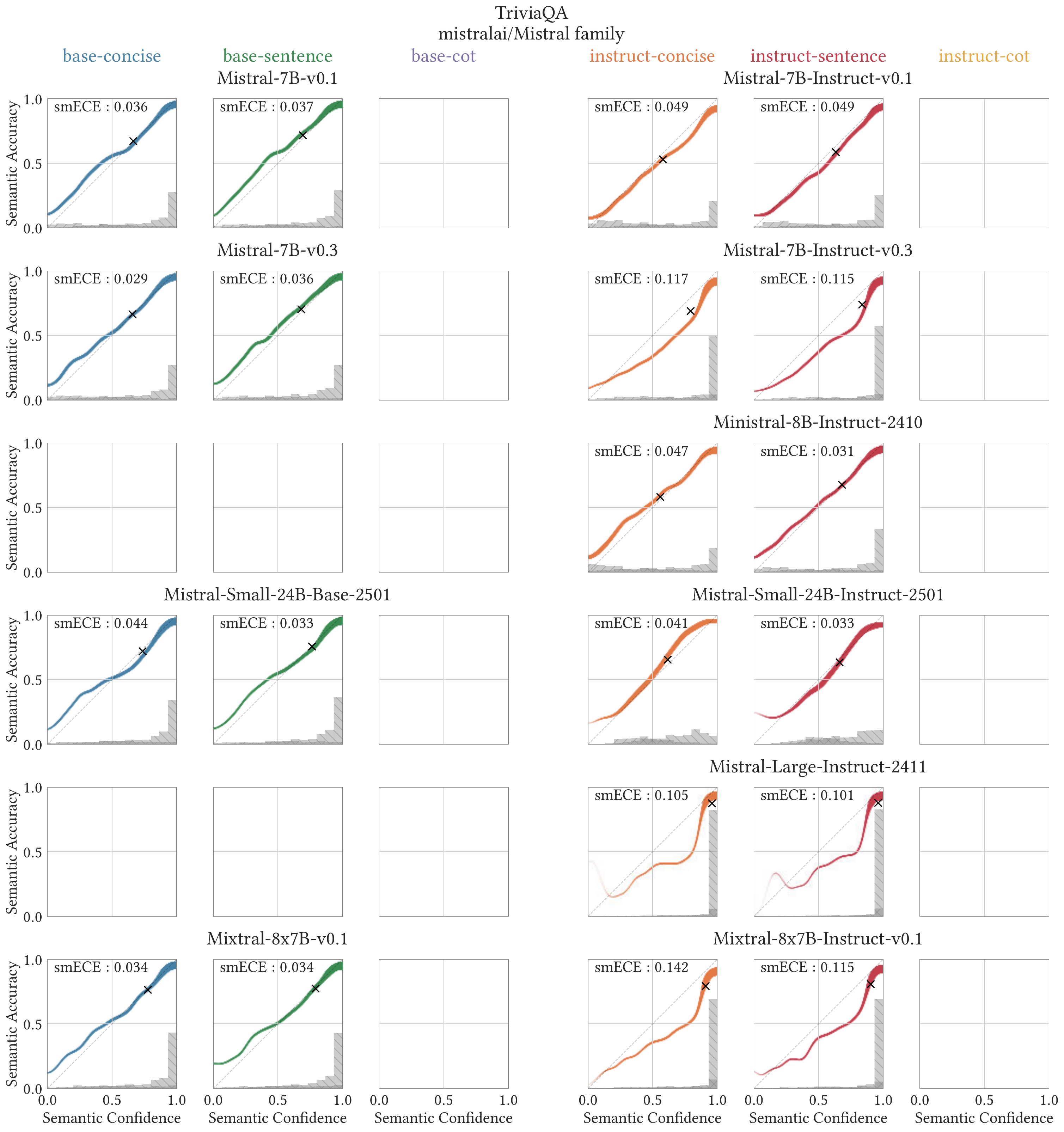}
\end{figure}
\begin{figure}[!htb]
\centering
\includegraphics[width=\linewidth]{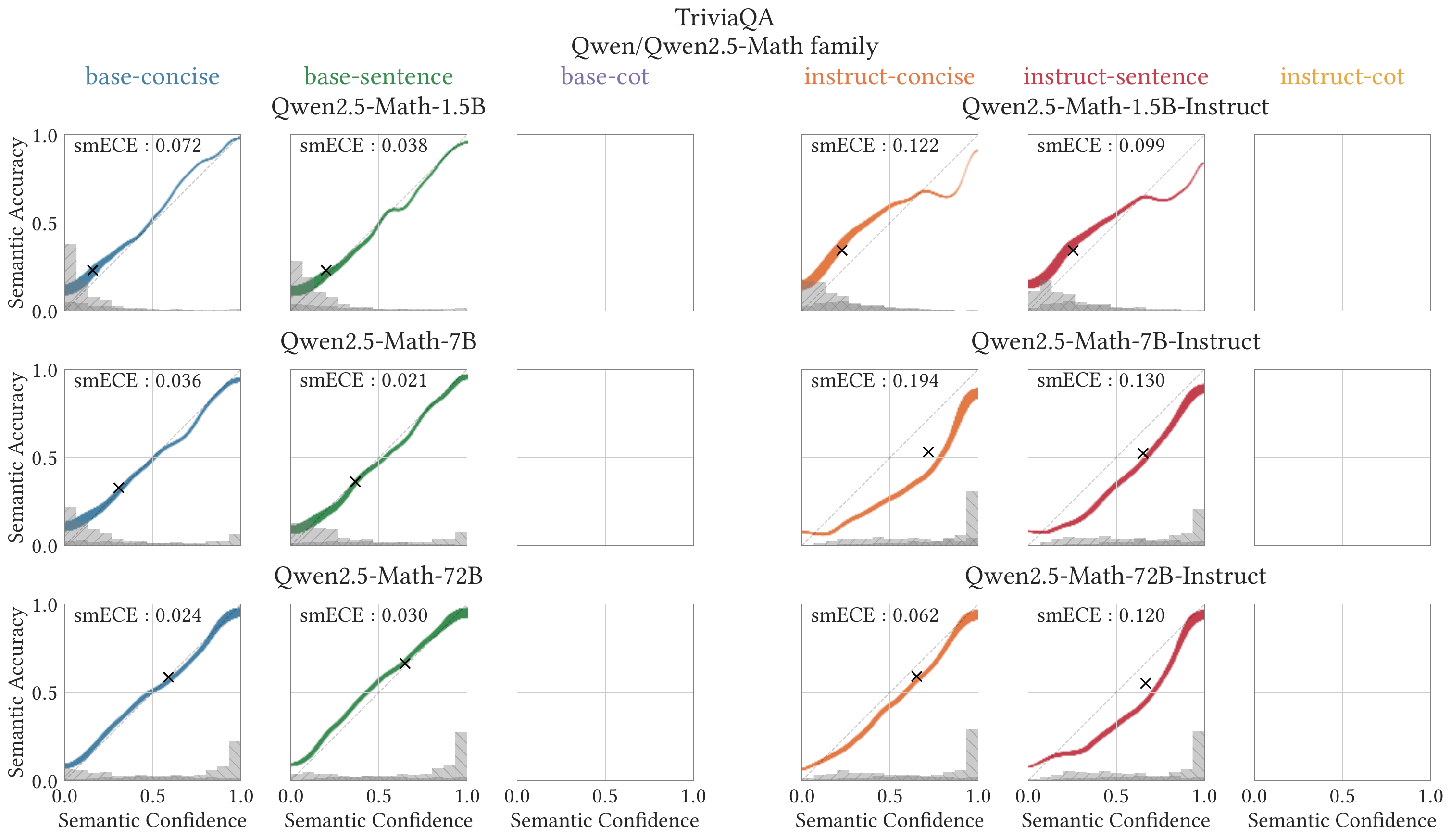}
\end{figure}
\begin{figure}[!htb]
\centering
\includegraphics[width=\linewidth]{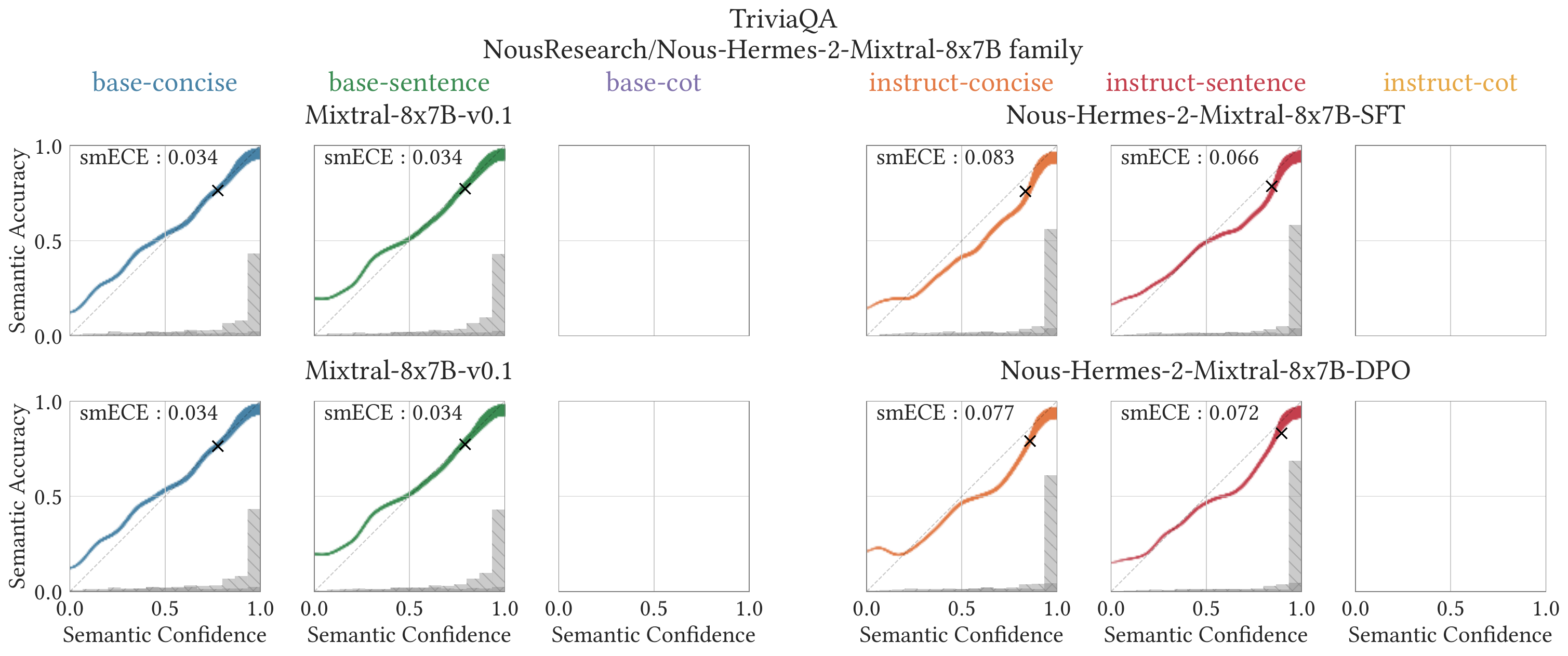}
\end{figure}
\begin{figure}[!htb]
\centering
\includegraphics[width=\linewidth]{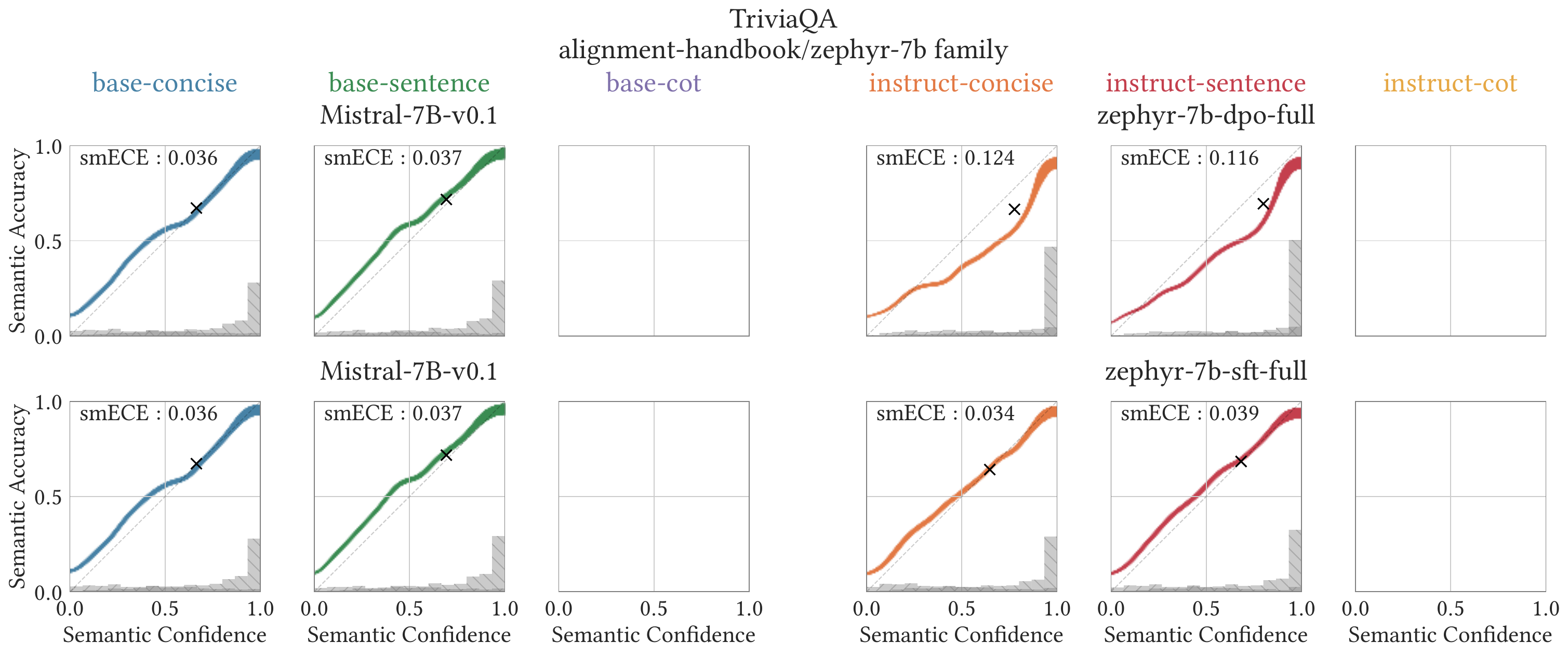}
\end{figure}
\begin{figure}[!htb]
\centering
\includegraphics[width=\linewidth]{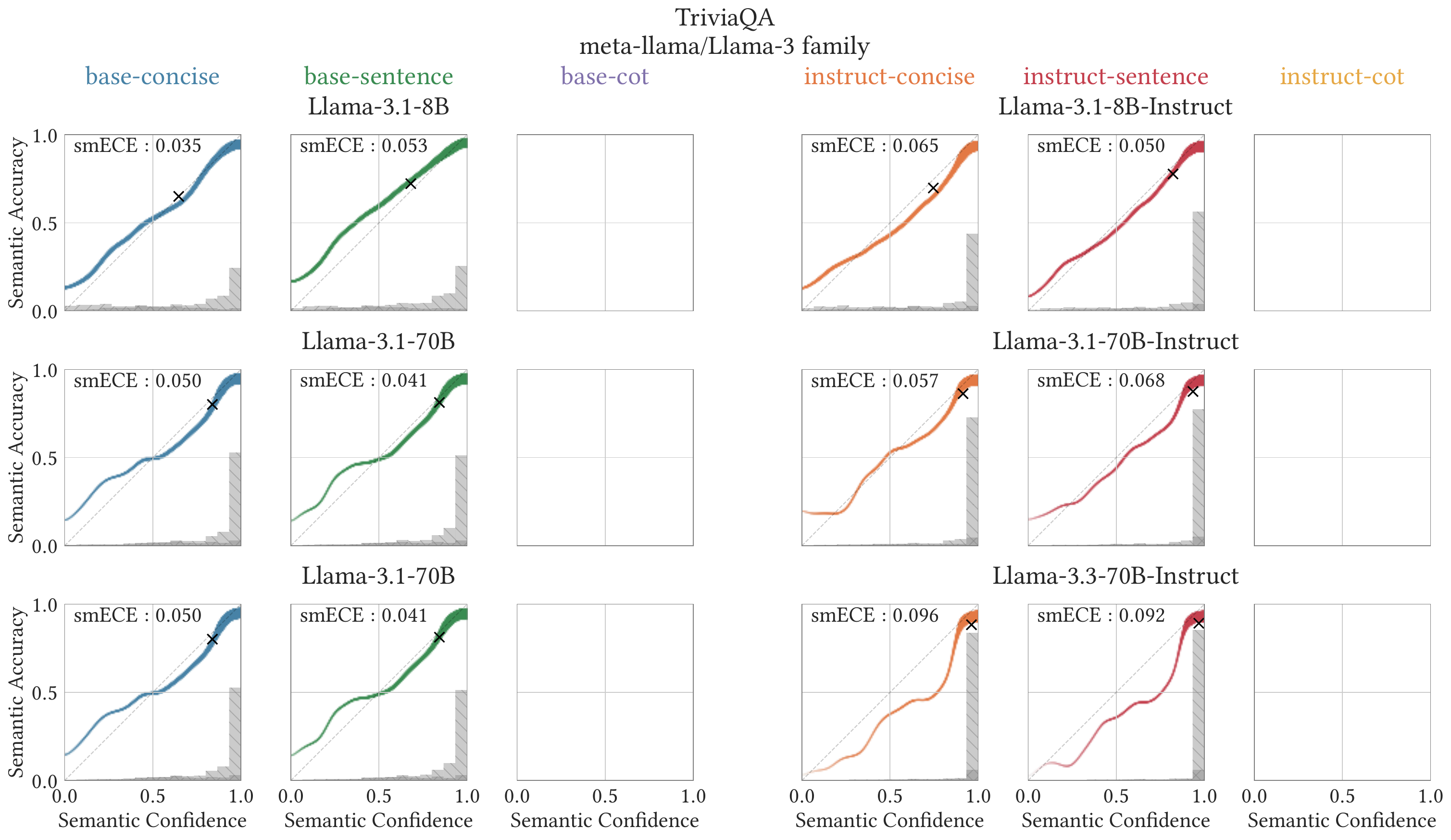}
\end{figure}
\begin{figure}[!htb]
\centering
\includegraphics[trim={0 10.5cm 0 0},clip,width=\linewidth]{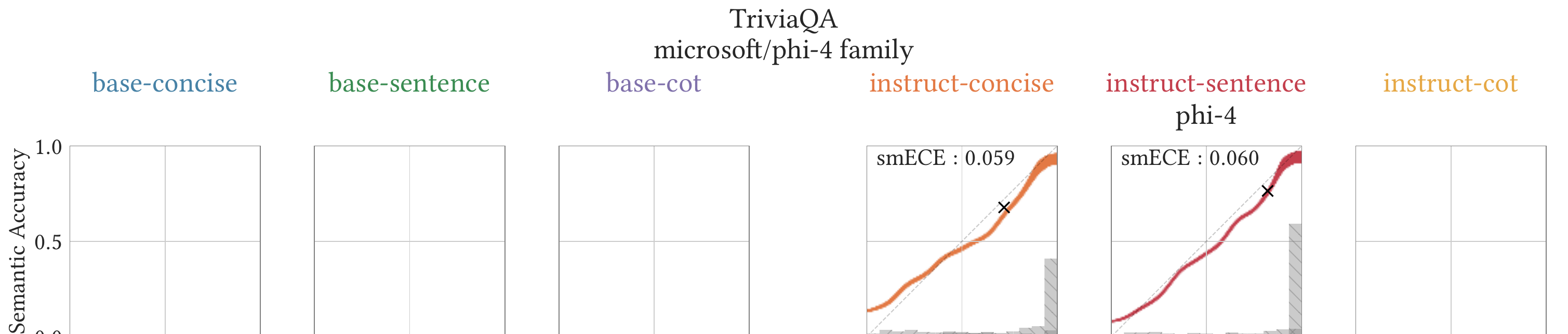}
\includegraphics[trim={0 0 0 20.2cm},clip,width=\linewidth]{figures/fig_appenidces_2_9.pdf}
\end{figure}

\FloatBarrier
\subsection{SimpleQA}
\label{subsection:SimpleQA}

\begin{figure}[!htb]
\centering
\includegraphics[width=\linewidth]{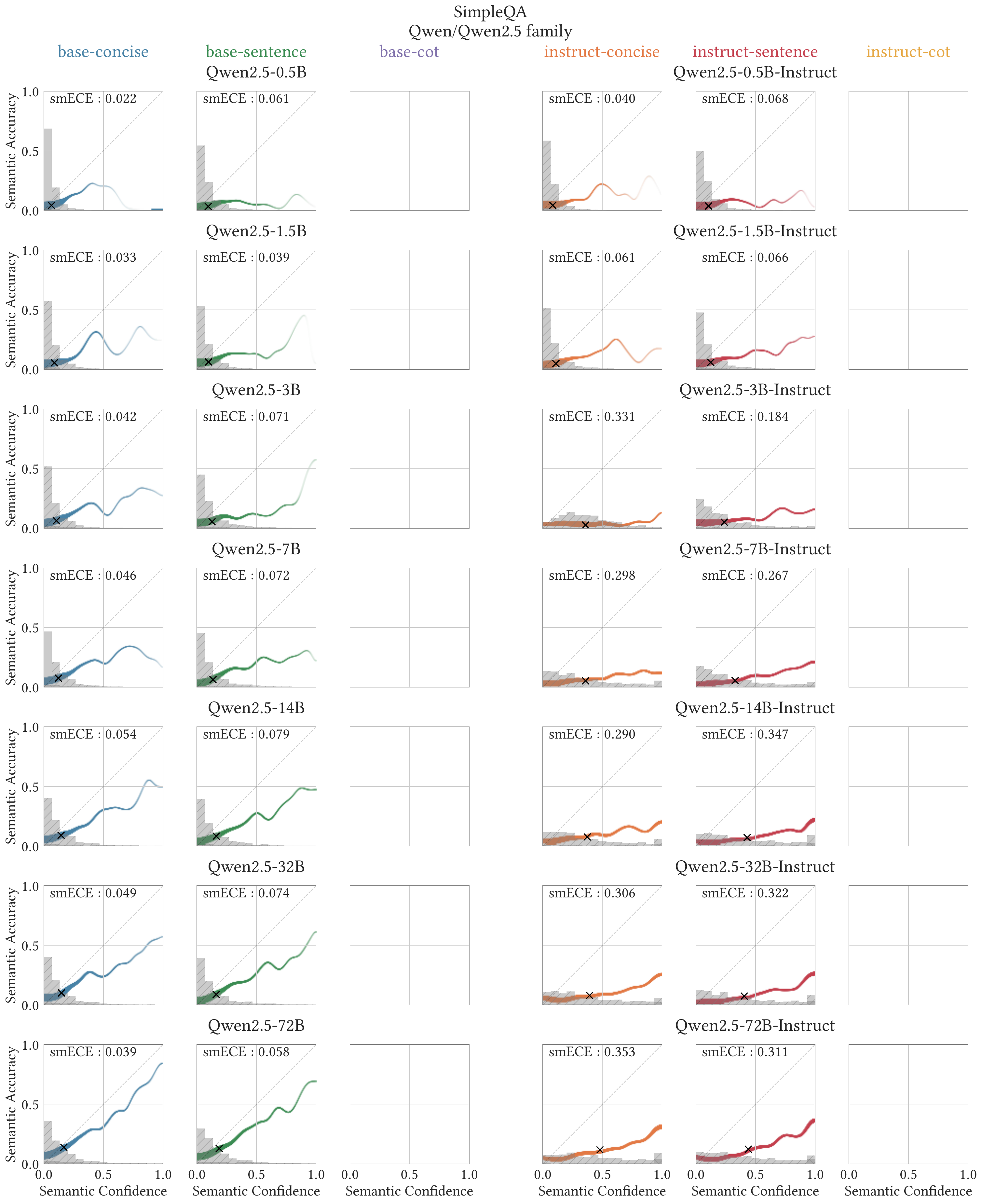}
\end{figure}
\begin{figure}[!htb]
\centering
\includegraphics[width=\linewidth]{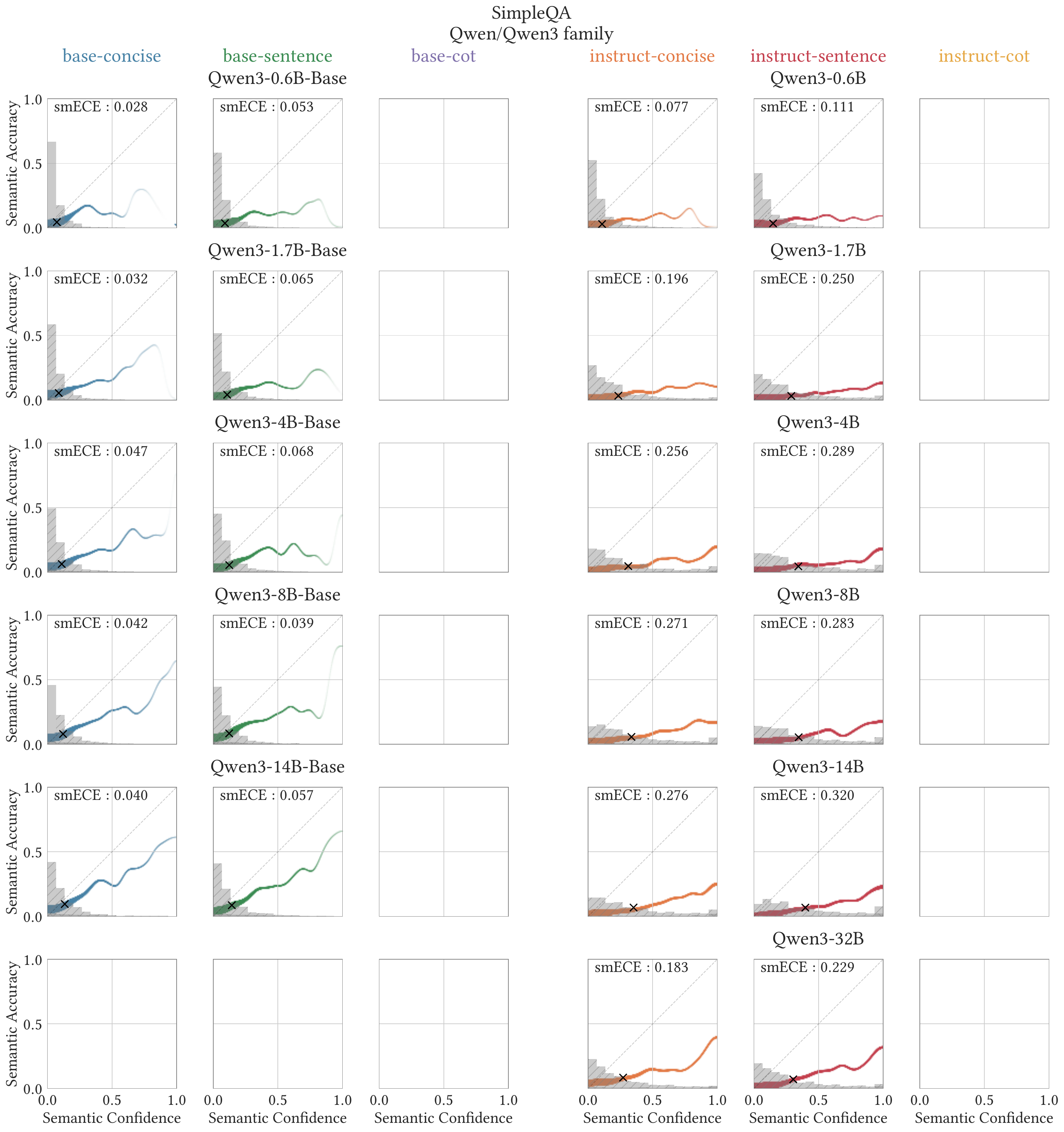}
\end{figure}
\begin{figure}[!htb]
\centering
\includegraphics[width=\linewidth]{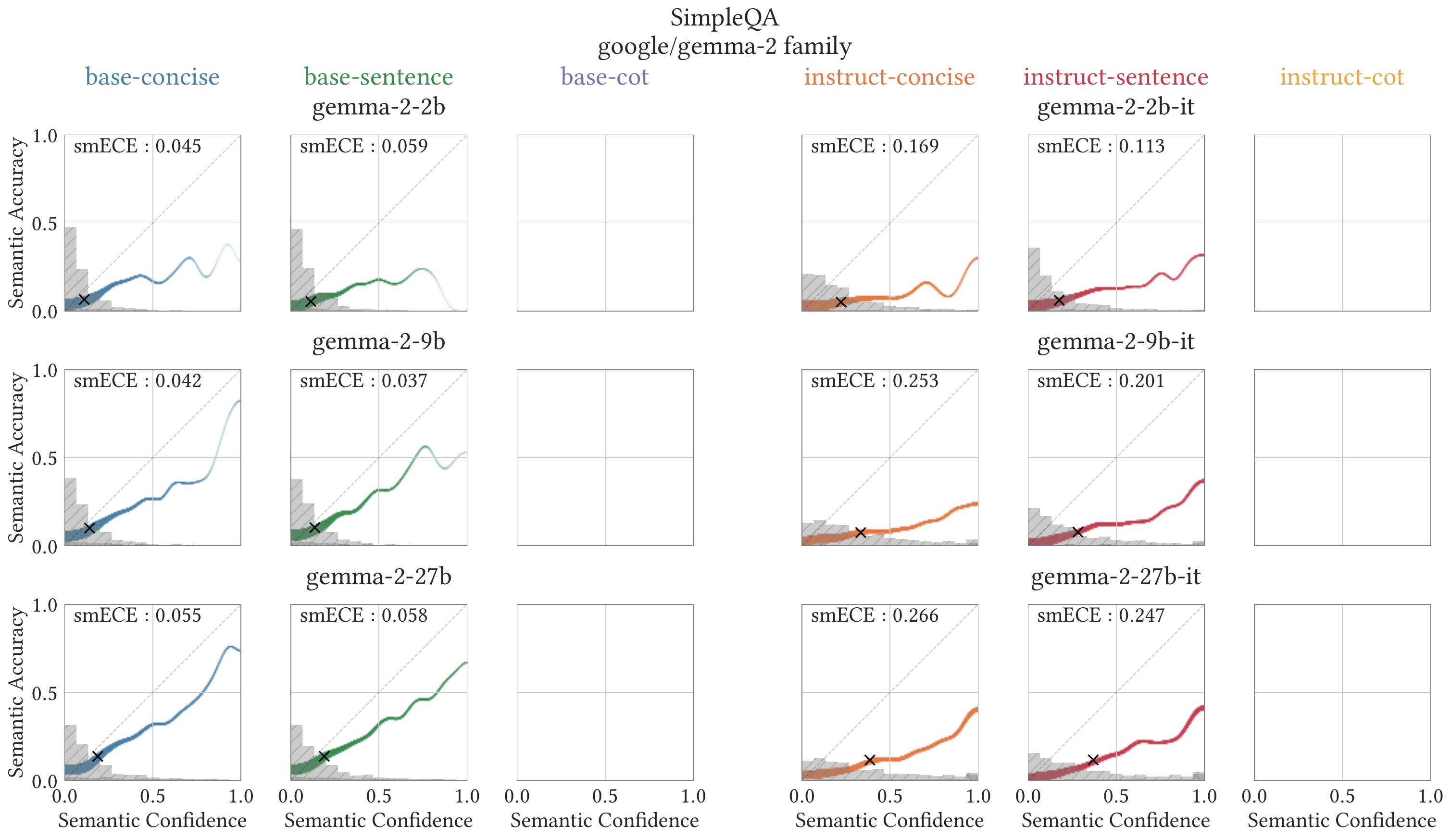}
\end{figure}
\begin{figure}[!htb]
\centering
\includegraphics[width=\linewidth]{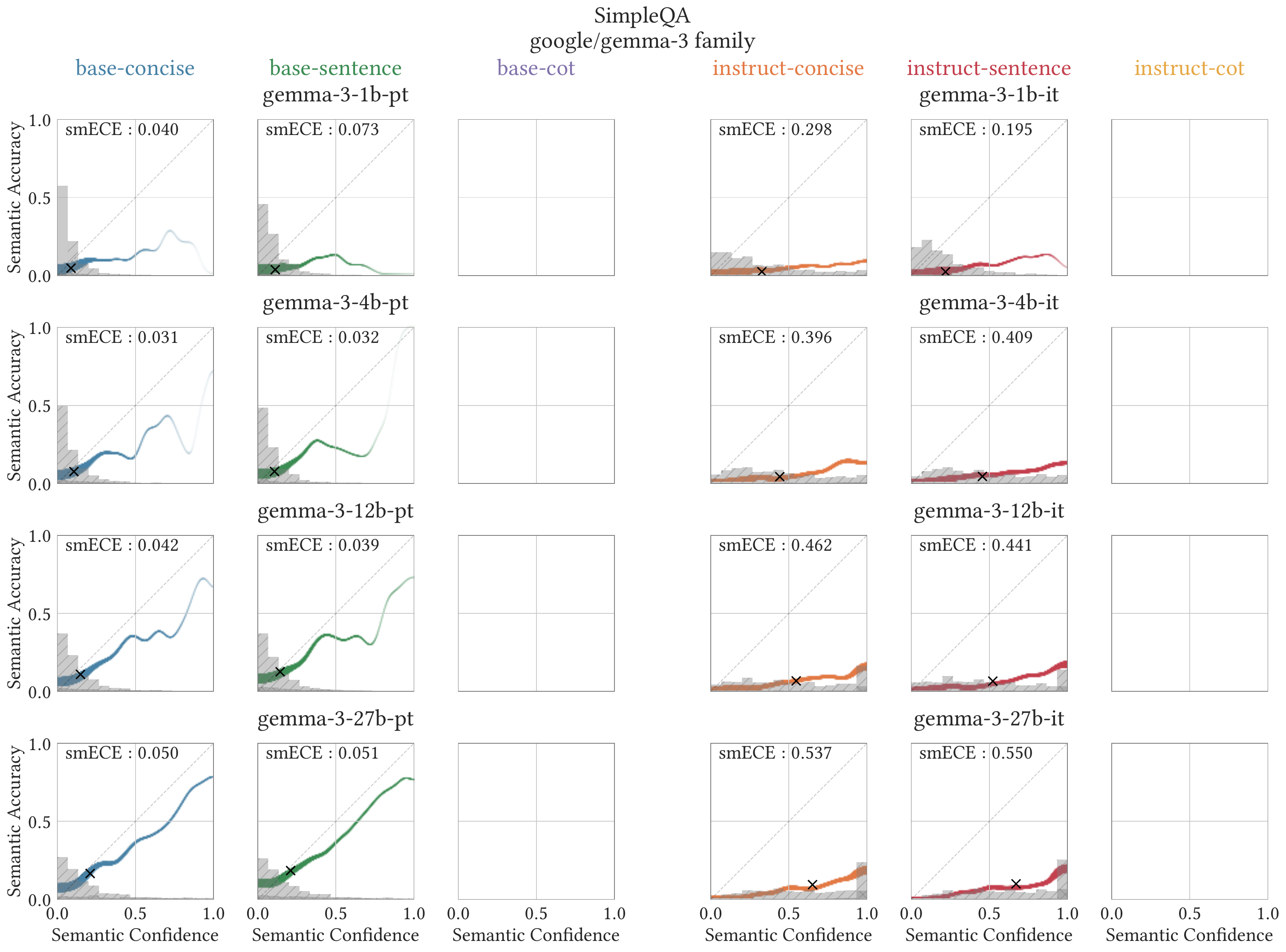}
\end{figure}
\begin{figure}[!htb]
\centering
\includegraphics[width=\linewidth]{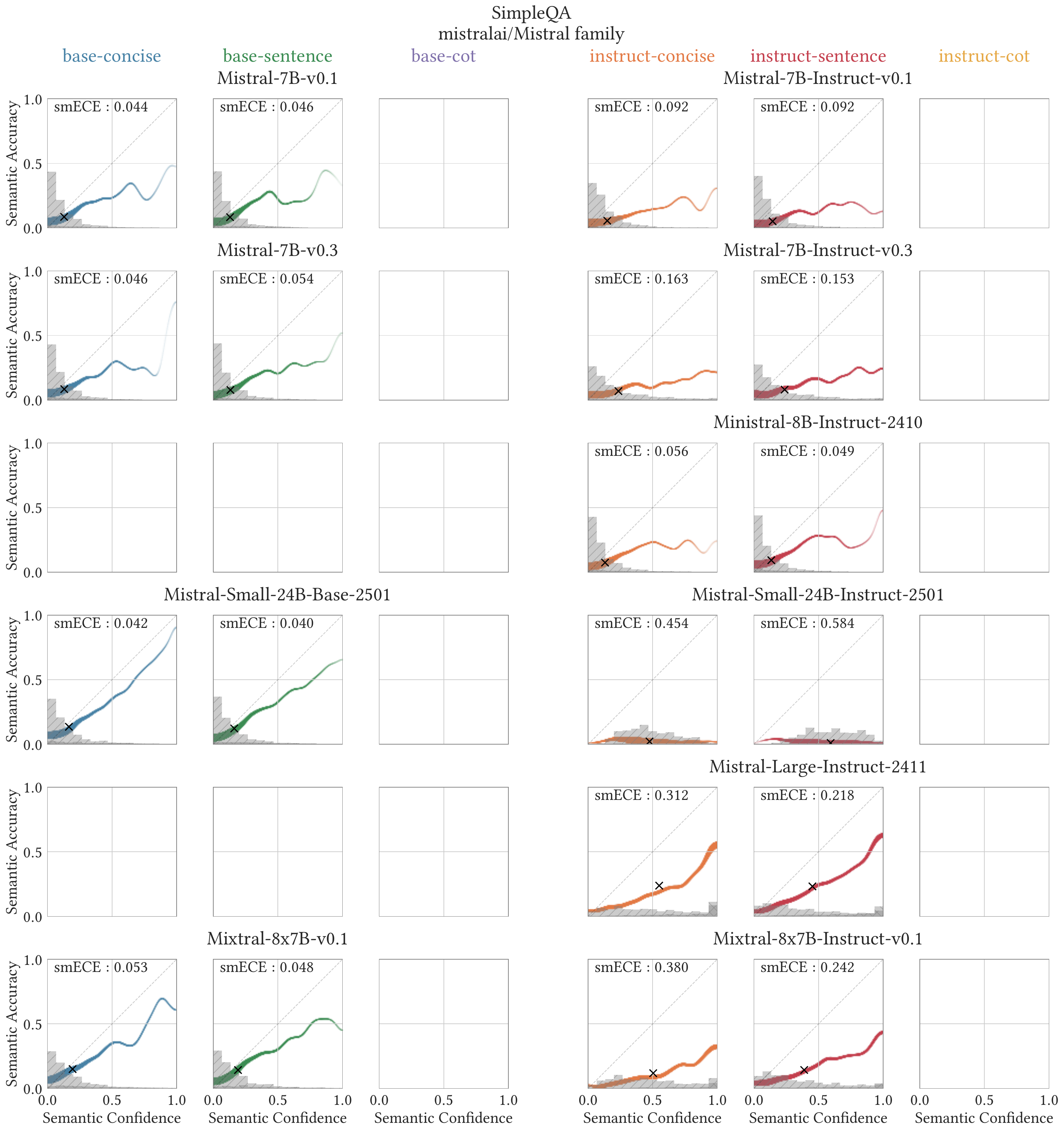}
\end{figure}
\begin{figure}[!htb]
\centering
\includegraphics[width=\linewidth]{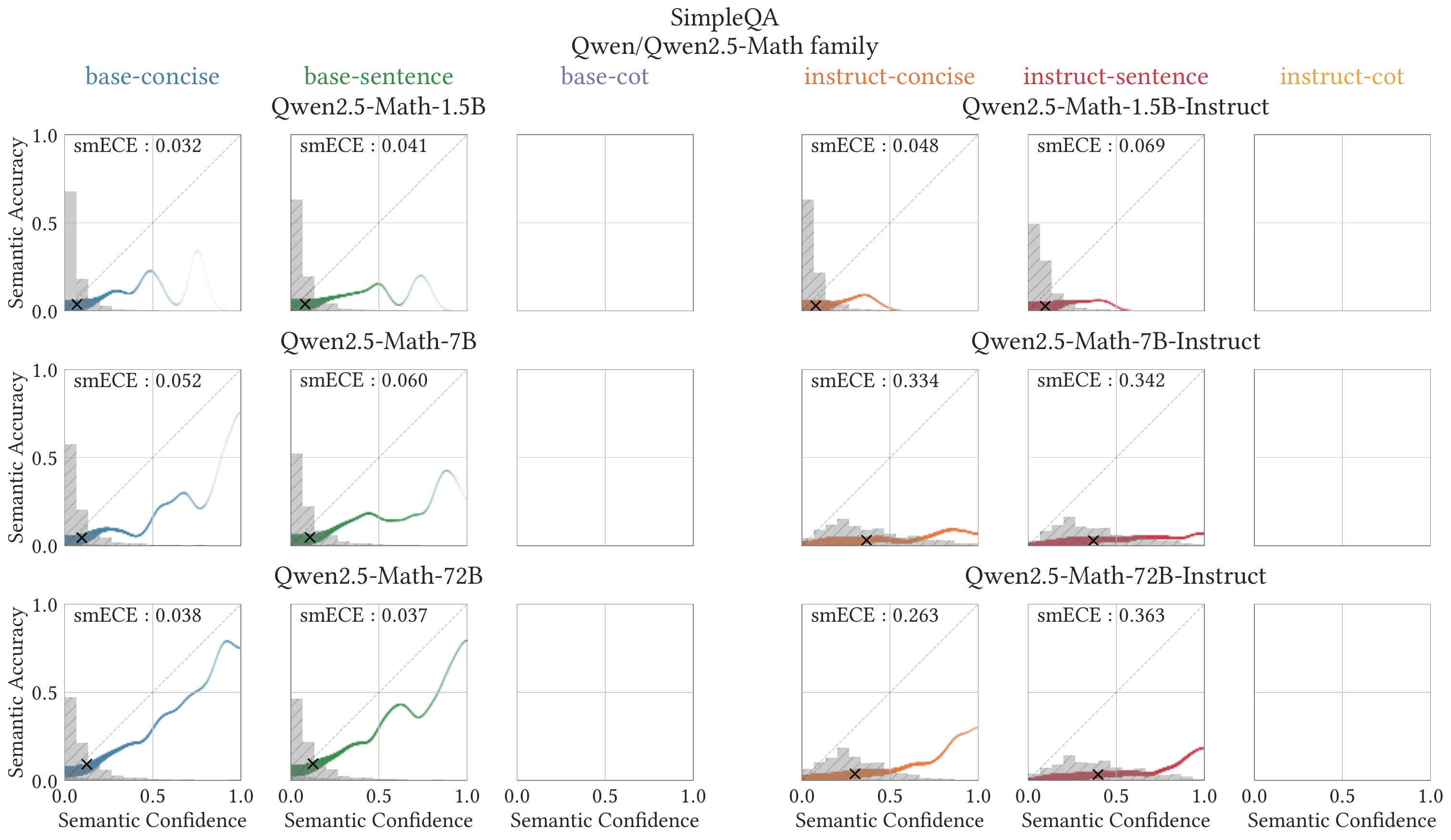}
\end{figure}
\begin{figure}[!htb]
\centering
\includegraphics[width=\linewidth]{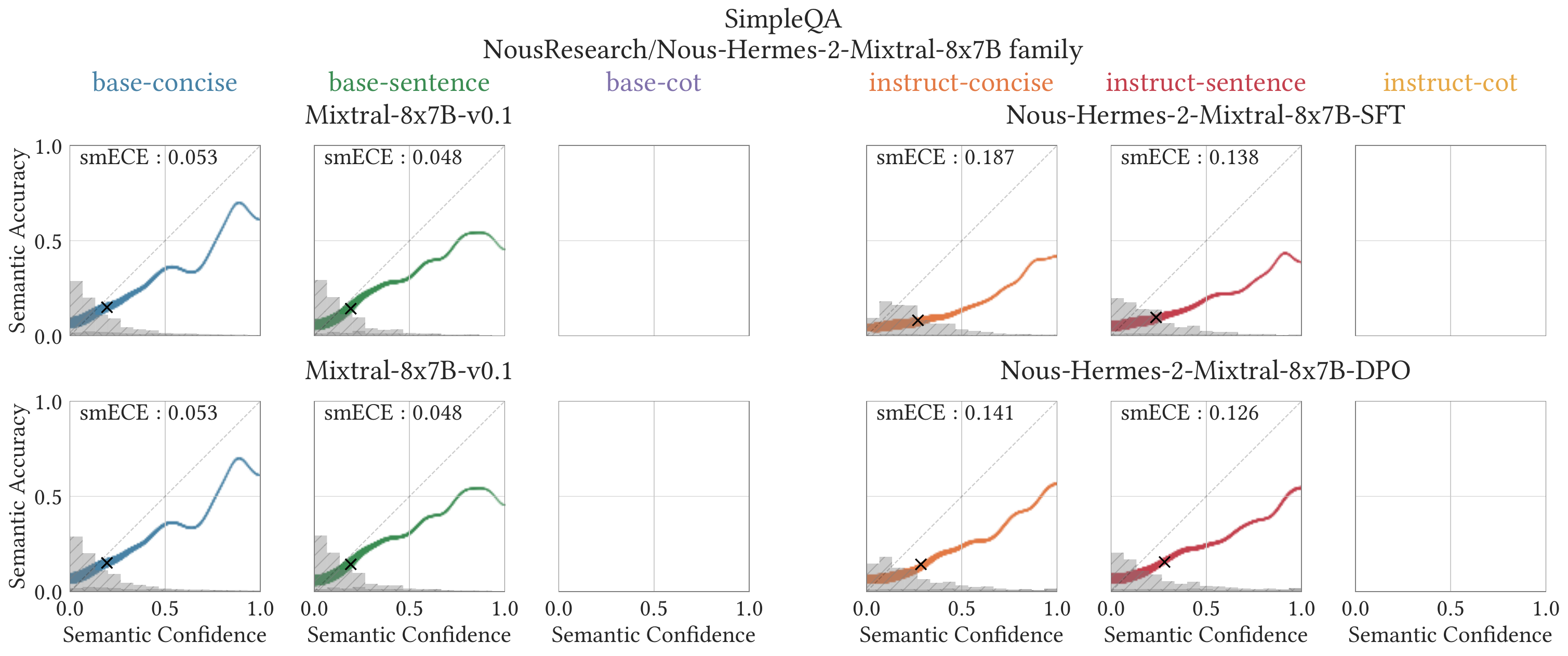}
\end{figure}
\begin{figure}[!htb]
\centering
\includegraphics[width=\linewidth]{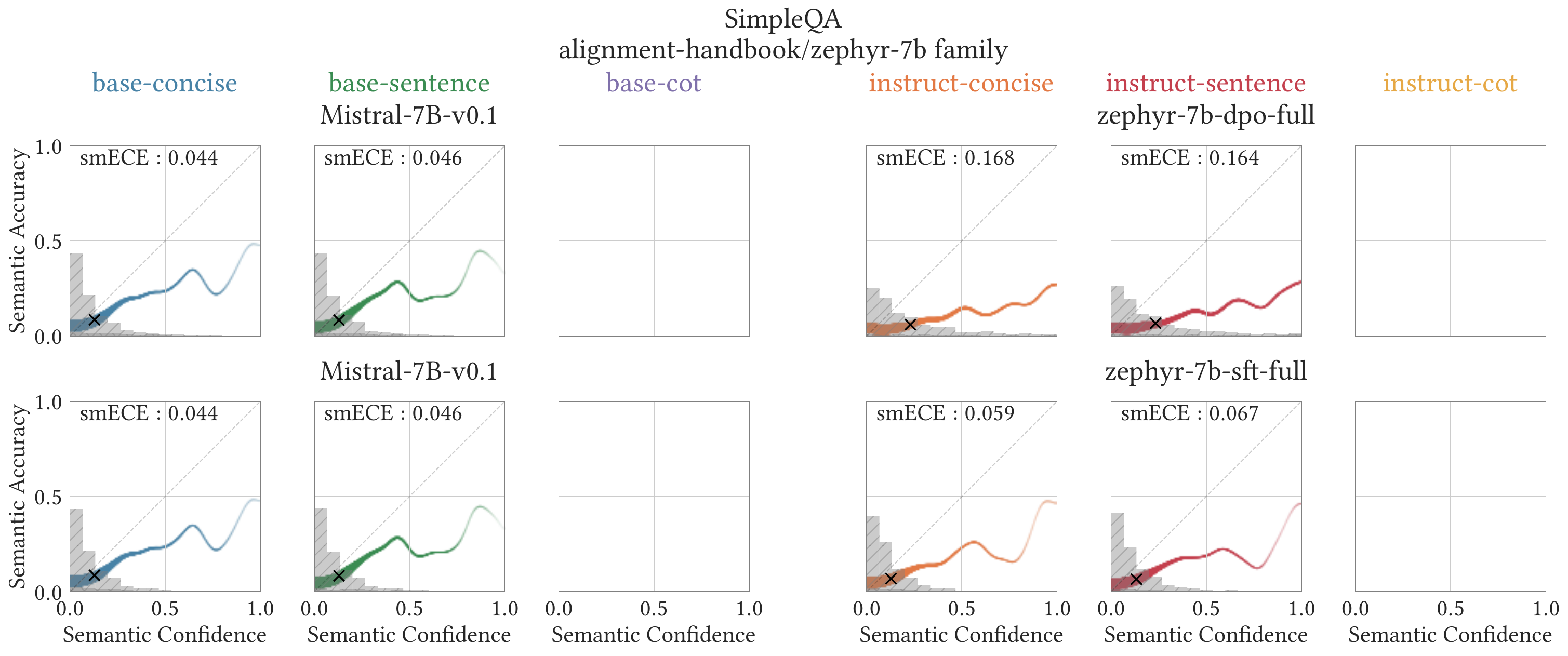}
\end{figure}
\begin{figure}[!htb]
\centering
\includegraphics[width=\linewidth]{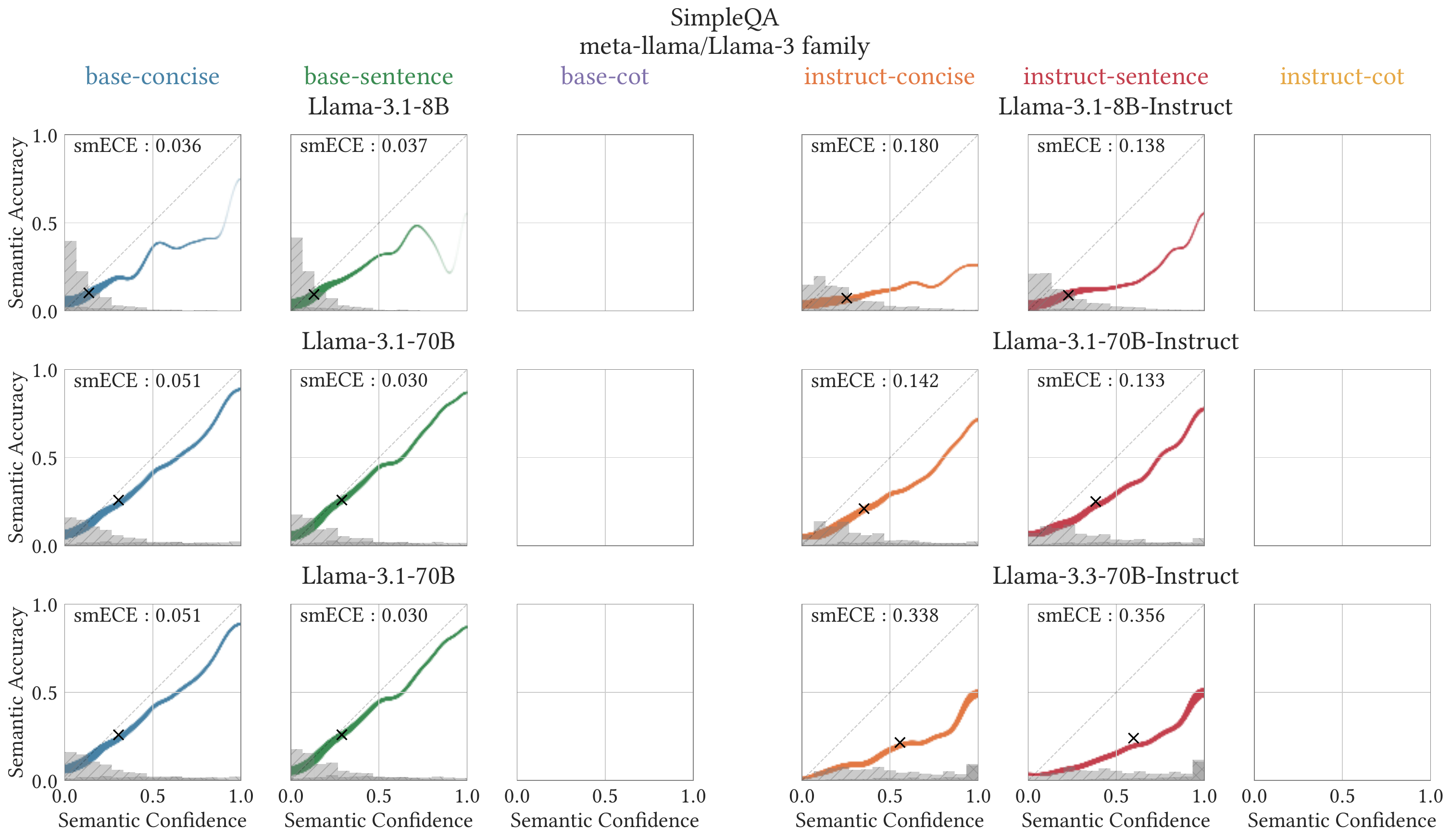}
\end{figure}

\begin{figure}[!htb]
\centering
\includegraphics[trim={0 10.5cm 0 0},clip,width=\linewidth]{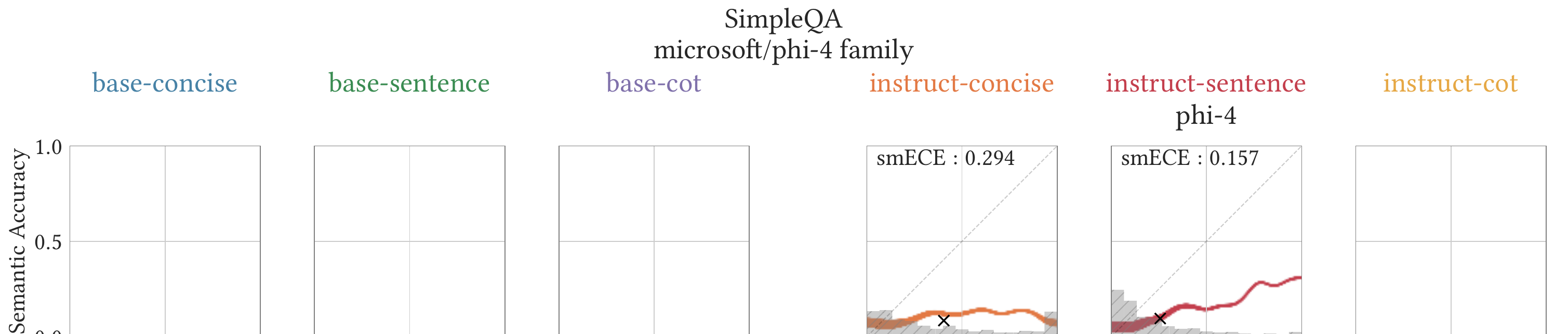}
\includegraphics[trim={0 0 0 20.2cm},clip,width=\linewidth]{figures/fig_appenidces_3_9.pdf}
\end{figure}

\end{document}